\documentclass{article} 
\usepackage{iclr2027_conference}
\iclrpreprintcopy

\usepackage{times}
\usepackage{booktabs}
\usepackage{algorithmic}
\usepackage[utf8]{inputenc} 
\usepackage[T1]{fontenc}    
\usepackage{url}            
\usepackage{booktabs}       
\usepackage{amsfonts}       
\usepackage{nicefrac}       
\usepackage{microtype}      
\usepackage{soul}
\usepackage{graphicx}
\usepackage{amsmath}
\usepackage{outlines}
\usepackage{xurl}

\usepackage{amsmath,amsfonts,bm}

\def\eqref#1{equation~\ref{#1}}

\def\1{\bm{1}}

\DeclareMathAlphabet{\mathsfit}{\encodingdefault}{\sfdefault}{m}{sl}
\SetMathAlphabet{\mathsfit}{bold}{\encodingdefault}{\sfdefault}{bx}{n}

\usepackage{multirow}
\usepackage{paralist}

\usepackage{multicol}
\usepackage{diagbox}

\usepackage[ruled,noend]{algorithm2e}

\newtheorem{theorem}{Theorem}
\SetCommentSty{mycommfont}

\usepackage{here}

\usepackage{amsmath,amssymb,amsfonts,amsbsy,amsfonts,latexsym}
\usepackage{makecell}
\usepackage{xcolor}
\usepackage{colortbl}

\usepackage{tabularx,colortbl,xcolor}
\usepackage[normalem]{ulem}
\useunder{\uline}{\ul}{}

\usepackage{enumitem}

\usepackage{xparse}

\SetKwInput{KwInput}{Input}
\SetKwInput{KwRequire}{Require}

\usepackage{longtable}

\NewDocumentCommand{\var}{O{s} m O{}}{%
  \ensuremath{#1_{#2}^{#3}}
}
\usepackage{siunitx}

\newcommand{\commentout}[1]{}

\definecolor{light-gray}{gray}{0.80}

\def\t{{\bf t}}

\usepackage{amsthm}

\usepackage{amsthm}

\newtheorem{mytheorem}{Theorem}[section]
\newtheorem{theorem}[mytheorem]{Theorem}
\newtheorem{corollary}{Corollary}[section]
\newtheorem{assumption}{Assumption}
\newtheorem{lemma}{Lemma}[section]
\theoremstyle{definition}
\newtheorem{definition}{Definition}[section]
\newtheorem{proposition}{Proposition}[section]
\newtheorem{remark}{Remark}[section]
\theoremstyle{plain}
\newtheorem{remk}{Remark}[section]

\definecolor{myblue}{rgb}{0,0.2,0.8}
\usepackage{hyperref}
\usepackage{xcolor}
\usepackage{listings}

\definecolor{upforestgreen}{rgb}{0.6, 0.8, 0.2}

\usepackage{chngcntr}
\usepackage{adjustbox}
\usepackage{wrapfig}
\usepackage{arydshln}
\usepackage{tcolorbox}
\usepackage{amsthm,amsmath,amssymb}
\usepackage{amsfonts,dsfont}
\usepackage{mathtools}
\usepackage{float}

\definecolor{mygreen}{HTML}{009901}
\definecolor{myred}{HTML}{A52A2A}

\title{Scaling Laws for Dynamic Mini-Batch SGD in Sketched Linear Regression}

\author{%
  Ziyan Chen \\
  The University of Sydney \\
  \texttt{ziyan.chen@sydney.edu.au}
  \And
  Zhongzhu Zhou \\
  Together AI \\
  \texttt{zhongzhu.zhou@sydney.edu.au}
  \And
  Dingxuan Zhou \\
  The University of Sydney \\
  \texttt{dingxuan.zhou@sydney.edu.au}
}

\begin{document}

\maketitle

\begin{abstract}
Mini-batching is central to large-scale optimization, yet its role in
statistical scaling laws remains limited. We study one-pass and
multi-pass batch SGD for sketched linear regression under power-law spectral
and source conditions. Our analysis reveals a two-horizon phenomenon induced
by warmup--stable--decay schedules: deterministic learning is governed by the
full optimization trajectory, while stochastic error retains only a shorter
terminal memory. For dynamic batch schedules, the individual batch sizes enter
through influence-weighted summaries that measure how strongly each update
affects the final risk. Consequently, batching leaves the approximation and
optimization-bias laws unchanged at a fixed update horizon, but controls the
one-pass variance and the multi-pass fluctuation around full-batch gradient
descent. We obtain matching one-pass variance bounds and nearly matching
multi-pass fluctuation bounds, recover static-batch and full-batch behavior as
special cases, and derive an oracle square-root rule for allocating a fixed
iteration budget. These results identify WSD horizon
separation and final-risk influence as the mechanisms governing dynamic
mini-batch scaling.
\end{abstract}


\section{Introduction}

Scaling laws have become a standard language for describing progress in modern
machine learning: across many domains, prediction error follows regular power
laws in model size, data size, and compute. This pattern has been documented
from early large-scale studies across translation, language, vision, and speech
\citep{hestness2017predictable} to modern language-model scaling analyses
\citep{kaplan2020neural,hoffmann2022training} and multimodal autoregressive
modeling \citep{henighan2020autoregressive}. As a result, scaling laws are now
used not only to summarize experiments, but also to guide forecasting,
resource allocation, and training design
\citep{rosenfeld2019constructive,zhai2022scaling,alabdulmohsin2022revisiting,besiroglu2024chinchilla,muennighoff2023scaling,paquette2024phases}.

Empirical scaling laws do not by themselves explain where the exponents come
from, which parts of the risk they describe, or how algorithmic choices change
them. Rigorous results are therefore rarer. In statistically transparent models,
however, approximation, optimization, and sampling effects can be separated.
Sketched linear regression with spectral and source conditions provides a
setting in which these components admit explicit scaling laws. It also lets us
isolate two questions that are not visible in a
single-horizon or static-batch description: how WSD separates persistent
learning from terminal stochastic memory, and how a time-varying batch
schedule is aggregated according to its influence on final risk. This
perspective connects to theoretical accounts based on manifold arguments
\citep{sharma2020neural,bahri2021explaining}, constructive learning curves and
solvable models \citep{hutter2021learning,maloney2022solvable}, dynamical and
feature-learning analyses \citep{bordelon2024dynamical,bordelon2024feature},
and renormalized high-dimensional asymptotics, quantization, and emergent
scaling behavior
\citep{michaud2023quantization,atanasov2024scaling,dohmatob2024tale,paquette2024phases,ren2025emergence}.

Batch size is one of the most important large-scale training knobs because it
affects hardware utilization and gradient noise. Empirically, large-batch rules enabled ImageNet training with batches up to
8192 \citep{goyal2017accurate}, LARS pushed convolutional training to 8K and
32K \citep{you2017largebatch}, and later studies showed that the gains are
highly workload-dependent and eventually saturate
\citep{shallue2019dataparallel,golmant2018largebatch}. An
influential synthesis is the gradient-noise-scale viewpoint of
\citet{mccandlish2018largescalebatch}, and batch-dependent empirical scaling
laws have recently been studied for language models
\citep{shuai2024batchsize}. Most recently, the theory-guided empirical study
of \citet{li2026joint} proposed joint loss laws for learning-rate and batch-size
sequences, used them to derive budget-aware schedules and validated the
predictions on deep-model settings.

There is also a growing theoretical understanding of batch size \(B\) through SGD
noise. \citet{smith2018bayesian} modeled SGD by an SDE with noise scale
proportional to \(\epsilon N/B\), suggesting that the optimal batch size should
grow with both the learning rate and the dataset size \(N\). \citet{smith2018dontdecay} studied the
closely related strategy of increasing batch size instead of decaying the
learning rate. In least-squares and related settings, the statistical effects
of mini-batching, multiple passes, tail averaging, and implicit regularization
have also been analyzed extensively
\citep{linrosasco2017,jain2017parallelizing,mucke2019tail,ge2019step,zou2021benefits,wu2022covshift,pillaud2018multiplepasses}.
What remains missing in the sketched linear-regression setting is a statistical
account of how WSD separates learning from stochastic memory and how a dynamic
schedule influences final risk under both one pass and data reuse.

This paper develops such a theory for sketched linear regression trained with
mini-batch methods. We study two stochastic procedures: one-pass batch SGD and
multi-pass batch SGD, with each multi-pass batch sampled uniformly without
replacement. For one-pass training we allow an arbitrary deterministic
batch-size schedule \(B_1,\dots,B_T\) with \(\sum_tB_t=N\); for multi-pass
training we allow any deterministic schedule \(B_1,\dots,B_L\) with
\(1\le B_t\le N\). We work within the sketched linear-regression framework of
\citet{lin2024scaling,lin2025datareuse}, while focusing on the mechanisms
specific to WSD and dynamic batching. WSD produces a full learning horizon and
a shorter stochastic-memory horizon, while the effect of a dynamic schedule is
summarized by final-risk influence weights. These mechanisms apply differently
to one-pass variance and multi-pass fluctuation, and they lead to an explicit
oracle allocation rule. Our main contributions are as follows.

\begin{itemize}
    \item \textbf{Dynamic-batch scaling law for one-pass WSD.}
    For every deterministic schedule satisfying \(\sum_t B_t=N\), we derive a
    sharp final-risk law that separates two WSD time scales. Approximation and
    optimization bias are governed by the sketch dimension and the full
    learning horizon, and are unchanged by reallocating a fixed
    number of updates. The centered variance instead depends on the shorter
    memory horizon \(\gamma T/\log T\), with the entire schedule summarized by
    the influence-weighted effective batch size \(B_{\mathrm{eff}}\). This
    characterization recovers the fixed-batch horizon--noise tradeoff as a
    special case; see Theorem~\ref{thm:onepass-batch-sgd-scaling}.

    \item \textbf{Dynamic-batch scaling law for multi-pass WSD.}
    We decompose the multi-pass risk into a full-batch GD reference contribution
    and a sampling fluctuation. At a fixed horizon \(L\), the approximation,
    GD bias, and GD variance are independent of the batch schedule; all schedule
    dependence is captured by the influence-weighted finite-population factor
    \(\rho_{\mathrm{eff}}\). The resulting law covers arbitrary
    without-replacement schedules, recovers the static-batch correction, and
    reduces exactly to GD when every batch is full; see
    Theorem~\ref{thm:multipass-batch-sgd-scaling}.

    \item \textbf{Influence-matched batch allocation.}
    The two scaling laws yield a common oracle design principle under a fixed
    update horizon and sample-processing budget. In the continuous interior
    regime, both the one-pass variance and multi-pass fluctuation are minimized
    by allocating \(B_t^\star\propto\sqrt{\kappa_t}\), where \(\kappa_t\)
    measures the final-risk influence of update \(t\). Thus the optimal
    schedule assigns more samples precisely where gradient noise survives most
    strongly at the final iterate; see
    Corollary~\ref{cor:influence-matched-allocation}.
\end{itemize}

\paragraph{Notation.}
For positive functions \(f\) and \(g\), we write \(f\lesssim g\) and
\(f\gtrsim g\) for upper and lower bounds up to logarithmic factors, and
\(f\asymp g\) when both hold. We use \(\Theta(g)\) for two-sided scaling up
to positive constant factors only, whereas \(\widetilde{\Theta}(g)\) may
additionally suppress multiplicative logarithmic factors. Vector inner products are denoted by
\(\langle u,v\rangle\) or \(u^\top v\), and matrix inner products by
\(\langle A,B\rangle:=\operatorname{tr}(A^\top B)\). The symbol
\(\|\cdot\|\) denotes the operator norm for matrices and the \(\ell_2\)-norm
for vectors; \(\|v\|_A^2:=v^\top A v\) for PSD \(A\). We write
\(A\preceq B\) when \(B-A\) is PSD, \(\mu_j(A)\) and \(r(A)\) for the
eigenvalues and rank of a symmetric matrix, and \(\log\) for the natural
logarithm.
\section{Preliminaries}
\label{sec:prelim}

We study the normal GD iterate \(\theta_t\), one-pass batch SGD, and
multi-pass batch SGD. Each stochastic update \(t\) averages a mini-batch of
size \(B_t\). When \(B_t=1\) at every update, the one-pass procedure reduces
to standard single-sample SGD. In the multi-pass method, each batch is sampled
without replacement, independently across iterations; this procedure should
therefore not be confused with random reshuffling across an epoch.

\paragraph{Problem setup.}
Let \(\mathcal{H}\) be a finite- or countably infinite-dimensional Hilbert
space, and let \(P\) be a Borel probability measure on
\(\mathcal{H}\times\mathbb{R}\). A generic input--response pair is denoted by
\((x,y)\sim P\). The input covariance is
\(H:=\mathbb{E}[xx^\top]\), with eigenvalues
\((\lambda_i)_{i\ge 1}\). For \(w\in\mathcal{H}\), define the population
square-loss risk and a population minimizer by
\[
R(w):=\mathbb{E}\bigl[(\langle x,w\rangle-y)^2\bigr],
\qquad
w^\ast\in\arg\min_{w\in\mathcal{H}}R(w).
\]
Rather than observing \(x\) directly, we observe its \(M\)-dimensional sketch
\(Sx\), where \(S:\mathcal{H}\to\mathbb{R}^M\). Throughout the main results,
\(S\) is a Gaussian sketch: in an orthonormal eigenbasis of \(H\), its entries
are independent \(\mathcal N(0,1/M)\) random variables. A sketched parameter
\(u\in\mathbb{R}^M\) represents the original-space predictor \(S^\top u\),
whose population risk is
\[
R_M(u):=R(S^\top u)
=\mathbb{E}\bigl[(\langle Sx,u\rangle-y)^2\bigr].
\]

Given \(N\) i.i.d. samples \(D=\{(x_i,y_i)\}_{i=1}^N\) drawn from \(P\), we observe
\(O=\{(Sx_i,y_i)\}_{i=1}^N\). Write the empirical covariates and responses as \(X:=(x_1,\dots,x_N)^\top\) and \(y:=(y_1,\dots,y_N)^\top\), and define
\[
\Sigma:=SHS^\top,
\qquad
\widehat{\Sigma}:=\frac{1}{N}SX^\top X S^\top,
\qquad
\widehat{b}:=\frac{1}{N}SX^\top y.
\]
Conditioned on \(S\), the minimizer of \(R_M\) is
\[
u^\ast=(SHS^\top)^{-1}SHw^\ast=\Sigma^{-1}SHw^\ast.
\]

\paragraph{Optimization procedures.}
We compare three methods using the same warmup--stable--decay (WSD) schedule
family. For a generic update horizon \(J\), choose a warmup endpoint
\(J_{\mathrm w}\) and a stable-phase endpoint \(J_{\mathrm s}\) with
\(1\le J_{\mathrm w}<J_{\mathrm s}<J\), and define
\begin{equation}
\gamma_t^{(J)}
=
\begin{cases}
\gamma t/J_{\mathrm w},&1\le t\le J_{\mathrm w},\\[1mm]
\gamma,&J_{\mathrm w}<t\le J_{\mathrm s},\\[1mm]
\gamma\exp\!\left(-\frac{(t-J_{\mathrm s})\log J}{J-J_{\mathrm s}}\right),
&J_{\mathrm s}<t\le J.
\end{cases}
\label{eq:wsd-schedule}
\end{equation}
For fixed constants \(c_{\mathrm s},c_{\mathrm d}>0\), we assume \(J_{\mathrm s}-J_{\mathrm w}\ge c_{\mathrm s}J\) and \(J-J_{\mathrm s}\ge c_{\mathrm d}J\).

For one-pass SGD we instantiate \(J=T\) and drop the superscript on
\(\gamma_t^{(T)}\). For normal GD and multi-pass batch SGD we instantiate
\(J=L\) and again write the rates as \(\gamma_t\). Consequently, the learning
horizons scale as \(\gamma T\) and \(\gamma L\), while the final stochastic
memory scales as \(\gamma T/\log T\) and \(\gamma L/\log L\), respectively.
We retain distinct symbols because \(T\) is constrained by the one-pass data
partition, whereas \(L\) is an independently chosen data-reuse horizon.

Below, the superscripts \(\mathrm{op}\) and \(\mathrm{mp}\) distinguish
procedure-specific iterates and empirical batch statistics. Shared quantities,
including \(B_t\), \(I_t\), and \(\gamma_t\), remain unsuperscripted.

\paragraph{1. Normal GD.}
Let \((\gamma_t)\) be the prescribed stepsize schedule. The normal GD iterate is
\begin{equation}
\theta_t
=
\theta_{t-1}
-
\gamma_t\widehat{\Sigma}\theta_{t-1}
+
\gamma_t\widehat{b},
\qquad
\theta_0=0.
\label{eq:proc-normal-gd}
\end{equation}

\paragraph{2. One-pass Batch SGD.}
Fix \(T\ge3\) and arbitrary positive integers \(B_1,\dots,B_T\) satisfying
\(\sum_{t=1}^T B_t=N\), and partition \([N]\) into disjoint batches
\([N]=I_1\sqcup\cdots\sqcup I_T\), \(|I_t|=B_t\). For each batch, define
\[
\widehat{\Sigma}_t^{\mathrm{op}}
:=
\frac{1}{B_t}\sum_{i\in I_t} Sx_i x_i^\top S^\top,
\qquad
\widehat{b}_t^{\mathrm{op}}
:=
\frac{1}{B_t}\sum_{i\in I_t} Sx_i y_i.
\]
The one-pass batch SGD iterate is
\begin{equation}
u_t^{\mathrm{op}}
=
u_{t-1}^{\mathrm{op}}
-
\gamma_t\widehat{\Sigma}_t^{\mathrm{op}} u_{t-1}^{\mathrm{op}}
+
\gamma_t\widehat{b}_t^{\mathrm{op}},
\qquad
t=1,\dots,T,
\qquad
u_0^{\mathrm{op}}=0.
\label{eq:proc-onepass-batch-sgd}
\end{equation}
Thus the method performs \(T\) updates and uses every observation exactly
once; the one-pass constraint is \(\sum_{t=1}^T B_t=N\), not \(T=N/B\).

\paragraph{3. Multi-pass Batch SGD.}
Fix a deterministic schedule \(B_1,\dots,B_L\) with \(1\le B_t\le N\). At each step \(t\in[L]\), independently sample a subset \(I_t\subset[N]\) of size \(B_t\), uniformly without replacement within the mini-batch.

For each sampled batch \(I_t\), define
\[
\widehat{\Sigma}_t^{\mathrm{mp}}
:=
\frac{1}{B_t}\sum_{i\in I_t} Sx_i x_i^\top S^\top,
\qquad
\widehat{b}_t^{\mathrm{mp}}
:=
\frac{1}{B_t}\sum_{i\in I_t} Sx_i y_i.
\]
The multi-pass batch SGD iterate is
\begin{equation}
u_t^{\mathrm{mp}}
=
u_{t-1}^{\mathrm{mp}}
-
\gamma_t \widehat{\Sigma}_t^{\mathrm{mp}} u_{t-1}^{\mathrm{mp}}
+
\gamma_t \widehat{b}_t^{\mathrm{mp}},
\qquad
t=1,\dots,L,
\qquad
u_0^{\mathrm{mp}}=0.
\label{eq:proc-multipass-batch-sgd}
\end{equation}
The number of updates remains \(L\), independently of the schedule. There is
no one-pass constraint on \(\sum_tB_t\); when a gradient-evaluation budget is
needed, we write \(D_{\mathrm{grad}}:=\sum_{t=1}^LB_t\) to avoid conflict with the dataset notation \(D\). If \(B_t=N\) at every step, this procedure is exactly normal GD.

In summary, all three procedures optimize the same sketched empirical problem
under a common WSD schedule. The one-pass procedure allows a deterministic schedule
\((B_t)_{t=1}^T\) with \(\sum_tB_t=N\), whereas the multi-pass procedure
allows any deterministic schedule \((B_t)_{t=1}^L\) with \(1\le B_t\le N\).
When comparing a stochastic iterate with its deterministic reference, we always use the same stepsize schedule and the same number of updates: \(T\) for one-pass case and \(L\) for multi-pass case.

Unless a subscript indicates otherwise, expectations in the theorem
statements are taken over \(w^\ast\), the sampled dataset, and the mini-batch
randomness when applicable. Probability statements over the Gaussian sketch
\(S\) are stated separately.

\section{Main Results}
\label{sec:main-results}

This section develops scaling laws for one-pass and multi-pass sketched SGD
under the power-law/source-condition model. The central phenomenon is that WSD
creates two distinct time scales: deterministic learning depends on the full
optimization trajectory, whereas stochastic error is governed by a shorter
terminal-memory horizon. Dynamic batch sizes do not alter the deterministic
contributions at a fixed update horizon; instead, their effects are aggregated
through final-risk influence weights in the one-pass variance and multi-pass
fluctuation.

We first state the assumptions needed for the WSD and dynamic-batch analysis,
then give a common risk decomposition and derive the two scaling laws.

\begin{assumption}[Data assumptions]
\label{ass:data}
Assume the following conditions on the data distribution \(P\).
\begin{enumerate}[label=\textbf{\Alph*.}, ref=\theassumption.\Alph*]
    \item \label{ass:data-gaussian} \textbf{Gaussian design.}
    The feature vector satisfies $x \sim \mathcal{N}(0,H).$

    \item \label{ass:data-well-specified} \textbf{Homoscedastic label model.}
    Conditional on \(w^\ast\), the response satisfies
    \[
    y=\langle x,w^\ast\rangle+\varepsilon,
    \qquad
    \varepsilon\perp x,
    \qquad
    \mathbb E[\varepsilon]=0,
    \qquad
    \mathbb E[\varepsilon^2]=\sigma^2.
    \]

    \item \label{ass:data-power-law} \textbf{Power-law spectrum.} Eigenvalues of \(H\) satisfy \(\lambda_i \asymp i^{-a}\) for some \(a>1\).

\end{enumerate}
\end{assumption}

\paragraph{Interpretation.}
Assumption~\ref{ass:data} is a stylized random-design regression model. Gaussian
features and homoscedastic noise are idealizations, but they are standard in
least-squares theory because they isolate spectral effects and provide exact
moment identities; related SGD analyses work under broader moment-controlled
designs \citep{dieuleveut2016nonparametric,jain2017parallelizing}.
The power law means that the data contain many progressively weaker directions,
rather than having a sharp finite-dimensional cutoff. Such spectral decay is a
common capacity model in learning theory and has also been observed in
high-dimensional neural population responses \citep{stringer2019geometry}.

\begin{assumption}[Source condition in diagonal coordinates]
\label{ass:source-diag}
Assume without loss of generality that \(H\) is diagonal with non-increasing diagonal entries \(H=\operatorname{diag}(\lambda_1,\lambda_2,\dots)\). Assume the true parameter \(w^\ast\) satisfies: for some \(b>1\),
\[
\mathbb{E}[w_i^\ast w_j^\ast]=0
\quad\text{for all } i\neq j; \qquad 
\mathbb{E}[\lambda_i (w_i^\ast)^2]\asymp i^{-b}
\quad\text{for all } i\ge 1.
\]

\end{assumption}

\paragraph{Interpretation.}
Assumption~\ref{ass:source-diag} measures how the target aligns with the
eigendirections of the data covariance. Larger \(b\) places less target energy
in weak, high-index directions and therefore represents a smoother or easier
prediction problem. Conditions of this type are usually called source
conditions and are standard in kernel regression and statistical inverse
problems because, together with eigenvalue decay, they determine approximation
and learning rates \citep{caponnetto2007optimal,dieuleveut2016nonparametric,
pillaud2018multiplepasses}. Diagonalizing \(H\) is only a coordinate choice;
the uncorrelated-coordinate prior is the substantive average-case
simplification.

\begin{assumption}[Empirical contraction]
\label{ass:stepsize}
Under the notation of the theorem and its proof, assume
\[
\gamma\mu_1(\widehat\Sigma)\le c_0<1.
\]
\end{assumption}

\paragraph{Interpretation.}
Assumption~\ref{ass:stepsize} retains only the sample-dependent contraction
condition used by the empirical recursions. The peak-stepsize restriction is
stated directly in the multi-pass theorem.
The trace, effective-dimension, and maximum-norm conditions used later are
consequences of the Gaussian power-law model and the theorem's horizon
condition and therefore need not be imposed separately. We retain empirical
contraction because a fixed, unclipped WSD stepsize cannot satisfy it pathwise
for every realization of an unbounded Gaussian design.

For both methods, let \(\mu_1,\dots,\mu_M\) be the eigenvalues of
\(\Sigma=SHS^\top\), and define the final-risk influence of update \(t\) over a total horizon \(J\) by
\[
\kappa_t^{(J)}
:=
\gamma_t^2
\sum_{j=1}^M
\mu_j^2
\prod_{s=t+1}^J(1-\gamma_s\mu_j)^2,
\qquad t=1,\dots,J.
\]
Similarly, \(J=T\) denotes the one-pass influence and \(J=L\) for multi-pass. For one-pass and multi-pass respectively, define the effective batch size and effective fluctuation prefactor
\begin{equation*}
\frac{1}{B_{\mathrm{eff}}}
:=
(\displaystyle\sum_{t=1}^T \frac{\kappa_t^{(T)}}{B_t}) / (\displaystyle\sum_{t=1}^T\kappa_t^{(T)}),
\quad
\rho_{\mathrm{eff}}
:=
(\displaystyle\sum_{t=1}^L
\rho_t\kappa_t^{(L)})/(\displaystyle\sum_{t=1}^L\kappa_t^{(L)}),
\quad \text{where } 
\rho_t
:=
\frac{N-B_t}{B_t(N-1)}.
\end{equation*}
On the full-rank sketch event the denominator is positive. If \(B_{\min}:=\min_tB_t\) and \(B_{\max}:=\max_tB_t\), then \(B_{\min}\le B_{\mathrm{eff}}\le B_{\max}\); a constant schedule \(B_t\equiv B\) gives \(B_{\mathrm{eff}}=B\).

\begin{remark}[Meaning of final-risk influence]
The weight \(\kappa_t^{(J)}\) measures how much stochastic-gradient noise
injected at iteration \(t\) remains in the prediction risk at the final
iterate. Under the power-law spectrum and before the spectral cutoff reaches
the model dimension, it has the scaling
\[
\kappa_t^{(J)}
\asymp
\gamma_t^2\bigl(1+R_t^{(J)}\bigr)^{1/a-2},
\qquad
R_t^{(J)}:=\sum_{s=t+1}^J\gamma_s;
\]
see Appendix~\ref{subsec:app-final-risk-influence} for the justification and more interpretation.
\end{remark}

The next proposition packages the two stochastic procedures into one
structural statement. Their risks share the same baseline contribution: the
one-pass method further splits into bias plus variance, while the multi-pass
method splits into a common GD reference contribution plus a fluctuation term.

\begin{proposition}[Risk decompositions for the two stochastic procedures]
\label{prop:main-risk-decomp}
Assume Assumption~\ref{ass:data-well-specified}. Let $\bar u_T:=\mathbb{E}[u_T^{\mathrm{op}}].$ Then the risks decompose as follows:
\[
\begin{aligned}
\mathbb{E}[R_M(u_T^{\mathrm{op}})]
={}&
\underbrace{R_M(u^\ast)}_{\substack{\text{common baseline risk}}}
+
\underbrace{R_M(\bar u_T)-R_M(u^\ast)}_{\substack{\text{one-pass bias}\\\text{excess risk}}}
+
\underbrace{\mathbb{E}\bigl[R_M(u_T^{\mathrm{op}})-R_M(\bar u_T)\bigr]}_{\substack{\text{one-pass variance}\\\text{excess risk}}}.
\end{aligned}
\]
\[
\begin{aligned}
\mathbb{E}[R_M(u_L^{\mathrm{mp}})]
={}&
\underbrace{R_M(u^\ast)}_{\substack{\text{common baseline risk}}}
+
\underbrace{\mathbb{E}\bigl[R_M(\theta_L)-R_M(u^\ast)\bigr]}_{\substack{\text{common GD-reference}\\\text{excess risk}}}
+
\underbrace{\mathbb{E}\bigl[R_M(u_L^{\mathrm{mp}})-R_M(\theta_L)\bigr]}_{\substack{\text{multi-pass fluctuation}\\\text{excess risk}}}.
\end{aligned}
\]
Note that \(R_M(u^\ast) = \underbrace{R(w^\ast)}_{\substack{\text{irreducible risk}}} + \underbrace{\bigl[R_M(u^\ast)-R(w^\ast)\bigr]}_{\substack{\text{approximation risk}}}.\)
\end{proposition}

\noindent The proof is given at the end of Appendix~\ref{sec:app-prelim}.
Proposition~\ref{prop:main-risk-decomp} is the organizing principle for the
rest of the section: the next theorem specializes the one-pass decomposition,
and the theorem after that specializes the multi-pass decomposition.

\begin{mytheorem}[One-pass dynamic-batch scaling law]
\label{thm:onepass-batch-sgd-scaling}
Assume Assumptions~\ref{ass:data-gaussian}, \ref{ass:data-well-specified},
\ref{ass:data-power-law}, and \ref{ass:source-diag}, together with the
one-pass WSD conditions in Section~\ref{sec:prelim}. Suppose
\[
b<a+1,
\qquad
\sigma^2\asymp 1,
\qquad
\frac{\gamma T}{\log T}\gtrsim 1,
\qquad
0<\gamma\le c.
\]
Then every deterministic schedule \(B_1,\dots,B_T\in\mathbb N\) satisfying
\(\sum_tB_t=N\) obeys, with probability at least \(1-\exp(-\Omega(M))\) over the randomness of \(S\),
\[
\begin{aligned}
\mathbb{E}\bigl[R_M(u_T^{\mathrm{op}})\bigr] - \sigma^2
=
\underbrace{\Theta\!\left(
M^{1-b}+(\gamma T)^{\frac{1-b}{a}}
\right)}_{\mathrm{Approx+Bias}}+\underbrace{\Theta\!\left(
\frac{\log T}{TB_{\mathrm{eff}}}
\min\!\left\{M,\left(\frac{\gamma T}{\log T}\right)^{\frac 1a}\right\}
\right)}_{\mathrm{Var}}.
\end{aligned}
\]
For the static schedule \(B_t\equiv B\), where \(B\mid N\) and \(T=N/B\ge3\), one has \(B_{\mathrm{eff}}=B\).
\end{mytheorem}

\paragraph{Interpretation.}
The irreducible term \(\sigma^2\) is the label-noise floor. Within the
approximation-plus-bias contribution, \(M^{1-b}\) measures the information
lost by restricting prediction to the \(M\)-dimensional sketched model; since
\(b>1\), it decreases as the sketch dimension \(M\) grows. The term
\((\gamma T)^{(1-b)/a}\) is the optimization bias remaining after the full
WSD trajectory, and it decreases with the learning horizon \(\gamma T\).
Neither deterministic term depends on how a fixed number of updates is divided
into batches. The variance term instead reflects stochastic gradient noise.
Its spectral factor counts the directions that remain active over the shorter
WSD memory horizon \(\gamma T/\log T\), while \(B_{\mathrm{eff}}\) averages
the batch sizes according to their final-risk influence. Consequently, larger
batches reduce variance most when assigned to influential updates, and two
schedules with the same arithmetic mean \(N/T\) can have different risks. For
the static schedule \(B_t\equiv B\), one recovers \(B_{\mathrm{eff}}=B\) and
the usual horizon--noise tradeoff with \(T=N/B\). See Appendix~\ref{sec:app-main-theorem-proofs}, and in particular
Appendix~\ref{subsec:app-proof-onepass}, for the proof of
Theorem~\ref{thm:onepass-batch-sgd-scaling}.

We now turn to the multi-pass setting. The key question is how
mini-batching affects the decomposed risks, where the next theorem shows that only the fluctuation changes.

\begin{mytheorem}[Multi-pass dynamic-batch scaling law]
\label{thm:multipass-batch-sgd-scaling}
Assume Assumptions~\ref{ass:data-gaussian}, \ref{ass:data-well-specified},
\ref{ass:data-power-law}, \ref{ass:source-diag}, and
\ref{ass:stepsize}, together with the multi-pass WSD conditions in
Section~\ref{sec:prelim}. For a sufficiently small constant \(c_1>0\),
\[
b<a+1,
\quad
\sigma^2\asymp1,
\quad
\frac{\gamma L}{\log L}\gtrsim1,
\quad
\gamma\le\frac{c}{\log N},
\quad
(\gamma L)^{1/a}\le c_1M,
\quad
\gamma L\log N\le c_1N.
\]
Then every schedule \(1\le B_t\le N\) satisfies, with
probability at least \(1-\exp(-\Omega(M))\) over \(S\),
\[
\begin{aligned}
\mathbb E[R_M(u_L^{\mathrm{mp}})]-\sigma^2
=
\underbrace{\Theta(M^{1-b})}_{\mathrm{Approx}}
+\underbrace{\Theta((\gamma L)^{\frac{1-b}{a}})}_{\mathrm{GD\,Bias}}
+\underbrace{\Theta\!\left(
\frac{(\gamma L)^{\frac 1a}}{N}
\right)}_{\mathrm{GD\,Var}} 
+\underbrace{\widetilde{\Theta}\!\left(\rho_{\mathrm{eff}}
\frac{\log L}{L}
\left(\frac{\gamma L}{\log L}\right)^{\frac 1a}\right)}_{\mathrm{Fluc}}.
\end{aligned}
\]

For the static schedule \(B_t\equiv B\), with \(1\le B\le N\),
\(\rho_{\mathrm{eff}}\) is replaced by
\(\rho_B:=\frac{N-B}{B(N-1)}\). Here
\(\widetilde{\Theta}\) hides only the
\(\log N\) gap between the matching lower and upper bounds.
\end{mytheorem}

\paragraph{Interpretation.}
The noise floor, approximation error, GD bias, and GD variance have their
standard meanings: they respectively represent irreducible label noise,
information lost through the finite sketch dimension, incomplete optimization
along the full GD trajectory, and estimation error from the finite dataset.
The new schedule-dependent term is the fluctuation around that full-batch GD
trajectory. A batch sampled without replacement contributes the
finite-population factor \(\rho_t=(N-B_t)/(B_t(N-1))\), and
\(\rho_{\mathrm{eff}}\) averages these factors according to the final-risk
influence of each update. Thus enlarging an influential batch reduces the
fluctuation, while a full batch has \(\rho_t=0\) and introduces none. The
remaining factor is governed by the shorter WSD stochastic-memory horizon.
Hence, when \(L\) and the learning-rate trajectory are fixed, the batch
schedule changes only the fluctuation; the approximation, GD bias, and GD
variance remain unchanged. See Appendix~\ref{subsec:app-proof-multipass} for the proof of
Theorem~\ref{thm:multipass-batch-sgd-scaling}.

\begin{remark}[Multi-pass sampling with replacement]
If each multi-pass mini-batch is sampled with replacement, the same scaling law holds after defining \(\rho_t=1/B_t\) in \(\rho_{\mathrm{eff}}\); all other terms are unchanged. See Appendix~\ref{subsec:wr-batch-fluc-source}, in particular Lemma~\ref{lem:wr-batch-fluc-source}.
\end{remark}

\subsection{Discussion and Related Works}

\paragraph{Empirical evidence.}
A substantial empirical literature shows that batch size is a central training
control. Large-batch ImageNet training motivated the linear learning-rate
scaling rule \citep{goyal2017accurate}; increasing the batch size during
training reproduced much of the effect of learning-rate decay while
improving parallelism \citep{smith2018dontdecay}; and measurements across
models and workloads identified a task-dependent critical batch size beyond
which additional parallelism yields diminishing returns
\citep{mccandlish2018largescalebatch,shallue2019dataparallel}. These studies
are primarily empirical and focus on training speed, hardware utilization, or
generalization in deep networks. The theory-inspired empirical work of
\citet{li2026joint} additionally fits joint loss laws to complete learning-rate
and batch-size sequences and validates the resulting budget-aware schedules
across deep-model settings. Our results address a complementary question:
within a tractable statistical model under one pass and data reuse, which iterations' batch sizes determine the final prediction risk?

\paragraph{Theoretical comparisons.}
\citet{jain2017parallelizing} analyze mini-batching and averaging for finite-dimensional least squares, but not power-law sketches, WSD, or dynamic
allocations. Under spectral and source conditions, \citet{mucke2019tail} give
rates for tail-averaged one-pass kernel SGD with mini-batching, but use
prescribed batches and an averaged estimator rather than a final WSD iterate or data reuse. \citet{wang2026fastcatchup} optimize batch schedules through a
continuous-time SDE and functional scaling laws for one-pass kernel regression,
but do not provide exact finite-horizon spectral weights or a multi-pass risk
decomposition. In contrast, our discrete analysis features another risk decomposition with dynamic
schedules governed by the influence \(\kappa_t\).

\paragraph{What batch size changes.}
The common message of the two theorems is that batch size acts as a
noise-control parameter rather than a deterministic regularizer. In the
one-pass theorem, batching lowers the centered variance; in the multi-pass
theorem, it lowers only the fluctuation around the GD reference path. This
interpretation is consistent with the gradient-noise-scale viewpoint of
\citet{smith2018bayesian,mccandlish2018largescalebatch} and with empirical
large-batch studies showing that large batches can be effective when properly
tuned but that their gains eventually saturate
\citep{goyal2017accurate,smith2018dontdecay,shallue2019dataparallel}. Our
theorems make this principle explicit in the present linear-regression
setting: once the stochastic terms fall below the approximation and
optimization terms, further increasing the influential batch sizes no longer
changes the leading statistical scaling. In the multi-pass setting, the
finite-population correction decreases as the batch grows. The benefit is therefore greatest when the most influential
iterations receive substantial batches.

\subsection{Influence-matched batch allocation}
\label{subsec:influence-matched-allocation}

The two theorems also yield an oracle allocation rule when the number of
updates and the total sample-processing budget are fixed.

\begin{corollary}[Influence-matched batch allocation]
\label{cor:influence-matched-allocation}
Assume the respective conditions of
Theorems~\ref{thm:onepass-batch-sgd-scaling} and
\ref{thm:multipass-batch-sgd-scaling}. In the continuous interior regime,
the optimal schedules are
\[
\boxed{
B_t^\star
=
(N\sqrt{\kappa_t^{(T)}})/
(\displaystyle\sum_{s=1}^T\sqrt{\kappa_s^{(T)}})
}
\qquad\text{(one pass, fixed \(T\), \(\sum_{t=1}^T B_t=N\))},
\]
and
\[
\boxed{
B_t^\star
=
(D_{\mathrm{grad}}\sqrt{\kappa_t^{(L)}})/
(\displaystyle\sum_{s=1}^L\sqrt{\kappa_s^{(L)}})
}
\qquad\text{(multi-pass, fixed \(L\), \(\sum_{t=1}^L B_t=D_{\mathrm{grad}}\))}.
\]
If a batch-size bound is active, the corresponding square-root allocation is
clipped through the usual KKT conditions; integer schedules can be obtained by
budget-preserving rounding.
\end{corollary}

\paragraph{Interpretation and related work.}
In both cases, \(B_t^\star\propto\sqrt{\kappa_t}\): more samples are assigned
to updates whose noise has greater final-risk influence. This is a
fixed-horizon oracle rule because \(\kappa_t\) depends on the population
sketched covariance. It optimizes the one-pass variance for fixed \(T\), or
the multi-pass fluctuation for fixed \(L\), without changing the deterministic
terms. By contrast, \citet{wang2026fastcatchup} study one-pass kernel
regression through continuous-time functional scaling laws and derive
task-dependent increasing or late-switch batch schedules. Our discrete rule
is based directly on final-risk influence at a fixed horizon and also applies
to multi-pass data reuse; it does not by itself imply a monotone or
single-peaked schedule. The proof is in
Appendix~\ref{subsec:app-proof-influence-matched-allocation}.

\subsection{Proof sketch}
The common idea is to examine each spectral direction and ask two questions:
how much signal is learned over the entire trajectory, and how much noise
injected at iteration \(t\) remains at the final iterate. WSD gives different
answers. Its cumulative stepsize controls learning and bias, while the decay
phase rapidly forgets earlier noise, leaving a shorter stochastic-memory
horizon.

For one-pass SGD, write the error as its mean plus a centered fluctuation,
\(e_t=m_t+\delta_t\). The mean follows population GD, so its final filter uses
the full WSD trajectory. By contrast, a batch of size \(B_t\) injects noise of
order \(1/B_t\), which is then contracted by all subsequent updates. The
quantity \(\kappa_t^{(T)}\) measures exactly how much of this time-\(t\) noise
survives in the final risk. Summing over independent batches therefore produces the schedule dependence \(B_{\mathrm{eff}}\). Separate covariance- and label-noise
comparisons give matching upper and lower bounds; the WSD decay phase then
determines the shorter variance horizon.

For multi-pass SGD, we instead subtract the full-batch GD trajectory and study
\(\Delta_t=u_t^{\mathrm{mp}}-\theta_t\). Conditioned on the dataset, each new
batch produces a centered perturbation. Without-replacement sampling scales
its covariance by \(\rho_t\), and future updates again damp it according to
\(\kappa_t^{(L)}\). Consequently, the entire dynamic schedule enters through \(\sum_{t=1}^L\rho_t\kappa_t^{(L)}\). 
This also explains why a full batch contributes no fluctuation. Covariance
comparison and martingale arguments turn this intuition into the two-sided
fluctuation estimates, up to the stated logarithmic factor.

Finally, the power-law spectrum and source condition convert these spectral
filters into the rates in Theorems~\ref{thm:onepass-batch-sgd-scaling} and
\ref{thm:multipass-batch-sgd-scaling}. Thus the proofs separate cleanly into a
full-horizon deterministic filter and an influence-weighted stochastic filter
with shorter memory.

\section{Experiments}
\label{sec:experiments}
\raggedbottom

We test whether final-risk influence organizes dynamic batch schedules and
whether the oracle allocation predicted by
Corollary~\ref{cor:influence-matched-allocation} improves on simple heuristic
schedules. We simulate the sketched Gaussian linear model with
\(a=2\), \(b=1.5\), \(d=10^4\), \(M=64\), \(N=512\),
\(\sigma=1\), and \(\gamma=0.05\). The sketch and target are fixed, while
the dataset and optimization randomness are refreshed over \(300\)
repetitions. Full implementation details are in
Appendix~\ref{sec:exp-setup}.

The first three experiments compare four schedules at each sample-processing
budget: the oracle square-root allocation
\(B_t\propto\sqrt{\kappa_t}\), a constant allocation, an
early-large/late-small allocation, and an early-small/late-large allocation.
All schedules have the same number of updates and total budget. For one-pass SGD, we fix \(T=64\), vary \(N\), and measure final-iterate variance. For multi-pass SGD, we fix \(N=L=512\), vary
\(D_{\mathrm{grad}}=\sum_tB_t\), and measure fluctuation errors.

\begin{figure}[t]
    \centering
    \begin{minipage}[t]{0.48\textwidth}
        \centering
        \includegraphics[width=\linewidth]{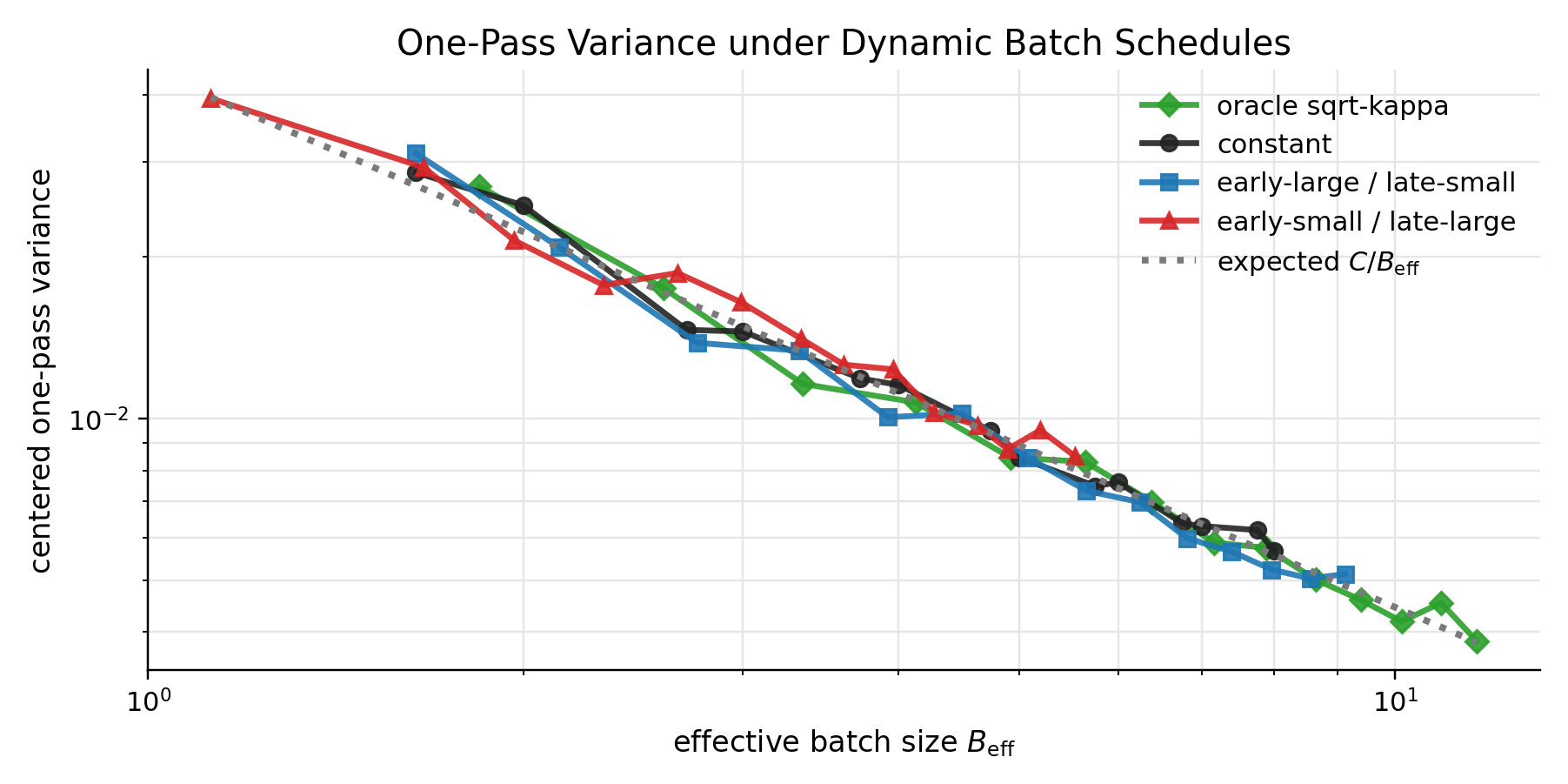}
        \par\smallskip
        {\small (a) One-pass \(B_{\mathrm{eff}}\) collapse\par}
    \end{minipage}
    \hfill
    \begin{minipage}[t]{0.48\textwidth}
        \centering
        \includegraphics[width=\linewidth]{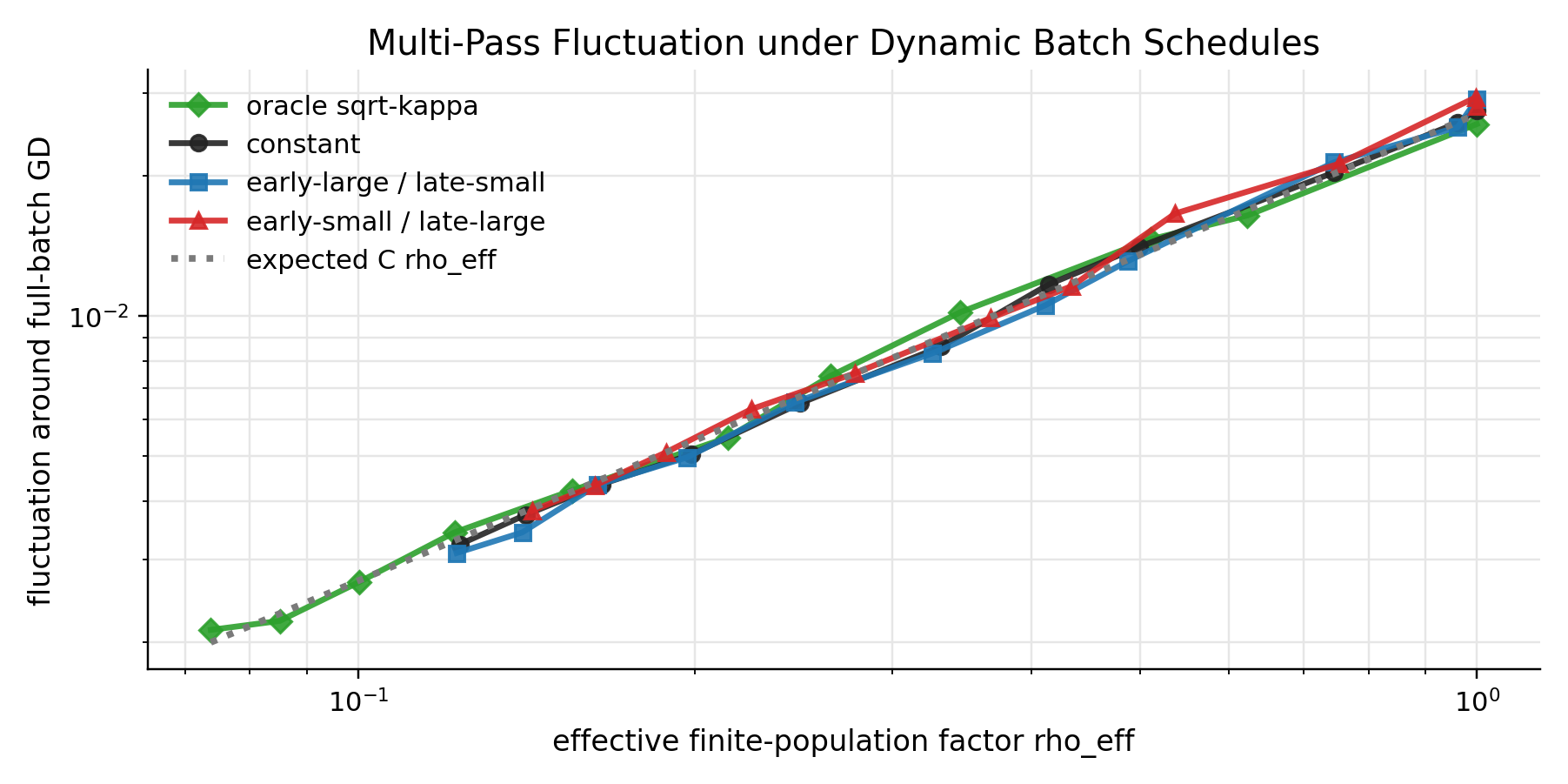}
        \par\smallskip
        {\small (b) Multi-pass \(\rho_{\mathrm{eff}}\) collapse\par}
    \end{minipage}

    \medskip

    \begin{minipage}[t]{0.48\textwidth}
        \centering
        \includegraphics[width=\linewidth]{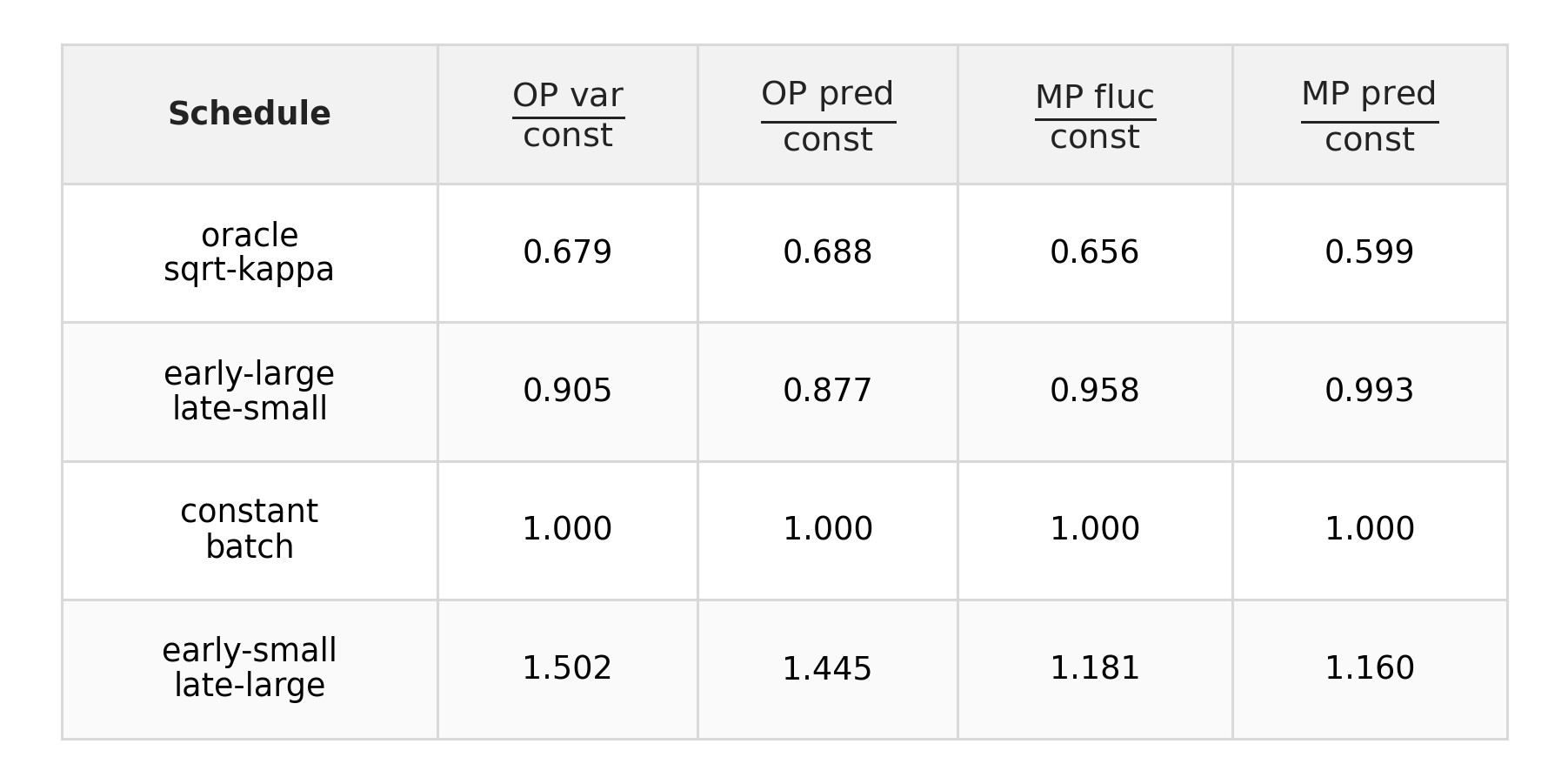}
        \par\smallskip
        {\small (c) Relative schedule comparison\par}
    \end{minipage}
    \hfill
    \begin{minipage}[t]{0.48\textwidth}
        \centering
        \includegraphics[width=\linewidth]{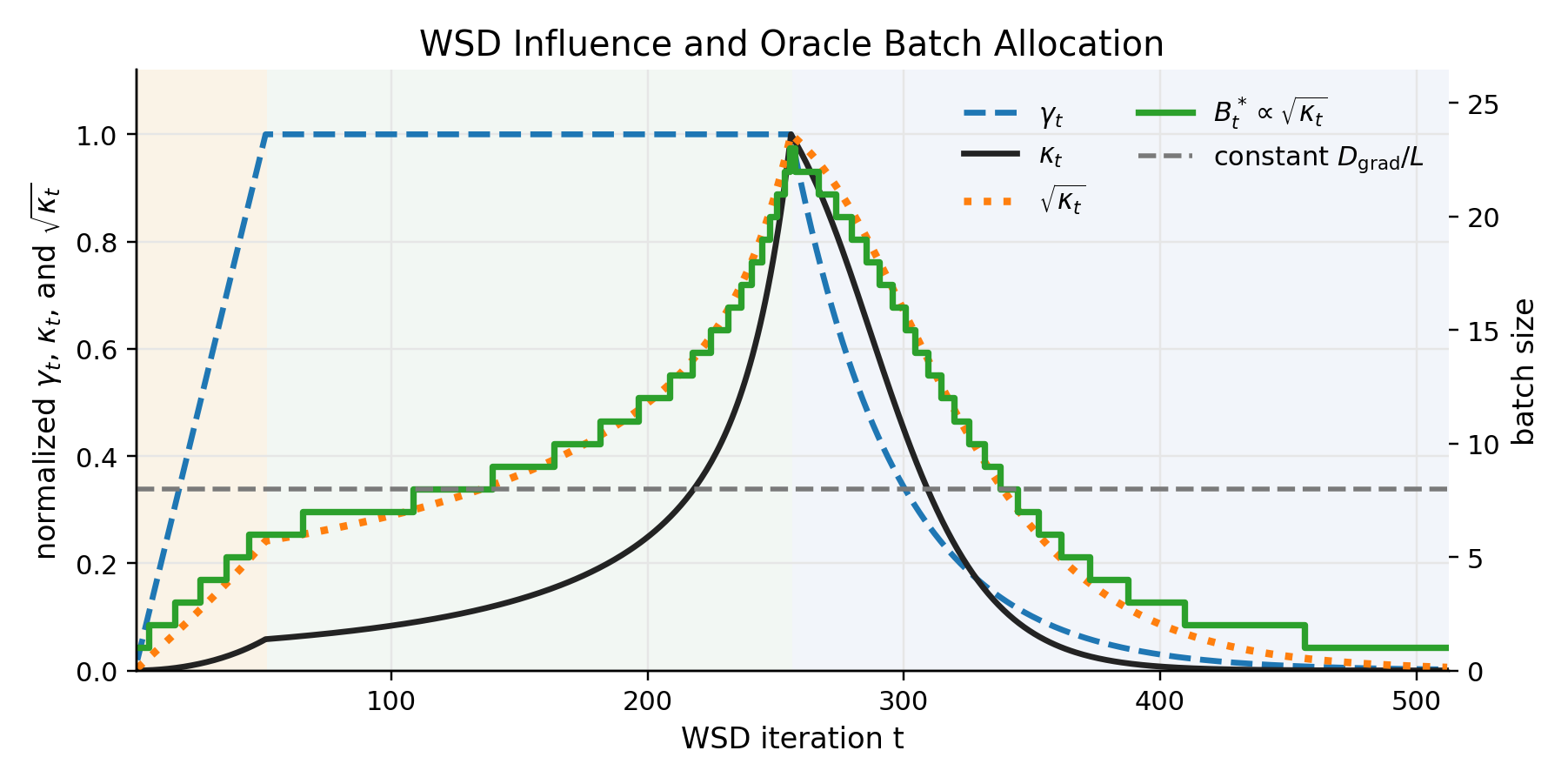}
        \par\smallskip
        {\small (d) WSD influence and oracle allocation\par}
    \end{minipage}
    \caption{Dynamic scheduling and oracle allocation. Panel (a) plots
    centered one-pass variance against \(B_{\mathrm{eff}}\), with a fitted
    \(B_{\mathrm{eff}}^{-1}\) reference. Panel (b) plots multi-pass
    fluctuation against \(\rho_{\mathrm{eff}}\), with a fitted linear
    reference. Panel (c) compares measured errors and influence-weighted
    predictions relative to the constant schedule. Panel (d) shows the WSD
    stepsize, influence profiles, and oracle allocation for \(L=512\) and
    \(D_{\mathrm{grad}}=4096\).}
    \label{fig:experiments}
\end{figure}

Figure~\ref{fig:experiments}(a) shows that all four one-pass schedule
families follow the predicted inverse trend when indexed by
\(B_{\mathrm{eff}}\). Figure~\ref{fig:experiments}(b) gives the
corresponding multi-pass collapse: fluctuation is approximately proportional
to \(\rho_{\mathrm{eff}}\). Thus schedules with different allocations become
comparable after weighting their batches by final-risk influence.

Experiment 3 compares the schedules at the largest displayed budgets,
\(N=512\) for one pass and \(D_{\mathrm{grad}}=4096\) for multi-pass. Relative
to the constant schedule, the oracle achieves one-pass variance \(0.679\) and
multi-pass fluctuation \(0.656\), close to the predicted ratios \(0.688\) and
\(0.599\). In contrast, the early-small/late-large schedule increases these
errors to \(1.502\) and \(1.181\), compared with predictions \(1.445\) and
\(1.160\). The full comparison in Figure~\ref{fig:experiments}(c)
therefore supports both the influence-weighted laws and the square-root oracle.

Experiment 4 explains the oracle's shape. Although the stepsize is constant
through most of the stable phase, \(\kappa_t\) increases toward its end because
later perturbations undergo less subsequent contraction. The influence peaks
at iteration \(256\), where the oracle batch also reaches its maximum of
\(23\), compared with the constant batch size \(8\). During decay, both
\(\sqrt{\kappa_t}\) and the oracle batch decrease. This confirms that the
optimal schedule follows final-risk influence rather than the learning rate
alone and need not be globally monotone.

\flushbottom

\section{Conclusion}

We established dynamic mini-batch scaling laws for one-pass and multi-pass
sketched SGD. The central finding is that WSD separates deterministic learning
from a shorter stochastic-memory horizon. At a fixed update horizon, batch
allocation leaves the deterministic risk contributions unchanged and affects
only one-pass variance or multi-pass fluctuation. These effects are summarized
by final-risk influence weights, which also yield a common square-root rule for
allocating a fixed sample-processing budget. Static batches, single-sample
updates, and full-batch gradient descent arise as special cases, and the
experiments support the predicted batch dependence.

Natural next steps include estimating the influence weights from data,
optimizing the batch schedule jointly with the training horizon, and testing
whether the two-horizon principle persists with momentum, adaptive methods,
and nonlinear feature-learning models.

\newpage
\bibliography{ref}
\bibliographystyle{iclr2027_conference}

\clearpage

\appendix
\counterwithin{mytheorem}{section}
\begin{center}
	\textbf{Scaling Laws for Dynamic Mini-Batch SGD in Sketched Linear Regression\\ \vspace{8pt} \textit{Supplementary Material}}
\end{center}

\counterwithin{figure}{section}
\counterwithin{table}{section}

\vspace*{1.1em}
\noindent{\LARGE\bfseries Table of Contents}\par
\vspace{0.4em}
\hrule
\vspace{0.8em}

\begingroup
\small
\setlength{\parindent}{0pt}
\newcommand{\apptocsection}[3]{%
  \noindent\hspace*{3.2em}%
  \makebox[1.8em][l]{\bfseries #1}%
  {\bfseries #2}\nobreak\hfill\makebox[2em][r]{\bfseries \pageref{#3}}\par%
}
\newcommand{\apptocsubsection}[3]{%
  \noindent\hspace*{5.6em}%
  \makebox[2.8em][l]{#1}%
  #2\nobreak\leaders\hbox to 0.55em{\hss.\hss}\hfill\makebox[2em][r]{\pageref{#3}}\par%
}

\apptocsection{A}{Appendix Preliminaries}{sec:app-prelim}
\apptocsubsection{A.1}{One-pass WSD and multi-pass smooth schedules}{subsec:app-smooth-schedule}
\apptocsubsection{A.2}{Spectral block notation}{subsec:app-block-notation}
\apptocsubsection{A.3}{Notations and setup formulas for normal GD}{subsec:app-gd-setup}
\apptocsubsection{A.4}{Notations and setup formulas for one-pass batch SGD}{subsec:app-onepass-batch-sgd-setup}
\apptocsubsection{A.5}{Notations and setup formulas for multi-pass batch SGD}{subsec:app-multipass-wor-setup}
\apptocsubsection{A.6}{Proof of the main risk decomposition}{subsec:app-proof-risk-decomp}
\apptocsubsection{A.7}{Influence-matched batch allocation}{subsec:app-proof-influence-matched-allocation}
\apptocsubsection{A.8}{Power-law interpretation of final-risk influence}{subsec:app-final-risk-influence}
\apptocsubsection{A.9}{Verification of derived multi-pass regularity}{subsec:app-derived-multipass-regularity}

\vspace{0.9em}
\apptocsection{B}{Assembling the proofs of the main scaling theorems}{sec:app-main-theorem-proofs}
\apptocsubsection{B.1}{Assembling the proof of the one-pass theorem}{subsec:app-proof-onepass}
\apptocsubsection{B.2}{Assembling the proof of the multi-pass theorem}{subsec:app-proof-multipass}

\vspace{0.9em}
\apptocsection{C}{Approximation Error}{sec:approx}
\apptocsubsection{C.1}{Upper bound}{subsec:approx-upper}
\apptocsubsection{C.2}{Lower bound}{subsec:approx-lower}
\apptocsubsection{C.3}{Bounds under our assumptions}{subsec:approx-ours}

\vspace{0.9em}
\apptocsection{D}{Bias Error under Normal GD}{sec:bias-gd}
\apptocsubsection{D.1}{Upper and lower bounds}{subsec:bias-gd-bounds}
\apptocsubsection{D.2}{Example under the source condition}{subsec:bias-gd-source}

\vspace{0.9em}
\apptocsection{E}{Variance Error under Normal GD}{sec:var-gd}
\apptocsubsection{E.1}{Upper and lower bounds}{subsec:var-gd-bounds}
\apptocsubsection{E.2}{Example under the source condition}{subsec:var-gd-source}

\vspace{0.9em}
\apptocsection{F}{One-pass Batch SGD: Excess Error Decomposition}{sec:onepass-batch-sgd-excess}
\apptocsection{G}{Bias Error for One-pass Batch SGD}{sec:batch-bias}
\apptocsubsection{G.1}{Upper and lower bounds for the bias term}{subsec:batch-bias-bounds}
\apptocsubsection{G.2}{Bounds under the source condition}{subsec:batch-bias-source}

\vspace{0.9em}
\apptocsection{H}{Variance Error for One-pass Batch SGD}{sec:batch-variance}
\apptocsubsection{H.1}{Time-dependent batch moments and the raw kernel}{subsec:batch-variance-raw}
\apptocsubsection{H.2}{The additive-noise component}{subsec:batch-modified-C1}
\apptocsubsection{H.3}{The covariance-fluctuation component}{subsec:batch-centered-covariance-component}
\apptocsubsection{H.4}{Bounds under the source condition}{subsec:batch-variance-source}
\apptocsubsection{H.5}{The WSD influence kernel}{subsec:batch-variance-spectral}

\vspace{0.9em}
\apptocsection{I}{Auxiliary with-replacement fluctuation analysis}{sec:multipass-batch-sgd-wr-fluc}
\apptocsubsection{I.1}{Conditional covariance identities}{subsec:batch-upper-bound}
\apptocsubsection{I.2}{A time-inhomogeneous stochastic approximation lemma}{subsec:batch-D2-analogue}
\apptocsubsection{I.3}{Influence kernels and covariance replacement}{subsec:multipass-dynamic-kernel}
\apptocsubsection{I.4}{Dynamic fluctuation bounds}{subsec:wr-batch-fluc-source}

\vspace{0.9em}
\apptocsection{J}{Fluctuation Error under Multi-pass Batch SGD without Replacement}{sec:multipass-batch-sgd-wor-fluc}

\vspace{0.9em}
\apptocsection{K}{Collected Auxiliary Lemmas}{sec:auxiliary-lemmas}
\apptocsubsection{K.1}{General concentration lemmas}{lem:aux-2025-E1}
\apptocsubsection{K.2}{Power-law auxiliary lemmas}{lem:aux-G4-E5}

\vspace{0.9em}
\apptocsection{L}{Experimental setup}{sec:exp-setup}
\apptocsection{M}{Limitations and Broader Effects}{sec:limitations-broader-effects}
\endgroup

\vspace{0.9em}
\hrule

\section{Appendix Preliminaries}
\label{sec:app-prelim}

We collect the common notation used throughout the appendix. Later appendix sections refer back to this section whenever possible, instead of restating the full stochastic-update setup each time.

Several proof ideas used throughout this appendix are borrowed from
\citet{lin2024scaling,lin2025datareuse}. We adapt those ideas to the WSD and
dynamic-batch setting developed in this paper.

\subsection{One-pass WSD and multi-pass smooth schedules}
\label{subsec:app-smooth-schedule}

For the one-pass procedure only, choose integers
\[
1\le T_{\mathrm w}<T_{\mathrm s}<T,
\qquad
T_{\mathrm d}:=T-T_{\mathrm s},
\qquad
\tau_{\mathrm d}^{\mathrm{op}}:=\frac{T_{\mathrm d}}{\log T},
\]
and use the WSD schedule
\begin{equation}
\gamma_t
=
\begin{cases}
\gamma t/T_{\mathrm w},&1\le t\le T_{\mathrm w},\\[1mm]
\gamma,&T_{\mathrm w}<t\le T_{\mathrm s},\\[1mm]
\gamma\exp(-(t-T_{\mathrm s})/\tau_{\mathrm d}^{\mathrm{op}}),
&T_{\mathrm s}<t\le T.
\end{cases}
\label{eq:app-onepass-wsd-schedule}
\end{equation}
Assume
\[
T_{\mathrm s}-T_{\mathrm w}\ge c_{\mathrm s}T,
\qquad
T-T_{\mathrm s}\ge c_{\mathrm d}T,
\]
and define
\[
A_T:=\sum_{t=1}^T\gamma_t,
\qquad
A_{\mathrm d}^{\mathrm{op}}:=\gamma\tau_{\mathrm d}^{\mathrm{op}}.
\]

\begin{lemma}[Mass and product calculus for the one-pass WSD schedule]
\label{lem:onepass-wsd-calculus}
Under the preceding phase-length conditions,
\[
A_T\asymp\gamma T,
\qquad
A_{\mathrm d}^{\mathrm{op}}\asymp\frac{\gamma T}{\log T}.
\]
If \(0\le\gamma\mu\le c_0<1\), then
\begin{equation}
\exp(-C\mu A_T)
\le
\prod_{t=1}^T(1-\gamma_t\mu)^2
\le
\exp(-c\mu A_T),
\label{eq:wsd-product-comparison}
\end{equation}
where \(c,C>0\) depend only on \(c_0\).
\end{lemma}

\begin{proof}
Direct summation gives
\[
A_T
=
\frac{\gamma(T_{\mathrm w}+1)}2
+\gamma(T_{\mathrm s}-T_{\mathrm w})
+\gamma\sum_{r=1}^{T_{\mathrm d}}e^{-r/\tau_{\mathrm d}^{\mathrm{op}}}.
\]
The stable-phase lower bound and \(\gamma_t\le\gamma\) imply
\(A_T\asymp\gamma T\). Since
\(e^{-T_{\mathrm d}/\tau_{\mathrm d}^{\mathrm{op}}}=T^{-1}\) and
\(1-e^{-1/\tau_{\mathrm d}^{\mathrm{op}}}\asymp(\tau_{\mathrm d}^{\mathrm{op}})^{-1}\), the final sum is
\(\asymp\tau_{\mathrm d}^{\mathrm{op}}\), proving the mass claims. Finally, summing
\(-u/(1-c_0)\le\log(1-u)\le-u\) with
\(u=\gamma_t\mu\) proves \eqref{eq:wsd-product-comparison}.
\end{proof}

For normal GD and multi-pass batch SGD, choose
\[
1\le L_{\mathrm w}<L_{\mathrm s}<L,
\qquad
L_{\mathrm d}:=L-L_{\mathrm s},
\qquad
\tau_{\mathrm d}^{\mathrm{mp}}:=\frac{L_{\mathrm d}}{\log L},
\]
and use the common WSD schedule
\begin{equation}
\gamma_t
=
\begin{cases}
\gamma t/L_{\mathrm w},&1\le t\le L_{\mathrm w},\\[1mm]
\gamma,&L_{\mathrm w}<t\le L_{\mathrm s},\\[1mm]
\gamma\exp(-(t-L_{\mathrm s})/\tau_{\mathrm d}^{\mathrm{mp}}),
&L_{\mathrm s}<t\le L.
\end{cases}
\label{eq:app-multipass-wsd-schedule}
\end{equation}
Assume
\[
L_{\mathrm s}-L_{\mathrm w}\ge c_{\mathrm s}L,
\qquad
L-L_{\mathrm s}\ge c_{\mathrm d}L,
\]
and define
\[
A_L:=\sum_{t=1}^L\gamma_t,
\qquad
A_{\mathrm d}^{\mathrm{mp}}:=\gamma\tau_{\mathrm d}^{\mathrm{mp}}.
\]

\begin{lemma}[Mass and product calculus for the multi-pass WSD schedule]
\label{lem:multipass-wsd-calculus}
Under the preceding phase-length conditions,
\[
A_L\asymp\gamma L,
\qquad
A_{\mathrm d}^{\mathrm{mp}}
\asymp\frac{\gamma L}{\log L}.
\]
If \(0\le\gamma\mu\le c_0<1\), then
\begin{equation}
\exp(-C\mu A_L)
\le
\prod_{t=1}^L(1-\gamma_t\mu)^2
\le
\exp(-c\mu A_L),
\label{eq:multipass-wsd-product-comparison}
\end{equation}
where \(c,C>0\) depend only on \(c_0\).
\end{lemma}

\begin{proof}
The proof is the same direct calculation as for
Lemma~\ref{lem:onepass-wsd-calculus}, with \(T\), \(T_{\mathrm w}\),
and \(T_{\mathrm s}\) replaced by \(L\), \(L_{\mathrm w}\), and
\(L_{\mathrm s}\). In particular, the stable phase supplies a constant
fraction of \(\gamma L\), the decay sum is
\(\asymp\gamma\tau_{\mathrm d}^{\mathrm{mp}}\), and summing the scalar
logarithmic bounds proves \eqref{eq:multipass-wsd-product-comparison}.
\end{proof}

\subsection{Spectral block notation}
\label{subsec:app-block-notation}

For integers \(0\le k_\ast \le k\) (allowing \(k=\infty\)), define
\[
H_{k_\ast:k}:=\operatorname{diag}(\lambda_{k_\ast+1},\dots,\lambda_k),
\qquad
w_{k_\ast:k}:=(w_{k_\ast+1},\dots,w_k)^\top.
\]
Similarly, let \(S_{k_\ast:k}\) denote the submatrix of \(S\) consisting of columns \(k_\ast+1,\dots,k\).

\subsection{Notations and setup formulas for normal GD}
\label{subsec:app-gd-setup}

This subsection records the notation for the normal GD procedure
\eqref{eq:proc-normal-gd}. Define
\[
\Sigma:=SHS^\top,
\qquad
\widehat{\Sigma}:=\frac{SX^\top X S^\top}{N},
\qquad
u^\ast:=\Sigma^{-1}SHw^\ast,
\qquad
\widehat{b}:=\frac{1}{N}SX^\top y.
\]
Let
\[
\widetilde{\varepsilon}_i:=y_i-x_i^\top S^\top u^\ast,
\qquad
\widetilde{\varepsilon}:=(\widetilde{\varepsilon}_1,\dots,\widetilde{\varepsilon}_N)^\top,
\qquad
\widehat{c}:=\frac{1}{N}SX^\top\widetilde{\varepsilon}.
\]
The normal GD iterate satisfies
\[
\theta_t
=
\theta_{t-1}
-
\gamma_t\widehat{\Sigma}\theta_{t-1}
+
\gamma_t\widehat{b},
\qquad
\theta_0=0.
\]

\subsection{Notations and setup formulas for one-pass batch SGD}
\label{subsec:app-onepass-batch-sgd-setup}

This subsection records the notation for the one-pass batch SGD procedure
\eqref{eq:proc-onepass-batch-sgd}. Fix positive integers
\(B_1,\dots,B_T\) satisfying \(\sum_{t=1}^T B_t=N\), and define
\[
T\ge 3,
\qquad
\Sigma:=SHS^\top,
\qquad
u^\ast:=\Sigma^{-1}SHw^\ast.
\]
Partition \([N]\) into disjoint batches \(I_1,\dots,I_T\) with
\(|I_t|=B_t\), and write each block as
\[
I_t=\{i_{t,1},\dots,i_{t,B_t}\}.
\]
Define
\[
\widehat{\Sigma}_t^{\mathrm{op}}
:=
\frac{1}{B_t}\sum_{i\in I_t} Sx_i x_i^\top S^\top,
\qquad
\widehat{b}_t^{\mathrm{op}}
:=
\frac{1}{B_t}\sum_{i\in I_t} Sx_i y_i,
\]
and
\[
\widehat{\xi}_t
:=
\frac{1}{B_t}\sum_{i\in I_t} Sx_i\bigl(y_i-x_i^\top S^\top u^\ast\bigr).
\]
The one-pass batch SGD iterate \((u_t^{\mathrm{op}})\) satisfies
\[
u_t^{\mathrm{op}}
=
u_{t-1}^{\mathrm{op}}
-
\gamma_t\widehat{\Sigma}_t^{\mathrm{op}}u_{t-1}^{\mathrm{op}}
+
\gamma_t\widehat{b}_t^{\mathrm{op}},
\qquad
u_0^{\mathrm{op}}=0.
\]
With the centered error \(e_t:=u_t^{\mathrm{op}}-u^\ast\), this becomes
\[
e_t
=
\bigl(I-\gamma_t\widehat{\Sigma}_t^{\mathrm{op}}\bigr)e_{t-1}
+
\gamma_t\widehat{\xi}_t.
\]
For later use, we also write
\[
z_{t,b}:=Sx_{i_{t,b}},
\qquad
\bar Z_t:=\frac{1}{B_t}\sum_{b=1}^{B_t} z_{t,b}z_{t,b}^\top,
\qquad
\bar \xi_t:=\frac{1}{B_t}\sum_{b=1}^{B_t} z_{t,b}\bigl(y_{i_{t,b}}-z_{t,b}^\top u^\ast\bigr),
\]
so that equivalently
\[
e_t=(I-\gamma_t\bar Z_t)e_{t-1}+\gamma_t\bar \xi_t.
\]
Since the blocks are disjoint subsets of an i.i.d. sample, the pairs
\((\bar Z_t,\bar \xi_t)\) are independent across \(t\). The mean error
\(m_t:=\mathbb{E}[e_t]\) therefore satisfies
\[
m_t=(I-\gamma_t\Sigma)m_{t-1},
\qquad
m_0=-u^\ast.
\]
Let \(\mu_1,\dots,\mu_M\) denote the eigenvalues of \(\Sigma\). Define
\[
\kappa_t^{\mathrm{op}}
:=
\gamma_t^2
\sum_{j=1}^M\mu_j^2
\prod_{s=t+1}^T(1-\gamma_s\mu_j)^2
=
\gamma_t^2
\left\langle
\Sigma,\,
\left[\prod_{s=t+1}^T(I-\gamma_s\Sigma)^2\right]\Sigma
\right\rangle
\]
and
\[
\mathcal K_{B_{1:T}}
:=
\sum_{t=1}^T\frac{\kappa_t^{\mathrm{op}}}{B_t},
\qquad
\frac{1}{B_{\mathrm{eff}}}
:=
\frac{\mathcal K_{B_{1:T}}}{\sum_{t=1}^T\kappa_t^{\mathrm{op}}}.
\]
Thus \(B_{\mathrm{eff}}\) is a weighted harmonic mean of the batch sizes.
In particular,
\[
B_{\min}:=\min_tB_t,
\qquad
B_{\max}:=\max_tB_t,
\qquad
B_{\min}\le B_{\mathrm{eff}}\le B_{\max},
\qquad
B_t\equiv B\ \Longrightarrow\ B_{\mathrm{eff}}=B.
\]
Throughout the one-pass appendix sections, \((\gamma_t)_{t=1}^T\) denotes
the WSD schedule in \eqref{eq:app-onepass-wsd-schedule}; \(A_T\) controls its
bias filter and \(A_{\mathrm d}^{\mathrm{op}}=\gamma\tau_{\mathrm d}^{\mathrm{op}}\) controls its variance
kernel. The multi-pass notation below continues to use the unchanged smooth
geometric schedule.

\subsection{Notations and setup formulas for multi-pass batch SGD}
\label{subsec:app-multipass-wor-setup}

This subsection records the notation for the main-text multi-pass batch SGD
procedure in \eqref{eq:proc-multipass-batch-sgd}, which samples without
replacement within every mini-batch. The main text uses the superscript
\(\mathrm{mp}\); only in the technical fluctuation proof do we write
\(\mathrm{wor}\) to distinguish this iterate from an auxiliary
with-replacement process. Fix a deterministic schedule \(B_1,\dots,B_L\)
satisfying \(1\le B_t\le N\), and define
\[
\Sigma:=SHS^\top,
\qquad
\widehat{\Sigma}:=\frac{SX^\top X S^\top}{N},
\qquad
u^\ast:=\Sigma^{-1}SHw^\ast,
\qquad
\widehat{b}:=\frac{1}{N}SX^\top y.
\]
Let
\[
\widetilde{\varepsilon}_i:=y_i-x_i^\top S^\top u^\ast,
\qquad
\widetilde{\varepsilon}:=(\widetilde{\varepsilon}_1,\dots,\widetilde{\varepsilon}_N)^\top,
\qquad
\widehat{c}:=\frac{1}{N}SX^\top\widetilde{\varepsilon},
\qquad
\mathcal D_{\mathrm{data}}:=\sigma(S,w^\ast,D).
\]
At each step \(t\in[L]\), sample a subset
\[
I_t\subset[N],
\qquad
|I_t|=B_t,
\]
uniformly without replacement from \([N]\), independently across iterations.
Define
\[
\widehat{\Sigma}_t^{\mathrm{wor}}
:=
\frac{1}{B_t}\sum_{i\in I_t} Sx_i x_i^\top S^\top,
\qquad
\widehat{b}_t^{\mathrm{wor}}
:=
\frac{1}{B_t}\sum_{i\in I_t} Sx_i y_i,
\]
and
\[
\widehat{c}_t^{\mathrm{wor}}
:=
\frac{1}{B_t}\sum_{i\in I_t} Sx_i\widetilde{\varepsilon}_i.
\]
The multi-pass batch SGD iterate without replacement is
\[
u_t^{\mathrm{wor}}
=
u_{t-1}^{\mathrm{wor}}
-
\gamma_t\widehat{\Sigma}_t^{\mathrm{wor}}u_{t-1}^{\mathrm{wor}}
+
\gamma_t\widehat{b}_t^{\mathrm{wor}},
\qquad
u_0^{\mathrm{wor}}=0.
\]
Define the multi-pass fluctuation process
\[
\Delta_t^{\mathrm{wor}}:=u_t^{\mathrm{wor}}-\theta_t,
\qquad
\Delta_0^{\mathrm{wor}}=0.
\]
Then
\[
\Delta_t^{\mathrm{wor}}
=
\bigl(I-\gamma_t\widehat{\Sigma}_t^{\mathrm{wor}}\bigr)\Delta_{t-1}^{\mathrm{wor}}
+
\gamma_t\xi_t^{\mathrm{wor}},
\]
where
\[
\xi_t^{\mathrm{wor}}
:=
-
\bigl(\widehat{\Sigma}_t^{\mathrm{wor}}-\widehat{\Sigma}\bigr)(\theta_{t-1}-u^\ast)
+
\bigl(\widehat{c}_t^{\mathrm{wor}}-\widehat{c}\bigr).
\]
Define
\[
\mathcal F_t^{\mathrm{wor}}
:=
\mathcal D_{\mathrm{data}}\vee\sigma(I_1,\dots,I_t).
\]
Then
\[
\mathbb E[\xi_t^{\mathrm{wor}}\mid\mathcal F_{t-1}^{\mathrm{wor}}]=0.
\]

\subsection{Proof of the main risk decomposition}
\label{subsec:app-proof-risk-decomp}
We record here the proof of the structural decomposition used later in
Section~\ref{sec:main-results}, since it only uses the basic identities from the present section.

\begin{proof}[Proof of Proposition~\ref{prop:main-risk-decomp}]
Since \(u^\ast\) minimizes the sketched risk \(R_M\), for every
\(u\in\mathbb{R}^M\) one has
\[
R_M(u)=R_M(u^\ast)+\|\Sigma^{1/2}(u-u^\ast)\|_2^2.
\]

For one-pass batch SGD, let \(\bar u_T:=\mathbb{E}[u_T^{\mathrm{op}}]\). Applying
the previous display with \(u=u_T^{\mathrm{op}}\) and taking expectation gives
\[
\mathbb{E}[R_M(u_T^{\mathrm{op}})]
=
R_M(u^\ast)+\mathbb{E}\bigl[\|\Sigma^{1/2}(u_T^{\mathrm{op}}-u^\ast)\|_2^2\bigr].
\]
Now write
\[
u_T^{\mathrm{op}}-u^\ast=(\bar u_T-u^\ast)+(u_T^{\mathrm{op}}-\bar u_T).
\]
Since \(\mathbb{E}[u_T^{\mathrm{op}}-\bar u_T]=0\), the cross term vanishes, so
\[
\mathbb{E}\bigl[\|\Sigma^{1/2}(u_T^{\mathrm{op}}-u^\ast)\|_2^2\bigr]
=
\bigl\|\Sigma^{1/2}(\bar u_T-u^\ast)\bigr\|_2^2
+
\mathbb{E}\bigl[\|\Sigma^{1/2}(u_T^{\mathrm{op}}-\bar u_T)\|_2^2\bigr].
\]
Since \(u^\ast\) minimizes \(R_M\), the first term is exactly
\[
R_M(\bar u_T)-R_M(u^\ast).
\]
Moreover,
\[
R_M(u_T^{\mathrm{op}})-R_M(\bar u_T)
=
2\bigl\langle \Sigma^{1/2}(\bar u_T-u^\ast),\Sigma^{1/2}(u_T^{\mathrm{op}}-\bar u_T)\bigr\rangle
+
\bigl\|\Sigma^{1/2}(u_T^{\mathrm{op}}-\bar u_T)\bigr\|_2^2,
\]
so taking expectation and using \(\mathbb{E}[u_T^{\mathrm{op}}-\bar u_T]=0\) gives
\[
\mathbb{E}\bigl[R_M(u_T^{\mathrm{op}})-R_M(\bar u_T)\bigr]
=
\mathbb{E}\bigl[\|\Sigma^{1/2}(u_T^{\mathrm{op}}-\bar u_T)\|_2^2\bigr].
\]
This proves the one-pass decomposition in excess-risk form.

For the main-text multi-pass iterate, identify
\(u_L^{\mathrm{mp}}=u_L^{\mathrm{wor}}\). The same identity gives
\[
\mathbb{E}[R_M(u_L^{\mathrm{mp}})]
=
R_M(u^\ast)+\mathbb{E}\bigl[\|\Sigma^{1/2}(u_L^{\mathrm{mp}}-u^\ast)\|_2^2\bigr].
\]
Writing
\[
u_L^{\mathrm{mp}}-u^\ast
=(\theta_L-u^\ast)+(u_L^{\mathrm{mp}}-\theta_L)
\]
and expanding the squared norm produces the cross term
\[
\mathbb{E}\bigl[
\langle
\Sigma^{1/2}(\theta_L-u^\ast),
\Sigma^{1/2}(u_L^{\mathrm{mp}}-\theta_L)
\rangle
\bigr].
\]
This term vanishes because the fluctuation process has zero conditional mean:
the recursion defining
\(\Delta_t^{\mathrm{wor}}=u_t^{\mathrm{wor}}-\theta_t\) has martingale-difference
driving noise with respect to the without-replacement sampling filtration, so
\[
\mathbb{E}[u_L^{\mathrm{mp}}-\theta_L\mid\mathcal D_{\mathrm{data}}]=0.
\]
Hence
\[
\mathbb{E}\bigl[\|\Sigma^{1/2}(u_L^{\mathrm{mp}}-u^\ast)\|_2^2\bigr]
=
\mathbb{E}\bigl[\|\Sigma^{1/2}(\theta_L-u^\ast)\|_2^2\bigr]
+
\mathbb{E}\bigl[\|\Sigma^{1/2}(u_L^{\mathrm{mp}}-\theta_L)\|_2^2\bigr],
\]
where the first term is
\[
\mathbb{E}\bigl[R_M(\theta_L)-R_M(u^\ast)\bigr]
\]
because \(u^\ast\) minimizes \(R_M\), and the second term is
\[
\mathbb{E}\bigl[R_M(u_L^{\mathrm{mp}})-R_M(\theta_L)\bigr]
\]
because the corresponding cross term vanishes by the same conditional-centering argument. This proves the multi-pass decomposition in excess-risk form.
\end{proof}

\subsection{Proof of the influence-matched batch-allocation corollary}
\label{subsec:app-proof-influence-matched-allocation}

\begin{proof}[Proof of Corollary~\ref{cor:influence-matched-allocation}]
In the appendix notation, \(\kappa_t^{(T)}=\kappa_t^{\mathrm{op}}\) and
\(\kappa_t^{(L)}=\kappa_t^{\mathrm{mp}}\). We work in the continuous interior
regime stated in the corollary.

\paragraph{One-pass allocation.}
For fixed \(T\), all factors in the variance law except
\(B_{\mathrm{eff}}\) are independent of the allocation, and
\[
\frac{1}{B_{\mathrm{eff}}}
=
\frac{\displaystyle\sum_{t=1}^T\kappa_t^{(T)}/B_t}
{\displaystyle\sum_{t=1}^T\kappa_t^{(T)}}.
\]
Hence it suffices to minimize \(\sum_t\kappa_t^{(T)}/B_t\) subject to
\(\sum_tB_t=N\). The Lagrange stationarity condition is
\[
-\frac{\kappa_t^{(T)}}{B_t^2}+\lambda=0,
\]
so normalization by the budget gives
\[
B_t^\star
=
\frac{N\sqrt{\kappa_t^{(T)}}}
{\displaystyle\sum_{s=1}^T\sqrt{\kappa_s^{(T)}}}.
\]

\paragraph{Multi-pass allocation.}
For without-replacement sampling,
\[
\rho_t
=
\frac{N-B_t}{B_t(N-1)}
=
\frac{N}{N-1}\frac1{B_t}-\frac1{N-1}.
\]
Consequently,
\[
\sum_{t=1}^L\rho_t\kappa_t^{(L)}
=
\frac{N}{N-1}
\sum_{t=1}^L\frac{\kappa_t^{(L)}}{B_t}
-
\frac1{N-1}\sum_{t=1}^L\kappa_t^{(L)}.
\]
The second term is allocation-independent. Applying the same stationarity
calculation under \(\sum_tB_t=D_{\mathrm{grad}}\) therefore yields
\[
B_t^\star
=
\frac{D_{\mathrm{grad}}\sqrt{\kappa_t^{(L)}}}
{\displaystyle\sum_{s=1}^L\sqrt{\kappa_s^{(L)}}}.
\]
If a lower or upper batch-size constraint is active, the KKT conditions clip
the same square-root rule. A budget-preserving rounding then gives an integer
schedule; naive coordinatewise rounding need not be exactly optimal.
\end{proof}

\subsection{Power-law interpretation of final-risk influence}
\label{subsec:app-final-risk-influence}

The exact influence weight retains the spectrum of the sketched covariance.
Under the power-law assumption, however, it admits a simpler pointwise form
in terms of the learning-rate mass remaining after an update.

\begin{lemma}[Pointwise power-law form of final-risk influence]
\label{lem:app-final-risk-influence}
Fix a horizon \(J\) and define
\[
R_t^{(J)}:=\sum_{s=t+1}^J\gamma_s,
\qquad
\kappa_t^{(J)}
:=
\gamma_t^2\sum_{j=1}^M\mu_j^2
\prod_{s=t+1}^J(1-\gamma_s\mu_j)^2.
\]
Suppose
\[
\mu_j\asymp j^{-a},
\qquad a>1,
\qquad
\gamma_s\mu_1\le c_0<1
\quad\text{for every }s\le J.
\]
There is a sufficiently small constant \(c_\star>0\), depending only on
the power-law and stability constants, such that whenever
\[
\bigl(1+R_t^{(J)}\bigr)^{1/a}\le c_\star M,
\]
we have
\begin{equation}
\kappa_t^{(J)}
\asymp
\gamma_t^2
\bigl(1+R_t^{(J)}\bigr)^{-(2-1/a)}
=
\gamma_t^2
\bigl(1+R_t^{(J)}\bigr)^{1/a-2}.
\label{eq:app-final-risk-influence-shape}
\end{equation}
Consequently, the comparison holds uniformly over \(t\) if
\[
\left(1+\sum_{s=1}^J\gamma_s\right)^{1/a}\le c_\star M.
\]
\end{lemma}

\begin{proof}
Write \(R:=R_t^{(J)}\). For \(0\le z\le c_0<1\), the scalar logarithmic
bounds
\[
-\frac{z}{1-c_0}\le\log(1-z)\le-z
\]
give constants \(c,C>0\) such that
\[
\exp(-C\mu_jR)
\le
\prod_{s=t+1}^J(1-\gamma_s\mu_j)^2
\le
\exp(-c\mu_jR).
\]
It therefore suffices to show that, for either fixed \(c'>0\),
\begin{equation}
\sum_{j=1}^M\mu_j^2e^{-c'R\mu_j}
\asymp
(1+R)^{-(2-1/a)}.
\label{eq:app-final-risk-spectral-sum}
\end{equation}

If \(R\le1\), the upper bound follows from
\(\sum_j\mu_j^2\lesssim\sum_jj^{-2a}\lesssim1\), while the \(j=1\) term
gives a constant lower bound. Now suppose \(R>1\) and put
\(k:=\lceil R^{1/a}\rceil\). For the spectral tail,
\[
\sum_{j\ge k}\mu_j^2e^{-c'R\mu_j}
\lesssim
\sum_{j\ge k}j^{-2a}
\lesssim
k^{1-2a}.
\]
For \(j<k\), divide the indices into dyadic blocks
\[
k2^{-(\ell+1)}<j\le k2^{-\ell},
\qquad \ell=0,1,\ldots.
\]
Each block contains at most \(k2^{-\ell}\) indices, while on that block
\(R\mu_j\asymp2^{a\ell}\). Hence its contribution is bounded by
\[
CkR^{-2}
2^{(2a-1)\ell}e^{-c''2^{a\ell}}.
\]
Summing over \(\ell\) gives an upper bound of order
\[
kR^{-2}\asymp R^{1/a-2}.
\]

For the lower bound, the cutoff condition ensures, after choosing
\(c_\star\) sufficiently small, that the block \(k\le j\le2k\) lies below
\(M\). On this block, \(R\mu_j\asymp1\), so every summand is of order
\(R^{-2}\). Since the block contains order \(k\) indices,
\[
\sum_{j=k}^{2k}\mu_j^2e^{-c'R\mu_j}
\gtrsim
kR^{-2}
\asymp
R^{1/a-2}.
\]
This proves \eqref{eq:app-final-risk-spectral-sum}, and multiplication by
\(\gamma_t^2\) proves
\eqref{eq:app-final-risk-influence-shape}.
\end{proof}

\paragraph{Interpretation.}
The factor \(\gamma_t^2\) measures the scale of the stochastic perturbation
injected at update \(t\). In direction \(j\), the factor \(\mu_j^2\) combines
the covariance scale of the injected gradient noise with its contribution to
prediction risk, while the product over \(s>t\) measures how much subsequent
optimization contracts that perturbation. Thus \(R_t^{(J)}\) acts as a
remaining forgetting horizon: directions with
\(\mu_jR_t^{(J)}\gg1\) are largely erased, whereas directions with
\(\mu_jR_t^{(J)}\lesssim1\) remain visible at the final iterate.

Under \(\mu_j\asymp j^{-a}\), the transition occurs near
\[
j\asymp\bigl(1+R_t^{(J)}\bigr)^{1/a}.
\]
There are order \((1+R_t^{(J)})^{1/a}\) directions near this transition,
each carrying squared spectral scale of order
\((1+R_t^{(J)})^{-2}\), which gives
\eqref{eq:app-final-risk-influence-shape}. It also yields the simpler oracle
allocation rule
\[
B_t^\star
\propto
\sqrt{\kappa_t^{(J)}}
\asymp
\gamma_t
\bigl(1+R_t^{(J)}\bigr)^{1/(2a)-1}.
\]
During the stable phase of WSD, \(\gamma_t\) is constant while
\(R_t^{(J)}\) decreases, so the influence grows toward the end of that
phase. During most of the exponential decay phase,
\(R_t^{(J)}\asymp\tau_{\mathrm d}\gamma_t\), where
\(\tau_{\mathrm d}\) denotes the WSD decay time constant for the chosen
horizon. When
\(R_t^{(J)}\gtrsim1\), this gives
\(\kappa_t^{(J)}\asymp
\tau_{\mathrm d}^{1/a-2}\gamma_t^{1/a}\), explaining why the influence
typically decreases after the stable--decay transition.

\subsection{Verification of derived multi-pass regularity}
\label{subsec:app-derived-multipass-regularity}

The following lemma discharges the trace, effective-dimension, and
maximum-sample-norm conditions that are consequences of the model.

\begin{lemma}[Derived regularity under the Gaussian power-law model]
\label{lem:app-derived-multipass-regularity}
Assume Assumptions~\ref{ass:data-gaussian} and
\ref{ass:data-power-law}, and let \(S\) be the Gaussian sketch specified in
Assumption~\ref{ass:gaussian-sketch}. Suppose
\[
0<\gamma\le\frac{c_\gamma}{\log N},
\qquad
\gamma L\log N\le c_LN,
\]
where \(c_\gamma,c_L>0\) are sufficiently small constants depending only on
\(a\). Then, for all sufficiently large \(N\), with probability at least
\(1-\exp(-\Omega(M))\) over \(S\),
\begin{equation}
\operatorname{tr}(\Sigma)\asymp1,
\qquad
\operatorname{tr}(\Sigma^2)\lesssim1,
\qquad
\gamma\le\frac{c}{\operatorname{tr}(\Sigma)},
\label{eq:app-derived-trace-regularity}
\end{equation}
and
\begin{equation}
\sum_{i=1}^M
\frac{\mu_i(\Sigma)}
{\mu_i(\Sigma)+1/(\gamma L)}
\le\frac N4.
\label{eq:app-derived-effective-dimension}
\end{equation}
Moreover, conditioned on every sketch in this event, there is an
\(a\)-dependent constant \(c_{\max}>0\) such that, for every \(r\ge1\),
\begin{equation}
\mathbb P_X\!\left(
8\max_{i\in[N]}\|Sx_i\|_2^2>\frac r\gamma
\ \middle|\ S
\right)
\le N^{-c_{\max}r}.
\label{eq:app-derived-maxnorm-tail}
\end{equation}
In particular, with conditional probability at least \(1-N^{-c_{\max}}\),
\[
\gamma\mu_1(\widehat\Sigma)
\le
\gamma\max_{i\in[N]}\|Sx_i\|_2^2
\le\frac18.
\]
\end{lemma}

\begin{proof}
On the high-probability sketch event in
Lemma~\ref{lem:aux-G4-E5},
\[
\mu_i(\Sigma)\asymp i^{-a},
\qquad i\in[M].
\]
Since \(a>1\),
\[
\operatorname{tr}(\Sigma)
\asymp\sum_{i=1}^Mi^{-a}\asymp1,
\qquad
\operatorname{tr}(\Sigma^2)
\asymp\sum_{i=1}^Mi^{-2a}\lesssim1.
\]
After reducing \(c_\gamma\) if necessary, the imposed bound
\(\gamma\le c_\gamma/\log N\) therefore also implies
\(\gamma\le c/\operatorname{tr}(\Sigma)\), proving
\eqref{eq:app-derived-trace-regularity}.

For the effective dimension, split the spectrum at
\[
k:=\left\lceil(\gamma L)^{1/a}\right\rceil.
\]
The power law gives
\[
\begin{aligned}
\sum_{i=1}^M
\frac{\mu_i(\Sigma)}
{\mu_i(\Sigma)+1/(\gamma L)}
&\le
k+\gamma L\sum_{i>k}\mu_i(\Sigma)\\
&\lesssim
1+(\gamma L)^{1/a}.
\end{aligned}
\]
The horizon condition yields
\[
(\gamma L)^{1/a}
\le
\left(\frac{c_LN}{\log N}\right)^{1/a}.
\]
Because \(a>1\), this is \(o(N)\). Choosing \(c_L\) sufficiently small
and then \(N\) sufficiently large proves
\eqref{eq:app-derived-effective-dimension}.

It remains to establish the maximum-norm tail. Conditioned on \(S\), put
\(z_i:=Sx_i\). Then \(z_i\) are independent
\(\mathcal N(0,\Sigma)\) vectors. The Gaussian quadratic-form inequality
gives, for every \(u\ge1\),
\[
\mathbb P_X\!\left(
\|z_i\|_2^2>
\operatorname{tr}(\Sigma)
+2\sqrt{\operatorname{tr}(\Sigma^2)u}
+2\|\Sigma\|_2u
\ \middle|\ S
\right)
\le e^{-u}.
\]
The trace bounds above and
\(\|\Sigma\|_2\le\operatorname{tr}(\Sigma)\lesssim1\) show that the
threshold is at most \(C(1+u)\). Taking
\(u=C_1r\log N\), applying a union bound over \(i\in[N]\), and choosing
\(c_\gamma\) sufficiently small gives
\[
\mathbb P_X\!\left(
\max_{i\in[N]}\|z_i\|_2^2>\frac r{8\gamma}
\ \middle|\ S
\right)
\le N^{-c_{\max}r},
\]
which proves \eqref{eq:app-derived-maxnorm-tail}. Finally,
\[
\widehat\Sigma
=\frac1N\sum_{i=1}^Nz_iz_i^\top
\preceq
\left(\max_{i\in[N]}\|z_i\|_2^2\right)I,
\]
so the stated empirical contraction follows on the \(r=1\) good event.
\end{proof}

\section{Assembling the proofs of the main scaling theorems}
\label{sec:app-main-theorem-proofs}

This section gathers the appendix ingredients used in the proofs of
Theorems~\ref{thm:onepass-batch-sgd-scaling} and
\ref{thm:multipass-batch-sgd-scaling}. The detailed derivations are carried
out in the subsequent appendix sections; here we simply record how those
results fit together. Throughout, we work on the intersection of the
high-probability events from the cited lemmas. Since only finitely many such
results are invoked, a union bound still yields probability
\[
1-\exp(-\Omega(M)).
\]

\subsection{Assembling the proof of Theorem~\ref{thm:onepass-batch-sgd-scaling}}
\label{subsec:app-proof-onepass}

Start from Proposition~\ref{prop:main-risk-decomp}:
\[
\mathbb{E}[R_M(u_T^{\mathrm{op}})]
=
R_M(u^\ast)
+
\bigl(R_M(\bar u_T)-R_M(u^\ast)\bigr)
+
\mathbb{E}\bigl[R_M(u_T^{\mathrm{op}})-R_M(\bar u_T)\bigr].
\]
The three terms are supplied by the appendix as follows.

\begin{enumerate}[label=\textbf{(\arabic*)}]
    \item \textbf{Common baseline risk.}
    Appendix~\ref{sec:approx} defines
    \[
    \mathrm{Approx}
    :=
    R_M(u^\ast)-R(w^\ast).
    \]
    Under Assumption~\ref{ass:data-well-specified},
    \(R(w^\ast)=\sigma^2\), and therefore
    \[
    R_M(u^\ast)=\sigma^2+\mathrm{Approx}.
    \]
    Lemma~\ref{lem:approx-final} then gives
    \[
    \mathbb{E}_{w^\ast}[\mathrm{Approx}]
    \asymp
    M^{1-b}.
    \]

    \item \textbf{One-pass bias term.}
    In Appendix~\ref{sec:onepass-batch-sgd-excess}, the mean error satisfies
    \(m_T=\bar u_T-u^\ast\), so Definition~\ref{def:batch-bias} identifies
    \[
    R_M(\bar u_T)-R_M(u^\ast)
    =
    \|\Sigma^{1/2}(\bar u_T-u^\ast)\|_2^2
    =
    \mathrm{Bias}_T.
    \]
    Lemma~\ref{lem:batch-bias-source} yields the general upper bound
    \[
    \mathbb{E}_{w^\ast}[\mathrm{Bias}_T]
    \lesssim
    \min\!\bigl\{M,A_T^{1/a}\bigr\}^{1-b}.
    \]
    Moreover, when \(A_T^{1/a}\le M/c_1\), the same
    lemma gives the matching lower bound
    \[
    \mathbb{E}_{w^\ast}[\mathrm{Bias}_T]
    \gtrsim
    A_T^{(1-b)/a}.
    \]

    \item \textbf{One-pass variance term.}
    Proposition~\ref{prop:batch-centered-bv-decomp} and
    Proposition~\ref{prop:batch-centered-var-split} identify the remaining
    term with the one-pass variance quantity \(\mathrm{Var}_{B_{1:T}}\).
    Lemma~\ref{lem:batch-var-source} then yields
    \[
    \mathbb{E}_{w^\ast}[\mathrm{Var}_{B_{1:T}}]
    \asymp
    \frac{\min\{M,(\gamma\tau_{\mathrm d}^{\mathrm{op}})^{1/a}\}}
    {\tau_{\mathrm d}^{\mathrm{op}}B_{\mathrm{eff}}}.
    \]
    Here \(B_{\mathrm{eff}}\) is the influence-weighted harmonic mean. It preserves the exact
    \(1/B_t\) contribution of each update. For a constant schedule
    \(B_t\equiv B\), one has \(B_{\mathrm{eff}}=B\) and \(T=N/B\).
\end{enumerate}

From items \textbf{(1)} and \textbf{(2)}, the deterministic contribution obeys
\[
\mathbb{E}_{w^\ast}[\mathrm{Approx}+\mathrm{Bias}_T]
=
\Theta\!\left(M^{1-b}+A_T^{(1-b)/a}\right).
\]
Indeed, when \(A_T^{1/a}\le M/c_1\), the lower bound on
\(\mathrm{Bias}_T\) gives the second scale directly; when
\(A_T^{1/a}>M/c_1\), one has
\(A_T^{(1-b)/a}\lesssim M^{1-b}\), so the approximation
lower bound already controls that term. Substituting this deterministic bound
and item \textbf{(3)} into the risk decomposition gives the dynamic-schedule
display in Theorem~\ref{thm:onepass-batch-sgd-scaling}. Substituting
\(A_T\asymp\gamma T\), \(B_{\mathrm{eff}}=B\), and \(T=N/B\) gives its
static-batch display.

\subsection{Assembling the proof of Theorem~\ref{thm:multipass-batch-sgd-scaling}}
\label{subsec:app-proof-multipass}

Again start from Proposition~\ref{prop:main-risk-decomp}:
\[
\mathbb{E}[R_M(u_L^{\mathrm{mp}})]
=
R_M(u^\ast)
+
\mathbb{E}\bigl[R_M(\theta_L)-R_M(u^\ast)\bigr]
+
\mathbb{E}\bigl[R_M(u_L^{\mathrm{mp}})-R_M(\theta_L)\bigr].
\]
The three contributions are assembled as follows.

\begin{enumerate}[label=\textbf{(\arabic*)}]
    \item \textbf{Common baseline risk.}
    As in the one-pass proof above,
    \[
    R_M(u^\ast)=\sigma^2+\mathrm{Approx},
    \qquad
    \mathbb{E}_{w^\ast}[\mathrm{Approx}]\asymp M^{1-b}
    \]
    by Assumption~\ref{ass:data-well-specified} and
    Lemma~\ref{lem:approx-final}.

    \item \textbf{Common GD-reference contribution.}
    Sections~\ref{sec:bias-gd} and \ref{sec:var-gd} control the deterministic
    and stochastic pieces of the normal-GD reference iterate \(\theta_L\).
    Their source-condition conclusions are Lemmas~\ref{lem:gd-bias-source}
    and \ref{lem:gd-var-source}. Since the multi-pass WSD mass satisfies
    \(A_L\asymp\gamma L\), the theorem's explicit condition
    \((\gamma L)^{1/a}\le c_1M\), with \(c_1\) chosen sufficiently small,
    implies the appendix unsaturated condition \(A_L^{1/a}\lesssim M\). The
    two lemmas therefore give
    \[
    \mathbb{E}_{w^\ast}[\mathrm{Bias}_{\mathrm{GD}}(w^\ast)]
    \asymp
    A_L^{(1-b)/a},
    \]
    and
    \[
    \mathrm{Var}_{\mathrm{GD}}
    \asymp
    \frac{\min\{M,A_L^{1/a}\}}{N}.
    \]

    \item \textbf{Multi-pass fluctuation term.}
    Write
    \[
    \mathcal K_{\mathrm d}
    :=
    \frac{\min\{M,(A_{\mathrm d}^{\mathrm{mp}})^{1/a}\}}
    {\tau_{\mathrm d}^{\mathrm{mp}}}.
    \]
    Corollary~\ref{cor:wor-batch-fluc}, whose proof uses the auxiliary
    with-replacement analysis in Lemma~\ref{lem:wr-batch-fluc-source}, gives
    \[
    \rho_{\mathrm{eff},\mathrm{WSD}}^{\mathrm{wor}}
    \mathcal K_{\mathrm d}
    \lesssim
    \mathbb{E}[\mathrm{Fluc}^{\mathrm{wor}}_{B_{1:L}}]
    \lesssim
    \rho_{\mathrm{eff},\mathrm{WSD}}^{\mathrm{wor}}
    \log N
    \left[1+\frac{\log^2N\,A_L^{2-s}}{N^2}\right]
    \mathcal K_{\mathrm d}.
    \]
\end{enumerate}

Thus WSD has two distinct horizons: \(A_L\asymp\gamma L\) controls GD
learning and leave-one-out fitting, whereas
\(A_{\mathrm d}^{\mathrm{mp}}\asymp\gamma L/\log L\) controls the final
stochastic-memory kernel. Combining the baseline term, the two GD-reference
orders, and the fluctuation estimate gives the claimed four-part decomposition and the sharp
fluctuation sandwich. For a static schedule, substituting
\(\rho_B=(N-B)/(B(N-1))\) gives the second display in the theorem. In the
main-text notation,
\(u_L^{\mathrm{mp}}=u_L^{\mathrm{wor}}\) and
\(\rho_{\mathrm{eff}}=\rho_{\mathrm{eff},\mathrm{WSD}}^{\mathrm{wor}}\).
\section{Approximation Error}
\label{sec:approx}

This section studies the approximation term, which is common to the three
optimization procedures \eqref{eq:proc-normal-gd},
\eqref{eq:proc-onepass-batch-sgd}, and
\eqref{eq:proc-multipass-batch-sgd}.

Retaining the common setup from Section~\ref{sec:prelim}, define
\[
\mathrm{Approx}
:=
\min_{u\in\mathbb{R}^M} R_M(u)-\min_{w\in\mathcal{H}}R(w)
=
R_M(u^\ast)-R(w^\ast)
\]
which further yields, by substituting the value of $u^*$:
\begin{equation}
\mathrm{Approx}
=
\bigl\|(I-H^{1/2}S^\top\Sigma^{-1}SH^{1/2})H^{1/2}w^\ast\bigr\|_2^2.
\label{eq:approx-basic-form}
\end{equation}

\begin{assumption}[Gaussian sketching]
\label{ass:gaussian-sketch}
The sketching operator \(S:\mathcal{H}\to\mathbb{R}^M\) is Gaussian, meaning
that in the diagonal coordinates of \(H\), its entries are i.i.d.
distributed as \(\mathcal{N}(0,1/M)\).
\end{assumption}

\begin{assumption}[Source-condition regime]
\label{ass:approx-source-regime}
In addition to Assumptions~\ref{ass:data-power-law} and
\ref{ass:source-diag}, we work in the regime
\[
1<b<a+1.
\]
\end{assumption}

\subsection{Upper bound}
\label{subsec:approx-upper}

\begin{lemma}[Upper bound on the approximation error]
\label{lem:approx-upper}
Fix any integer \(k\ge 0\) such that \(r(H)\ge k+M\), and define
\[
A_k := S_{k:\infty}H_{k:\infty}S_{k:\infty}^\top.
\]
Then the approximation error satisfies the deterministic bound
\[
\mathrm{Approx}
\lesssim
\|w^\ast_{k:\infty}\|_{H_{k:\infty}}^2
+
 w_{0:k}^{\ast\top}
\bigl(H_{0:k}^{-1}+S_{0:k}^\top A_k^{-1}S_{0:k}\bigr)^{-1}
 w_{0:k}^{\ast}.
\]
If in addition Assumption~\ref{ass:gaussian-sketch} holds and \(k\le M/2\), then with probability at least
\[
1-\exp(-\Omega(M))
\]
over the randomness of \(S\),
\[
\mathrm{Approx}
\lesssim
\|w^\ast_{k:\infty}\|_{H_{k:\infty}}^2
+
\beta_k\,\|w^\ast_{0:k}\|_2^2,
\]
where
\[
\beta_k
:=
\frac{\sum_{i>k}\lambda_i}{M}
+
\lambda_{k+1}
+
\sqrt{\frac{\sum_{i>k}\lambda_i^2}{M}}.
\]
\end{lemma}

\begin{proof}
For the deterministic identities below, it is enough to work in an eigenbasis of \(H\), so we write \(H\) in diagonal form and use the block notation from Section~\ref{subsec:app-block-notation}. When Assumption~\ref{ass:gaussian-sketch} is invoked later for the high-probability part, this is also the natural coordinate system by rotational invariance. Set
\[
\mathcal{T} := H^{1/2}S^\top\Sigma^{-1}SH^{1/2}-I.
\]
Using the block decomposition induced by the split \(0\text{:}k\) and \(k\text{:}\infty\), write
\[
\mathcal{T}
=
\begin{pmatrix}
U & V \\
V^\top & W
\end{pmatrix},
\]
where
\[
U := H_{0:k}^{1/2}S_{0:k}^\top\Sigma^{-1}S_{0:k}H_{0:k}^{1/2}-I,
\]
\[
V := H_{0:k}^{1/2}S_{0:k}^\top\Sigma^{-1}S_{k:\infty}H_{k:\infty}^{1/2},
\]
\[
W := H_{k:\infty}^{1/2}S_{k:\infty}^\top\Sigma^{-1}S_{k:\infty}H_{k:\infty}^{1/2}-I.
\]
Then \eqref{eq:approx-basic-form} gives
\[
\mathrm{Approx}
=
\|\mathcal{T} H^{1/2}w^\ast\|_2^2.
\]
By the inequality \((a+b)^2\le 2a^2+2b^2\),
\begin{align}
\mathrm{Approx}
&\le
2 w_{0:k}^{\ast\top}H_{0:k}^{1/2}(U^2+VV^\top)H_{0:k}^{1/2}w_{0:k}^{\ast}
\nonumber\\
&\qquad+
2 w_{k:\infty}^{\ast\top}H_{k:\infty}^{1/2}(W^2+V^\top V)H_{k:\infty}^{1/2}w_{k:\infty}^{\ast}.
\label{eq:approx-two-blocks}
\end{align}

We first control the tail block. Since
\[
\Sigma = S_{0:k}H_{0:k}S_{0:k}^\top + A_k,
\]
a direct calculation shows that
\begin{equation}
W^2+V^\top V = -W.
\label{eq:approx-W-identity}
\end{equation}
Moreover,
\[
0\preceq H_{k:\infty}^{1/2}S_{k:\infty}^\top\Sigma^{-1}S_{k:\infty}H_{k:\infty}^{1/2}\preceq I,
\]
so \(-I\preceq W\preceq 0\). Combining this with \eqref{eq:approx-W-identity}, we obtain
\[
0\preceq W^2+V^\top V = -W \preceq I,
\]
and therefore
\begin{equation}
 w_{k:\infty}^{\ast\top}H_{k:\infty}^{1/2}(W^2+V^\top V)H_{k:\infty}^{1/2}w_{k:\infty}^{\ast}
\le
\|w^\ast_{k:\infty}\|_{H_{k:\infty}}^2.
\label{eq:approx-tail-control}
\end{equation}

We next control the head block. Applying the Woodbury identity to
\[
\Sigma^{-1} = (S_{0:k}H_{0:k}S_{0:k}^\top + A_k)^{-1},
\]
we get
\[
\Sigma^{-1}
=
A_k^{-1}
-
A_k^{-1}S_{0:k}
\bigl(H_{0:k}^{-1}+S_{0:k}^\top A_k^{-1}S_{0:k}\bigr)^{-1}
S_{0:k}^\top A_k^{-1}.
\]
Substituting this identity into the definitions of \(U\) and \(V\), we then have
\begin{equation}
U^2+VV^\top
=
H_{0:k}^{-1/2}
\bigl(H_{0:k}^{-1}+S_{0:k}^\top A_k^{-1}S_{0:k}\bigr)^{-1}
H_{0:k}^{-1/2}.
\label{eq:approx-head-identity}
\end{equation}
Hence
\begin{equation}
 w_{0:k}^{\ast\top}H_{0:k}^{1/2}(U^2+VV^\top)H_{0:k}^{1/2}w_{0:k}^{\ast}
=
 w_{0:k}^{\ast\top}
\bigl(H_{0:k}^{-1}+S_{0:k}^\top A_k^{-1}S_{0:k}\bigr)^{-1}
 w_{0:k}^{\ast}.
\label{eq:approx-head-control}
\end{equation}
Putting \eqref{eq:approx-tail-control} and \eqref{eq:approx-head-control} into \eqref{eq:approx-two-blocks} proves the first claim.

For the high-probability bound, define
\[
\beta_k
:=
\frac{\sum_{i>k}\lambda_i}{M}
+
\lambda_{k+1}
+
\sqrt{\frac{\sum_{i>k}\lambda_i^2}{M}}.
\]
By Lemma~\ref{lem:aux-G2-E3}, with probability at least \(1-\exp(-\Omega(M))\),
\[
\|A_k\|_2 \lesssim \beta_k.
\]
Equivalently,
\[
A_k^{-1} \succeq c\,\beta_k^{-1}I
\]
for some absolute constant \(c>0\). Also, since \(k\le M/2\), a standard Gaussian covariance concentration bound gives
\[
S_{0:k}^\top S_{0:k} \succeq c_0 I_k
\]
with probability at least \(1-\exp(-\Omega(M))\), for some absolute constant
\(c_0>0\). Therefore,
\[
S_{0:k}^\top A_k^{-1}S_{0:k}
\succeq
c\,\beta_k^{-1} S_{0:k}^\top S_{0:k}
\succeq
c'\beta_k^{-1}I_k.
\]
Hence
\[
\bigl(H_{0:k}^{-1}+S_{0:k}^\top A_k^{-1}S_{0:k}\bigr)^{-1}
\preceq
\bigl(c'\beta_k^{-1}I_k\bigr)^{-1}
\lesssim
\beta_k I_k.
\]
Substituting this into the deterministic bound gives
\[
\mathrm{Approx}
\lesssim
\|w^\ast_{k:\infty}\|_{H_{k:\infty}}^2
+
\beta_k\|w^\ast_{0:k}\|_2^2,
\]
which completes the proof.
\end{proof}

\subsection{Lower bound}
\label{subsec:approx-lower}

\begin{lemma}[Lower bound on the approximation error under the source condition]
\label{lem:approx-lower}
Assume Assumption~\ref{ass:source-diag}, and let
\[
H_w := \mathbb{E}[w^\ast w^{\ast\top}].
\]
Then, conditioned on the sketch matrix \(S\),
\[
\mathbb{E}_{w^\ast}[\mathrm{Approx}]
\gtrsim
\sum_{i>M} \lambda_i i^{a-b}.
\]
In particular, under Assumptions~\ref{ass:data-power-law} and \ref{ass:source-diag},
\[
\mathbb{E}_{w^\ast}[\mathrm{Approx}] \gtrsim M^{1-b}.
\]
\end{lemma}

\begin{proof}
Reuse the block decomposition from the proof of Lemma~\ref{lem:approx-upper}. Since Assumption~\ref{ass:source-diag} implies that \(H\) and \(H_w\) are diagonal and that the coordinates of \(w^\ast\) are uncorrelated, the cross terms vanish after taking expectation over \(w^\ast\), and we obtain
\begin{align}
\mathbb{E}_{w^\ast}[\mathrm{Approx}]
&=
\operatorname{tr}\bigl((U^2+VV^\top)H_{0:k}H_{w,0:k}\bigr)
+
\operatorname{tr}\bigl((W^2+V^\top V)H_{k:\infty}H_{w,k:\infty}\bigr)
\nonumber\\
&\ge
\operatorname{tr}\bigl((W^2+V^\top V)H_{k:\infty}H_{w,k:\infty}\bigr)
\nonumber\\
&=
-\operatorname{tr}\bigl(W H_{k:\infty}H_{w,k:\infty}\bigr),
\label{eq:approx-lower-start}
\end{align}
where the last identity uses \eqref{eq:approx-W-identity}.

Define
\[
P_k
:=
I-H_{k:\infty}^{1/2}S_{k:\infty}^\top A_k^{-1}S_{k:\infty}H_{k:\infty}^{1/2}.
\]
Since
\[
\Sigma = S_{0:k}H_{0:k}S_{0:k}^\top + A_k \succeq A_k,
\]
we have
\[
\Sigma^{-1}\preceq A_k^{-1}.
\]
The matrix
\[
H_{k:\infty}^{1/2}S_{k:\infty}^\top A_k^{-1}S_{k:\infty}H_{k:\infty}^{1/2}
\]
is an orthogonal projection onto the row space induced by the tail sketch, hence
\(P_k\) is also a projection matrix. Moreover,
\[
H_{k:\infty}^{1/2}S_{k:\infty}^\top\Sigma^{-1}S_{k:\infty}H_{k:\infty}^{1/2}
\preceq
H_{k:\infty}^{1/2}S_{k:\infty}^\top A_k^{-1}S_{k:\infty}H_{k:\infty}^{1/2},
\]
so
\[
-W
=
I-H_{k:\infty}^{1/2}S_{k:\infty}^\top\Sigma^{-1}S_{k:\infty}H_{k:\infty}^{1/2}
\succeq
P_k.
\]
Since \(H_{k:\infty}H_{w,k:\infty}\succeq 0\), \eqref{eq:approx-lower-start} implies
\[
\mathbb{E}_{w^\ast}[\mathrm{Approx}]
\ge
\operatorname{tr}\bigl(P_k H_{k:\infty}H_{w,k:\infty}\bigr).
\]
Finally, all eigenvalues of \(P_k\) are either \(0\) or \(1\), with at most
\(M\) zeros. Applying Von Neumann's trace inequality to the last display, we obtain
\[
\mathbb{E}_{w^\ast}[\mathrm{Approx}]
\ge
\sum_{i>k+M} \mu_i(HH_w).
\]
Since Assumption~\ref{ass:source-diag} gives
\[
\mu_i(HH_w)=\lambda_i\,\mathbb{E}[(w_i^\ast)^2]\asymp i^{-b}=\lambda_i i^{a-b},
\]
we conclude that
\[
\mathbb{E}_{w^\ast}[\mathrm{Approx}]
\gtrsim
\sum_{i>k+M}\lambda_i i^{a-b}.
\]
Setting \(k=0\) gives
\[
\mathbb{E}_{w^\ast}[\mathrm{Approx}]
\gtrsim
\sum_{i>M}\lambda_i i^{a-b}.
\]
Under the power-law assumption \(\lambda_i\asymp i^{-a}\), this becomes
\[
\mathbb{E}_{w^\ast}[\mathrm{Approx}]
\gtrsim
\sum_{i>M} i^{-b}
\asymp
M^{1-b},
\]
which completes the proof.
\end{proof}

\subsection{Bounds under our assumptions}
\label{subsec:approx-ours}

\begin{lemma}[Approximation error under our assumptions]
\label{lem:approx-final}
Assume Assumptions~\ref{ass:source-diag}, \ref{ass:gaussian-sketch}, and
\ref{ass:approx-source-regime}. Then with probability at least
\[
1-\exp(-\Omega(M))
\]
over the randomness of \(S\),
\[
\mathbb{E}_{w^\ast}[\mathrm{Approx}] \asymp M^{1-b}.
\]
\end{lemma}

\begin{proof}
We first prove the upper bound. Let \(k = \lfloor M/2\rfloor\). By Lemma~\ref{lem:approx-upper},
\[
\mathbb{E}_{w^\ast}[\mathrm{Approx}]
\lesssim
\mathbb{E}_{w^\ast}\|w^\ast_{k:\infty}\|_{H_{k:\infty}}^2
+
\beta_k\,\mathbb{E}_{w^\ast}\|w^\ast_{0:k}\|_2^2
\]
with probability at least \(1-\exp(-\Omega(M))\). Under Assumptions~\ref{ass:data-power-law} and \ref{ass:source-diag},
\[
\mathbb{E}_{w^\ast}\|w^\ast_{k:\infty}\|_{H_{k:\infty}}^2
=
\sum_{i>k} \lambda_i\mathbb{E}[(w_i^\ast)^2]
\asymp
\sum_{i>k} i^{-b}
\asymp
k^{1-b}.
\]
Also, using \(\lambda_i\asymp i^{-a}\),
\[
\beta_k
\lesssim
\frac{\sum_{i>k}i^{-a}}{M}
+
 k^{-a}
+
\sqrt{\frac{\sum_{i>k}i^{-2a}}{M}}
\lesssim M^{-a}
\]
when \(k\asymp M\). Moreover,
\[
\mathbb{E}_{w^\ast}\|w^\ast_{0:k}\|_2^2
=
\sum_{i\le k} \mathbb{E}[(w_i^\ast)^2]
\asymp
\sum_{i\le k} i^{a-b}
\lesssim
k^{a-b+1},
\]
where the last step uses Assumption~\ref{ass:approx-source-regime}. Therefore,
\[
\beta_k\,\mathbb{E}_{w^\ast}\|w^\ast_{0:k}\|_2^2
\lesssim
M^{-a}k^{a-b+1}
\asymp
M^{1-b}.
\]
Since also \(k^{1-b}\asymp M^{1-b}\), this proves
\[
\mathbb{E}_{w^\ast}[\mathrm{Approx}] \lesssim M^{1-b}.
\]

For the lower bound, Lemma~\ref{lem:approx-lower} gives
\[
\mathbb{E}_{w^\ast}[\mathrm{Approx}]
\gtrsim
\sum_{i>M}\lambda_i i^{a-b}
\asymp
\sum_{i>M} i^{-b}
\asymp
M^{1-b}.
\]
Combining the two bounds yields the claim.
\end{proof}
\section{Bias Error under Normal GD}
\label{sec:bias-gd}

This section focuses on the normal GD procedure \eqref{eq:proc-normal-gd}. Similarly, we retain the notation of Section~\ref{sec:prelim}. In particular,
\[
\theta_t
=
\theta_{t-1}
-
\gamma_t\widehat{\Sigma}\theta_{t-1}
+
\gamma_t\widehat{b},
\qquad
\theta_0=0,
\]
with
\[
\Sigma:=SHS^\top,
\qquad
\widehat{\Sigma}:=\frac{1}{N}SX^\top X S^\top,
\qquad
u^\ast:=\Sigma^{-1}SHw^\ast,
\qquad
\widehat{b}:=\frac{1}{N}SX^\top y.
\]
Define
\[
\widetilde{\varepsilon}_i:=y_i-x_i^\top S^\top u^\ast,
\qquad
\widetilde{\varepsilon}:=(\widetilde{\varepsilon}_1,\dots,\widetilde{\varepsilon}_N)^\top,
\qquad
\widehat{c}:=\frac{1}{N}SX^\top\widetilde{\varepsilon},
\]
and introduce the shorthand
\[
C_L:=\prod_{t=1}^L (I-\gamma_t\widehat{\Sigma}),
\qquad
V(\widehat{\Sigma})
:=
\frac{1}{N}\sum_{t=1}^L \gamma_t\prod_{i=t+1}^L (I-\gamma_i\widehat{\Sigma}).
\]
Then
\[
\theta_L-u^\ast
=
-C_Lu^\ast+V(\widehat{\Sigma})SX^\top\widetilde{\varepsilon}.
\]
Accordingly, the GD bias term is
\[
\mathrm{Bias}_{\mathrm{GD}}(w^\ast)
:=
\mathbb{E}_X\bigl[\|\Sigma^{1/2}C_Lu^\ast\|_2^2\bigr].
\]

\subsection{Upper and lower bounds}
\label{subsec:bias-gd-bounds}

\begin{lemma}[Upper bound on the GD bias term]
\label{lem:gd-bias-upper}
Assume Assumptions~\ref{ass:data-gaussian}, \ref{ass:data-power-law},
\ref{ass:gaussian-sketch}, and \ref{ass:stepsize}, and work on the derived
regularity event in Lemma~\ref{lem:app-derived-multipass-regularity}. Suppose
\[
A_L\lesssim N^a.
\]
Fix any integer \(k\le M/3\) such that \(\operatorname{rank}(H)\ge k+M\), and
define
\[
A_k:=S_{k:\infty}H_{k:\infty}S_{k:\infty}^\top,
\qquad
\widetilde{k}:=\lceil N/2\rceil.
\]
Writing \(\Sigma=\sum_{j=1}^M\mu_jq_jq_j^\top\), define its spectral tail
\[
\Sigma_{>\widetilde{k}}
:=
\sum_{j=\widetilde{k}+1}^M\mu_jq_jq_j^\top,
\]
with the convention \(\Sigma_{>\widetilde{k}}=0\) when
\(\widetilde{k}\ge M\).
Then with probability at least
\[
1-\exp(-\Omega(M))
\]
over the randomness of \(S\),
\[
\mathrm{Bias}_{\mathrm{GD}}(w^\ast)
\lesssim
\frac{\|w^\ast_{0:k}\|_2^2}{A_L}
\left(\frac{\mu_{M/2}(A_k)}{\mu_M(A_k)}\right)^2
+
\overline{B}\,\|w^\ast_{k:\infty}\|_{H_{k:\infty}}^2,
\]
where
\[
\overline{B}
:=
1
+(A_L^2+1)
\left(
\frac{\operatorname{tr}(\Sigma_{>\widetilde{k}})^2}{N^2}
+\|\Sigma_{>\widetilde{k}}\|_2^2
+\frac{\operatorname{tr}(\Sigma_{>\widetilde{k}}^2)}{N}
+\sqrt{\frac{\operatorname{tr}(\Sigma_{>\widetilde{k}}^4)}{N}}
\right).
\]
\end{lemma}

\begin{proof}
By rotational invariance of Gaussian sketching, we may work in the diagonal coordinates of \(H\) and use the block notation from Section~\ref{subsec:app-block-notation}. Define
\[
M_L:=C_L\Sigma C_L.
\]
Substituting
\[
SH=(S_{0:k}H_{0:k},\,S_{k:\infty}H_{k:\infty})
\]
into the identity
\[
\mathrm{Bias}_{\mathrm{GD}}(w^\ast)
=
\mathbb{E}_X\bigl[w^{\ast\top}HS^\top\Sigma^{-1}M_L\Sigma^{-1}SHw^\ast\bigr]
\]
and splitting head and tail blocks gives
\[
\mathrm{Bias}_{\mathrm{GD}}(w^\ast)
\le 2T_1+2T_2,
\]
where
\[
T_1
:=
\mathbb{E}_X\bigl[w_{0:k}^{\ast\top}H_{0:k}S_{0:k}^\top\Sigma^{-1}M_L\Sigma^{-1}S_{0:k}H_{0:k}w_{0:k}^\ast\bigr],
\]
and
\[
T_2
:=
\mathbb{E}_X\bigl[w_{k:\infty}^{\ast\top}H_{k:\infty}S_{k:\infty}^\top\Sigma^{-1}M_L\Sigma^{-1}S_{k:\infty}H_{k:\infty}w_{k:\infty}^\ast\bigr].
\]

We first bound the head term. For any \(\lambda>0\), inserting
\((\widehat\Sigma+\lambda I)^{1/2}
(\widehat\Sigma+\lambda I)^{-1/2}\) gives
\[
\begin{aligned}
\|M_L\|_2
&=\|C_L\Sigma^{1/2}\|_2^2\\
&\le
\|C_L(\widehat\Sigma+\lambda I)^{1/2}\|_2^2
\|(\widehat\Sigma+\lambda I)^{-1/2}\Sigma^{1/2}\|_2^2\\
&\le
\left(\|C_L^2\widehat\Sigma\|_2+\lambda\right)
\| (\widehat\Sigma+\lambda I)^{-1/2}
(\Sigma+\lambda I)^{1/2}\|_2^2.
\end{aligned}
\]
The contraction condition and Lemma~\ref{lem:multipass-wsd-calculus} imply
\[
\begin{aligned}
\|C_L^2\widehat\Sigma\|_2
&=
\max_{j\in[M]}
\mu_j(\widehat\Sigma)
\prod_{t=1}^L(1-\gamma_t\mu_j(\widehat\Sigma))^2\\
&\lesssim
\sup_{x\ge0}xe^{-cA_Lx}
\lesssim
\frac1{A_L}.
\end{aligned}
\]
Choose \(\lambda=c_\lambda/A_L\), for a sufficiently large absolute
constant \(c_\lambda\). Since \(A_L\asymp\gamma L\), the
effective-dimension bound in
Lemma~\ref{lem:app-derived-multipass-regularity} permits the use
of Lemma~\ref{lem:aux-2025-E1}. Its fourth-moment estimate and
Cauchy--Schwarz yield
\[
\begin{aligned}
\mathbb E_X\|M_L\|_2
&\lesssim
\frac1{A_L}
\left(
\mathbb E_X
\|(\widehat\Sigma+\lambda I)^{-1/2}
(\Sigma+\lambda I)^{1/2}\|_2^4
\right)^{1/2}\\
&\lesssim
\frac1{A_L}.
\end{aligned}
\]
The exponentially small remainder in Lemma~\ref{lem:aux-2025-E1} is absorbed
using \(A_L\lesssim N^a\) and
\(\|\Sigma\|_2^2\le\operatorname{tr}(\Sigma^2)\lesssim1\).
Consequently,
\[
T_1
\le
\mathbb{E}_X\|M_L\|_2\cdot \|\Sigma^{-1}S_{0:k}H_{0:k}\|_2^2\cdot \|w_{0:k}^\ast\|_2^2.
\]
Lemma~\ref{lem:aux-head-tail-resolvent} gives
\[
\|\Sigma^{-1}S_{0:k}H_{0:k}\|_2
\lesssim
\frac{\mu_{M/2}(A_k)}{\mu_M(A_k)}
\]
with probability at least \(1-\exp(-\Omega(M))\). Hence
\[
T_1
\lesssim
\frac{\|w_{0:k}^\ast\|_2^2}{A_L}
\left(\frac{\mu_{M/2}(A_k)}{\mu_M(A_k)}\right)^2.
\]

For the tail term, define
\[
\mathcal{B}:=\mathbb{E}_X\bigl[\Sigma^{-1/2}M_L\Sigma^{-1/2}\bigr].
\]
Then
\[
T_2
\le
\|\mathcal{B}\|_2\cdot
\|H_{k:\infty}^{1/2}S_{k:\infty}^\top\Sigma^{-1}S_{k:\infty}H_{k:\infty}^{1/2}\|_2
\cdot \|w_{k:\infty}^\ast\|_{H_{k:\infty}}^2.
\]
To see that the middle norm is at most one, put
\(D_k:=S_{k:\infty}H_{k:\infty}^{1/2}\). Since
\(D_kD_k^\top=A_k\preceq\Sigma\),
\[
\|H_{k:\infty}^{1/2}S_{k:\infty}^\top
\Sigma^{-1}S_{k:\infty}H_{k:\infty}^{1/2}\|_2
=
\|\Sigma^{-1/2}D_k\|_2^2
\le1.
\]
It remains to prove \(\|\mathcal B\|_2\lesssim\overline B\).

Condition on \(S\). Let \(G\in\mathbb R^{N\times M}\) have independent
\(\mathcal N(0,1/N)\) entries. Gaussian design gives
\[
\frac{XS^\top}{\sqrt N}
\stackrel{d}{=}
G\Sigma^{1/2},
\qquad
\widehat\Sigma
\stackrel{d}{=}
\Sigma^{1/2}G^\top G\Sigma^{1/2}.
\]
Set
\[
K:=G\Sigma G^\top,
\qquad
R_L(K)
:=
\sum_{i=0}^{L-1}\gamma_{i+1}
\prod_{t=1}^{i}(I-\gamma_tK),
\]
where the empty product is the identity. The polynomial intertwining identity
between \(\widehat\Sigma\) and \(K\), followed by telescoping, gives
\[
C_L
=
I-\Sigma^{1/2}G^\top R_L(K)G\Sigma^{1/2}.
\]
Moreover, \(R_L(K)\) is a polynomial in \(K\), and the contraction condition
implies
\begin{equation}
\|R_L(K)\|_2\le\sum_{t=1}^L\gamma_t=A_L,
\qquad
\|R_L(K)K\|_2
=
\left\|I-\prod_{t=1}^L(I-\gamma_tK)\right\|_2
\le1.
\label{eq:gd-bias-residual-bounds}
\end{equation}
For brevity write \(R:=R_L(K)\). Applying
\((U-V)(U-V)^\top\preceq2UU^\top+2VV^\top\) to the representation of
\(C_L\) yields
\begin{equation}
\mathcal B
\preceq
2I+2\mathbb E_G[Q],
\qquad
Q:=G^\top RG\Sigma^2G^\top RG.
\label{eq:gd-bias-normalized-covariance}
\end{equation}

We record the Gaussian moment estimate used below. If \(E\succeq0\) is
deterministic and \(G_E\) is the corresponding column block of \(G\), then,
for every \(u\ge1\), Gaussian covariance concentration gives
\begin{equation}
\begin{aligned}
\|G_EEG_E^\top\|_2
\lesssim{}&
\frac{\operatorname{tr}(E)}{N}+\|E\|_2
+\sqrt{\frac{\operatorname{tr}(E^2)}{N}}\\
&+\frac{u\|E\|_2}{N}
+\sqrt{\frac{u\operatorname{tr}(E^2)}{N}}
\end{aligned}
\label{eq:gd-bias-gaussian-tail}
\end{equation}
with probability at least \(1-2e^{-u}\). Integrating this tail bound gives
the corresponding second and fourth moments. Since, for every unit vector
\(v\), \(Gv\sim\mathcal N(0,I_N/N)\), Cauchy--Schwarz then implies
\begin{equation}
\begin{aligned}
\left\|
\mathbb E_G[\|G_EEG_E^\top\|_2^2G^\top G]
\right\|_2
&\lesssim
\frac{\operatorname{tr}(E)^2}{N^2}
+\|E\|_2^2
+\frac{\operatorname{tr}(E^2)}{N},\\
\left\|
\mathbb E_G[\|G_EE^2G_E^\top\|_2G^\top G]
\right\|_2
&\lesssim
\frac{\operatorname{tr}(E^2)}{N}
+\|E\|_2^2
+\sqrt{\frac{\operatorname{tr}(E^4)}{N}}.
\end{aligned}
\label{eq:gd-bias-weighted-gaussian-moments}
\end{equation}
Indeed, for a nonnegative random variable \(Z\),
\[
v^\top\mathbb E_G[ZG^\top G]v
=
\mathbb E_G[Z\|Gv\|_2^2]
\le
(\mathbb E_GZ^2)^{1/2}
(\mathbb E_G\|Gv\|_2^4)^{1/2},
\]
and \(\mathbb E_G\|Gv\|_2^4\lesssim1\). Combining this observation with
the moments obtained from \eqref{eq:gd-bias-gaussian-tail} proves
\eqref{eq:gd-bias-weighted-gaussian-moments}.

We now bound \(\mathbb E_G[Q]\) in two regimes.

\emph{Case 1: \(M\le N/2\).}
With probability at least \(1-e^{-cN}\), Gaussian singular-value
concentration gives \(G^\top G\succeq cI_M\). On this event,
\[
G\Sigma^2G^\top
\preceq
c^{-1}(G\Sigma G^\top)^2
=c^{-1}K^2.
\]
Since \(R\) commutes with \(K\),
\eqref{eq:gd-bias-residual-bounds} implies
\[
Q
\preceq
c^{-1}G^\top RK^2RG
\preceq
c^{-1}G^\top G.
\]
The complement contributes only a constant: for every unit vector \(v\),
\[
v^\top Qv
\le
A_L^2\|G\Sigma^2G^\top\|_2\|Gv\|_2^2,
\]
and Cauchy--Schwarz, Gaussian moments,
\(\operatorname{tr}(\Sigma^2)\lesssim1\), and
\(A_L\lesssim N^a\) absorb the exponentially small failure probability.
Therefore \(\|\mathbb E_GQ\|_2\lesssim1\).

\emph{Case 2: \(M>N/2\).}
Let \(\widetilde k=\lceil N/2\rceil\), rotate the sketched coordinates so
that \(\Sigma\) is diagonal, and partition
\[
G=(G_{\mathrm h},G_{\mathrm t}),
\qquad
\Sigma=\operatorname{diag}(\Sigma_{\mathrm h},\Sigma_{\mathrm t}),
\qquad
\Sigma_{\mathrm t}=\Sigma_{>\widetilde k}.
\]
Define
\[
K_{\mathrm h}:=G_{\mathrm h}\Sigma_{\mathrm h}G_{\mathrm h}^\top,
\quad
K_{\mathrm t}:=G_{\mathrm t}\Sigma_{\mathrm t}G_{\mathrm t}^\top,
\quad
K_{\mathrm t,2}:=G_{\mathrm t}\Sigma_{\mathrm t}^2G_{\mathrm t}^\top.
\]
On an event of probability at least \(1-e^{-cN}\),
\(G_{\mathrm h}^\top G_{\mathrm h}\succeq cI_{\widetilde k}\), so
\[
G_{\mathrm h}\Sigma_{\mathrm h}^2G_{\mathrm h}^\top
\preceq
c^{-1}K_{\mathrm h}^2.
\]
Using \(K=K_{\mathrm h}+K_{\mathrm t}\),
\eqref{eq:gd-bias-residual-bounds}, and \(\|R\|_2\le A_L\), we obtain
\begin{equation}
\begin{aligned}
Q
&\preceq
C\left(
\|RK_{\mathrm h}\|_2^2
+A_L^2\|K_{\mathrm t,2}\|_2
\right)G^\top G\\
&\preceq
C\left(
1+A_L^2\|K_{\mathrm t}\|_2^2
+A_L^2\|K_{\mathrm t,2}\|_2
\right)G^\top G.
\end{aligned}
\label{eq:gd-bias-tail-reduction}
\end{equation}
Applying \eqref{eq:gd-bias-weighted-gaussian-moments} with
\(E=\Sigma_{\mathrm t}\), and treating the exponentially small complement
as in Case 1, gives
\begin{equation}
\begin{aligned}
\|\mathbb E_GQ\|_2
\lesssim{}&1+A_L^2
\left(
\frac{\operatorname{tr}(\Sigma_{\mathrm t})^2}{N^2}
+\|\Sigma_{\mathrm t}\|_2^2
+\frac{\operatorname{tr}(\Sigma_{\mathrm t}^2)}{N}
+\sqrt{\frac{\operatorname{tr}(\Sigma_{\mathrm t}^4)}{N}}
\right)\\
\lesssim{}&\overline B.
\end{aligned}
\label{eq:gd-bias-tail-expectation}
\end{equation}
Combining \eqref{eq:gd-bias-normalized-covariance} and
\eqref{eq:gd-bias-tail-expectation}, and recalling that
\(\overline B\ge1\), proves
\[
\|\mathcal B\|_2\lesssim\overline B.
\]
Combining the estimates for \(T_1\) and \(T_2\) proves the claim.
\end{proof}

\begin{lemma}[Lower bound on the GD bias term]
\label{lem:gd-bias-lower}
Assume Assumptions~\ref{ass:data-gaussian}, \ref{ass:gaussian-sketch}, and
\ref{ass:stepsize}. Let
\[
H_w:=\mathbb{E}_{w^\ast}[w^\ast w^{\ast\top}],
\qquad
\Sigma_w:=SHH_wHS^\top.
\]
Then with probability at least
\[
1-\exp(-\Omega(M))
\]
over the randomness of \(S\),
\[
\mathbb{E}_{w^\ast}[\mathrm{Bias}_{\mathrm{GD}}(w^\ast)]
\gtrsim
\sum_{i=\lceil 2\tau\rceil+1}^{\lfloor M/3\rfloor}
\frac{\mu_{3i}(\Sigma_w)}{\mu_i(\Sigma)},
\]
where
\[
\tau:=\mathbb{E}_X\Bigl[\#\bigl\{i\in[M]:\mu_i(\widehat{\Sigma})A_L>1/4\bigr\}\Bigr].
\]
The sum is understood to be zero when its lower endpoint exceeds
\(\lfloor M/3\rfloor\).
\end{lemma}

\begin{proof}
Set
\[
C_L:=\prod_{t=1}^L(I-\gamma_t\widehat{\Sigma}),
\qquad
A:=\mathbb E_X[C_L],
\qquad
B:=\Sigma^{-1/2}\Sigma_w\Sigma^{-1/2}.
\]
Both \(C_L\) and \(A\) are PSD: all factors defining \(C_L\) are
commuting PSD matrices under the contraction condition in
Assumption~\ref{ass:stepsize}. We have
\[
\mathbb{E}_{w^\ast}[\mathrm{Bias}_{\mathrm{GD}}(w^\ast)]
=
\operatorname{tr}\!\Bigl(
\mathbb{E}_X\bigl[\Sigma^{-1/2}C_L\Sigma C_L^\top\Sigma^{-1/2}\bigr]
B\Bigr).
\]
Moreover,
\[
\mathbb{E}_X\bigl[\Sigma^{-1/2}C_L\Sigma C_L^\top\Sigma^{-1/2}\bigr]
\succeq
\Sigma^{-1/2}A\Sigma A^\top\Sigma^{-1/2},
\]
since
\[
\mathbb{E}_X[(C_L-A)\Sigma(C_L-A)^\top]\succeq0.
\]

We next verify that \(A\) commutes with \(\Sigma\). Condition on \(S\) and
diagonalize \(\Sigma=U\Lambda U^\top\). If
\(G\in\mathbb R^{N\times M}\) has independent \(\mathcal N(0,1/N)\)
entries, Gaussian design gives
\[
\widehat\Sigma
\stackrel{d}{=}
U\Lambda^{1/2}G^\top G\Lambda^{1/2}U^\top.
\]
For any diagonal sign matrix \(D\), \(GD\stackrel{d}{=}G\) and
\(D\Lambda D=\Lambda\). Since
\[
f(Z):=\prod_{t=1}^L(I-\gamma_tZ)
\]
satisfies \(f(DZD)=Df(Z)D\), sign invariance implies
\[
U^\top A U=D(U^\top A U)D
\qquad\text{for every diagonal sign matrix }D.
\]
Every off-diagonal entry changes sign under a suitable choice of \(D\), so
\(U^\top A U\) is diagonal. Hence \(A\Sigma=\Sigma A\), and therefore
\[
\Sigma^{-1/2}A\Sigma A^\top\Sigma^{-1/2}=A^2.
\]
Because \(B\succeq0\), trace monotonicity under the preceding PSD bound gives
\begin{equation}
\mathbb{E}_{w^\ast}[\mathrm{Bias}_{\mathrm{GD}}(w^\ast)]
\ge \operatorname{tr}(A^2B).
\label{eq:gd-bias-lower-psd-reduction}
\end{equation}

It remains to identify sufficiently many eigenvalues of \(A\) that are
bounded below by a constant. For each realization of \(X\), let
\[
r_X:=\#\{j\in[M]:\mu_j(\widehat\Sigma)A_L>1/4\},
\]
and let \(P_X\) be the spectral projector of \(\widehat\Sigma\) onto these
\(r_X\) directions. On every direction orthogonal to \(P_X\), the lower
product estimate in Lemma~\ref{lem:multipass-wsd-calculus} gives
\[
\prod_{t=1}^L(1-\gamma_t\mu_j(\widehat\Sigma))
\ge
\exp\!\left(-C A_L\mu_j(\widehat\Sigma)\right)
\ge e^{-C/4}=:c_{\mathrm{wsd}}>0.
\]
Consequently,
\[
C_L\succeq c_{\mathrm{wsd}}(I-P_X),
\qquad
A\succeq c_{\mathrm{wsd}}(I-\overline P),
\qquad
\overline P:=\mathbb E_X[P_X].
\]
The matrix \(\overline P\) satisfies \(0\preceq\overline P\preceq I\) and
\[
\operatorname{tr}(\overline P)
=\mathbb E_X[\operatorname{tr}(P_X)]
=\mathbb E_X[r_X]
=\tau.
\]
Thus at most \(\lceil2\tau\rceil\) eigenvalues of \(\overline P\) can exceed
\(1/2\). It follows from PSD eigenvalue monotonicity that \(A\) has at least
\(M-\lceil2\tau\rceil\) eigenvalues no smaller than
\(c_{\mathrm{wsd}}/2\). Equivalently,
\begin{equation}
\mu_{M-i+1}(A)\ge \frac{c_{\mathrm{wsd}}}{2},
\qquad
i=\lceil2\tau\rceil+1,\ldots,M.
\label{eq:gd-bias-lower-eigenvalue-count}
\end{equation}

We can now apply the lower rearrangement form of von Neumann's trace
inequality. For two PSD matrices, it pairs the decreasing eigenvalues of one
matrix with the increasing eigenvalues of the other and yields
\[
\operatorname{tr}(A^2B)
\ge
\sum_{i=1}^M\mu_{M-i+1}(A^2)\mu_i(B)
=
\sum_{i=1}^M\mu_{M-i+1}(A)^2\mu_i(B).
\]
Combining this inequality with
\eqref{eq:gd-bias-lower-psd-reduction} and
\eqref{eq:gd-bias-lower-eigenvalue-count}, and then retaining only the first
\(\lfloor M/3\rfloor\) admissible terms, gives
\[
\mathbb{E}_{w^\ast}[\mathrm{Bias}_{\mathrm{GD}}(w^\ast)]
\gtrsim
\sum_{i=\lceil 2\tau\rceil+1}^{\lfloor M/3\rfloor}
\mu_i(B).
\]

Finally, we derive the spectral comparison rather than invoking it
implicitly. Let \(s_i(\cdot)\) denote decreasing singular values. From
\[
\Sigma_w\Sigma^{-1/2}=\Sigma^{1/2}B,
\qquad
\Sigma_w=(\Sigma_w\Sigma^{-1/2})\Sigma^{1/2},
\]
the product inequality
\(s_{p+q-1}(YZ)\le s_p(Y)s_q(Z)\) gives, for
\(i\le\lfloor M/3\rfloor\),
\[
\begin{aligned}
s_{2i-1}(\Sigma_w\Sigma^{-1/2})
&\le
s_i(\Sigma^{1/2})s_i(B)
=\sqrt{\mu_i(\Sigma)}\,\mu_i(B),\\
\mu_{3i-2}(\Sigma_w)
&=s_{3i-2}(\Sigma_w)\\
&\le
s_{2i-1}(\Sigma_w\Sigma^{-1/2})s_i(\Sigma^{1/2}).
\end{aligned}
\]
Combining these inequalities and using monotonicity of the eigenvalues of
\(\Sigma_w\),
\[
\mu_i(B)
\ge
\frac{\mu_{3i-2}(\Sigma_w)}{\mu_i(\Sigma)}
\ge
\frac{\mu_{3i}(\Sigma_w)}{\mu_i(\Sigma)}.
\]
Substitution into the preceding trace lower bound proves the lemma.
\end{proof}

\subsection{Example under the source condition}
\label{subsec:bias-gd-source}

\begin{lemma}[Bias bounds under the source condition]
\label{lem:gd-bias-source}
Assume Assumptions~\ref{ass:data-gaussian}, \ref{ass:data-power-law},
\ref{ass:source-diag}, \ref{ass:gaussian-sketch}, and
\ref{ass:stepsize}, and use the derived regularity conclusions of
Lemma~\ref{lem:app-derived-multipass-regularity}. Suppose
\[
a>b-1,
\qquad
A_L\gtrsim1,
\qquad
A_L\lesssim N^a.
\]
Then there exists an \((a,b)\)-dependent constant \(c>0\) such that, whenever
\[
\gamma\le \frac{c}{\log N},
\]
we have with probability at least
\[
1-\exp(-\Omega(M))
\]
over the randomness of \(S\),
\[
\mathbb{E}_{w^\ast}[\mathrm{Bias}_{\mathrm{GD}}(w^\ast)]
\lesssim
\min\!\bigl\{M,A_L^{1/a}\bigr\}^{1-b}.
\]
Moreover, there exists an \((a,b)\)-dependent constant \(c_0>0\) such
that, in the unsaturated regime
\[
A_L^{1/a}\le c_0M,
\]
the matching lower bound holds:
\[
\mathbb{E}_{w^\ast}[\mathrm{Bias}_{\mathrm{GD}}(w^\ast)]
\asymp
A_L^{(1-b)/a}.
\]
\end{lemma}

\begin{proof}
We first verify that the ingredients entering Lemmas~\ref{lem:gd-bias-upper} and \ref{lem:gd-bias-lower} are all of the claimed order.

By Lemma~\ref{lem:aux-G4-E5},
\[
\mu_i(\Sigma)\asymp i^{-a}
\qquad\text{for } i\in[M]
\]
with probability at least \(1-\exp(-\Omega(M))\). Therefore,
\[
\operatorname{tr}(\Sigma^2)\lesssim 1,
\qquad
\sum_{i=1}^M\frac{\mu_i(\Sigma)}{\mu_i(\Sigma)+1/A_L}
\lesssim A_L^{1/a}.
\]

Next, Lemma~\ref{lem:aux-G5-E6} implies
\[
\frac{\mu_{M/2}(A_k)}{\mu_M(A_k)}\lesssim 1
\]
for every admissible \(k\), while the power-law tail bounds give
\[
\operatorname{tr}(\Sigma_{>\widetilde{k}})\lesssim N^{1-a},
\qquad
\|\Sigma_{>\widetilde{k}}\|_2\lesssim N^{-a},
\qquad
\operatorname{tr}(\Sigma_{>\widetilde{k}}^2)\lesssim N^{1-2a},
\qquad
\operatorname{tr}(\Sigma_{>\widetilde{k}}^4)\lesssim N^{1-4a}.
\]
Hence the quantity \(\overline{B}\) in Lemma~\ref{lem:gd-bias-upper} satisfies
\[
\overline{B}\lesssim 1
\]
whenever \(A_L\lesssim N^a\).

For the upper bound, choose
\[
k:=\min\!\bigl\{M/3,A_L^{1/a}\bigr\}.
\]
Using Assumption~\ref{ass:source-diag},
\[
\mathbb{E}_{w^\ast}\|w_{0:k}^\ast\|_2^2
\asymp
\sum_{i=1}^k i^{a-b},
\qquad
\mathbb{E}_{w^\ast}\|w_{k:\infty}^\ast\|_{H_{k:\infty}}^2
\asymp
\sum_{i>k} i^{-b}.
\]
Plugging these relations into Lemma~\ref{lem:gd-bias-upper} gives
\[
\mathbb{E}_{w^\ast}[\mathrm{Bias}_{\mathrm{GD}}(w^\ast)]
\lesssim
\frac{1}{A_L}\sum_{i=1}^k i^{a-b}
+
\sum_{i>k} i^{-b}
\lesssim
k^{1-b}.
\]
Since \(k\asymp \min\{M,A_L^{1/a}\}\), this yields
\[
\mathbb{E}_{w^\ast}[\mathrm{Bias}_{\mathrm{GD}}(w^\ast)]
\lesssim
\min\!\bigl\{M,A_L^{1/a}\bigr\}^{1-b}.
\]

For the lower bound, set
\[
k_L:=A_L^{1/a}.
\]
Lemma~\ref{lem:aux-2025-E8} gives
\(\mu_i(\widehat\Sigma)\lesssim i^{-a}\) up to the empirical rank with
probability at least \(1-\exp(-\Omega(N))\) over \(X\). On this event,
\[
\#\bigl\{i:\mu_i(\widehat\Sigma)A_L>1/4\bigr\}
\lesssim k_L.
\]
On the complementary event the count is at most \(\min\{M,N\}\). Since
\(k_L\gtrsim1\), taking expectation over \(X\) therefore gives
\[
\tau\lesssim k_L.
\]
Lemma~\ref{lem:gd-bias-lower} now gives
\[
\mathbb{E}_{w^\ast}[\mathrm{Bias}_{\mathrm{GD}}(w^\ast)]
\gtrsim
\sum_{i=\lceil2\tau\rceil+1}^{\lfloor M/3\rfloor}
\frac{\mu_{3i}(\Sigma_w)}{\mu_i(\Sigma)}.
\]
Under Assumption~\ref{ass:source-diag}, the operator \(HH_wH\) has eigenvalues of order \(i^{-a-b}\), so Lemma~\ref{lem:aux-G4-E5} applied to \(HH_wH\) yields
\[
\mu_i(\Sigma_w)\asymp i^{-a-b}.
\]
Combining this with \(\mu_i(\Sigma)\asymp i^{-a}\), and choosing \(c_0\)
small enough that \(k_L\le c_0M\) leaves a nonempty interval from a constant
multiple of \(k_L\) to \(M/3\), gives
\[
\mathbb{E}_{w^\ast}[\mathrm{Bias}_{\mathrm{GD}}(w^\ast)]
\gtrsim
\sum_{Ck_L\le i\le M/3} i^{-b}
\gtrsim
k_L^{1-b}
=
A_L^{(1-b)/a}.
\]
Together with the global upper bound, this proves the two-sided estimate in
the unsaturated regime. When \(k_L\) is comparable to or larger than \(M\),
the lower-bound index set can be empty; no positive \(M^{1-b}\) lower bound
on the GD bias is claimed in that regime.
\end{proof}
\section{Variance Error under Normal GD}
\label{sec:var-gd}

Retain the notation of Section~\ref{sec:bias-gd}. In particular,
\[
V(\widehat{\Sigma})
:=
\frac{1}{N}\sum_{t=1}^L \gamma_t\prod_{i=t+1}^L (I-\gamma_i\widehat{\Sigma}),
\qquad
V_L:=I-\prod_{t=1}^L (I-\gamma_t\widehat{\Sigma}).
\]
We define the normalized GD variance term by
\[
\mathrm{Var}_{\mathrm{GD}}
:=
\mathbb{E}_X\Bigl[\operatorname{tr}\bigl(XS^\top V(\widehat{\Sigma})\Sigma V(\widehat{\Sigma})SX^\top\bigr)\Bigr].
\]
Then using
\[
V(\widehat{\Sigma})
=
\frac{1}{N}\Bigl(I-\prod_{t=1}^L(I-\gamma_t\widehat{\Sigma})\Bigr)\widehat{\Sigma}^{-1}
=
\frac{1}{N}V_L\widehat{\Sigma}^{-1},
\]
we may rewrite
\[
\mathrm{Var}_{\mathrm{GD}}
=
\frac{1}{N}\,\mathbb{E}_X\Bigl[\operatorname{tr}\bigl(\Sigma V_L\widehat{\Sigma}^{-1}V_L\bigr)\Bigr].
\]

\subsection{Upper and lower bounds}
\label{subsec:var-gd-bounds}

\begin{lemma}[Upper and lower bounds on the GD variance term]
\label{lem:gd-var-upper-lower}
Assume Assumptions~\ref{ass:data-gaussian}, \ref{ass:data-power-law},
\ref{ass:gaussian-sketch}, and \ref{ass:stepsize}, and work on the derived
regularity event in Lemma~\ref{lem:app-derived-multipass-regularity}. Suppose
\[
A_L\lesssim N^a.
\]
Then
\[
\mathrm{Var}_{\mathrm{GD}}
\lesssim
\frac{D_U}{N},
\qquad
\mathrm{Var}_{\mathrm{GD}}
\gtrsim
\frac{D_L}{N},
\]
where
\[
D_U
:=
\mathbb{E}_X\Biggl[
\#\bigl\{i\in[M]:\mu_i(\widehat{\Sigma})A_L>1/4\bigr\}
+
A_L\sum_{i:\,\mu_i(\widehat{\Sigma})A_L\le 1/4}
\mu_i(\widehat{\Sigma})
\Biggr],
\]
and
\[
D_L
:=
\mathbb{E}_X\Biggl[
A_L^2
\sum_{i:\,\mu_i(\widehat{\Sigma})A_L\le 1/4}
\mu_i(\Sigma)\mu_i(\widehat{\Sigma})
+
\frac{1}{5}
\sum_{i:\,\mu_i(\widehat{\Sigma})A_L>1/4}
\frac{\mu_i(\Sigma)}{\mu_i(\widehat{\Sigma})}
\Biggr].
\]
\end{lemma}

\begin{proof}
Using the identity
\[
\mathrm{Var}_{\mathrm{GD}}
=
\frac{1}{N}\,\mathbb{E}_X\bigl[\operatorname{tr}(\Sigma V_L\widehat{\Sigma}^{-1}V_L)\bigr],
\]
and we fix any \(\lambda>0\) for later separation. Then
\[
\operatorname{tr}(\Sigma V_L\widehat{\Sigma}^{-1}V_L)
\le
\|\Sigma^{1/2}(\widehat{\Sigma}+\lambda I)^{-1/2}\|_2^2
\cdot
\operatorname{tr}\bigl(V_L^2+\lambda\widehat{\Sigma}^{-1}V_L^2\bigr).
\]
Let
\[
v_L(x):=1-\prod_{t=1}^L(1-\gamma_t x).
\]
The product comparison in
Lemma~\ref{lem:multipass-wsd-calculus} gives
\[
v_L(x)\asymp \min\{1,A_L x\}
\qquad
\text{for }0\le x\le \|\widehat\Sigma\|_2.
\]
Consequently, choosing \(\lambda=A_L^{-1}\),
\[
\operatorname{tr}\bigl(V_L^2+\lambda\widehat{\Sigma}^{-1}V_L^2\bigr)
\lesssim
\#\bigl\{i:\mu_i(\widehat{\Sigma})A_L>1/4\bigr\}
+
A_L\sum_{i:\mu_i(\widehat{\Sigma})A_L\le 1/4}
\mu_i(\widehat{\Sigma})
\]
up to a change of universal constants in the threshold. Applying
Lemma~\ref{lem:aux-2025-E1} then yields
\[
\mathbb{E}_X\bigl[\operatorname{tr}(\Sigma V_L\widehat{\Sigma}^{-1}V_L)\bigr]
\lesssim D_U,
\]
which proves the upper bound.

For the lower bound, Von Neumann's trace inequality gives
\[
\operatorname{tr}(\Sigma V_L\widehat{\Sigma}^{-1}V_L)
\ge
\sum_{i=1}^M
\mu_i(\Sigma)\mu_i(\widehat{\Sigma})\,
\mu_{2(M-i)+1}\bigl(V_L^2\widehat{\Sigma}^{-2}\bigr).
\]
The scalar function
\[
f_L(x):=\frac{v_L(x)^2}{x^2}
\]
is decreasing on the stable interval because
\(v_L(x)/x=\sum_{t=1}^L\gamma_t\prod_{j=t+1}^L(1-\gamma_jx)\).
The WSD product comparison also gives
\[
\mu_{2(M-i)+1}\bigl(V_L^2\widehat{\Sigma}^{-2}\bigr)
\gtrsim
\begin{cases}
A_L^2, & \mu_i(\widehat{\Sigma})A_L\le 1/4,\\[3pt]
1/\mu_i(\widehat{\Sigma})^2, & \mu_i(\widehat{\Sigma})A_L>1/4.
\end{cases}
\]
Substituting this bound into the previous display gives
\[
\operatorname{tr}(\Sigma V_L\widehat{\Sigma}^{-1}V_L)
\gtrsim
A_L^2
\sum_{i:\,\mu_i(\widehat{\Sigma})A_L\le 1/4}
\mu_i(\Sigma)\mu_i(\widehat{\Sigma})
+
\frac{1}{5}
\sum_{i:\,\mu_i(\widehat{\Sigma})A_L>1/4}
\frac{\mu_i(\Sigma)}{\mu_i(\widehat{\Sigma})}.
\]
Taking expectation over \(X\) proves the lower bound.
\end{proof}

\subsection{Example under the source condition}
\label{subsec:var-gd-source}

\begin{lemma}[Variance bounds under the source condition]
\label{lem:gd-var-source}
Assume Assumptions~\ref{ass:data-gaussian}, \ref{ass:data-power-law},
\ref{ass:gaussian-sketch}, and \ref{ass:stepsize}, and use the derived
regularity conclusions of
Lemma~\ref{lem:app-derived-multipass-regularity}. Suppose
\[
A_L\lesssim N^a.
\]
Then there exists an \(a\)-dependent constant \(c>0\) such that, whenever
\[
\gamma\le \frac{c}{\log N},
\]
we have with probability at least
\[
1-\exp(-\Omega(M))
\]
over the randomness of \(S\),
\[
\mathrm{Var}_{\mathrm{GD}}
\asymp
\frac1N\sum_{j=1}^M\min\{1,(A_L\mu_j)^2\}
\asymp
\frac{\min\{M,A_L^{1/a}\}}{N}.
\]
\end{lemma}

\begin{proof}
For the upper bound, Lemma~\ref{lem:gd-var-upper-lower} gives
\[
\mathrm{Var}_{\mathrm{GD}}\lesssim \frac{D_U}{N}.
\]
By Lemma~\ref{lem:aux-G4-E5}, \(\mu_i(\Sigma)\asymp i^{-a}\), and the same
truncation argument used in the proof of Lemma~\ref{lem:gd-bias-source}
implies
\[
\mathbb{E}_X\Bigl[\#\{i:\mu_i(\widehat{\Sigma})A_L>1/4\}\Bigr]
\lesssim A_L^{1/a}.
\]
Moreover,
\[
A_L
\sum_{i:\mu_i(\widehat{\Sigma})A_L\le 1/4}
\mu_i(\widehat{\Sigma})
\lesssim
A_L
\sum_{i\gtrsim A_L^{1/a}} i^{-a}
\lesssim
A_L^{1/a}.
\]
Since also \(D_U\le M\), we obtain
\[
\mathrm{Var}_{\mathrm{GD}}
\lesssim
\frac{\min\{M,A_L^{1/a}\}}{N}.
\]

For the lower bound, first assume
\[
A_L^{1/a}\le M/c
\]
for a sufficiently large constant \(c>0\). By Lemmas~\ref{lem:aux-G4-E5} and \ref{lem:aux-2025-E8}, with high probability we have
\[
\mu_i(\Sigma)\asymp i^{-a},
\qquad
\mu_i(\widehat{\Sigma})\asymp i^{-a}
\qquad\text{for } i\le c^{-1}\min\{M,N\}.
\]
Therefore the first term in \(D_L\) satisfies
\[
D_L
\gtrsim
A_L^2
\sum_{i\gtrsim A_L^{1/a}} i^{-2a}
\gtrsim
A_L^{1/a}.
\]

If instead
\[
A_L^{1/a}\ge M/c,
\]
then the second term in \(D_L\) and the same empirical-spectrum estimate imply
\[
D_L
\gtrsim
\sum_{i\le M/c}\frac{\mu_i(\Sigma)}{\mu_i(\widehat{\Sigma})}
\gtrsim M.
\]
Combining the two regimes with Lemma~\ref{lem:gd-var-upper-lower} yields
\[
\mathrm{Var}_{\mathrm{GD}}
\gtrsim
\frac{\min\{M,A_L^{1/a}\}}{N}.
\]
This proves the lemma.
\end{proof}
\section{One-pass Batch SGD: Excess Error Decomposition}
\label{sec:onepass-batch-sgd-excess}

This section focuses on the one-pass batch SGD procedure
\eqref{eq:proc-onepass-batch-sgd} and uses the notation from
Section~\ref{subsec:app-onepass-batch-sgd-setup}. In particular, the centered error \(e_t:=u_t^{\mathrm{op}}-u^\ast\) satisfies
\begin{equation}
e_t
=
\bigl(I-\gamma_t \widehat{\Sigma}_t^{\mathrm{op}}\bigr)e_{t-1}
+
\gamma_t \widehat{\xi}_t.
\label{eq:batch-one-pass-centered-recursion}
\end{equation}

\subsection{Excess error decomposition}

Since \(u^\ast\) is the minimizer of the sketched population risk, the excess error is
\[
\mathcal E_{\mathrm{ex}}
:=
R_M(u_T^{\mathrm{op}})-R_M(u^\ast)
=
\|u_T^{\mathrm{op}}-u^\ast\|_{\Sigma}^2
=
\|e_T\|_{\Sigma}^2.
\]
Taking expectation, we will decompose
\[
\mathbb E\mathcal E_{\mathrm{ex}}
=
\mathbb E\|e_T\|_{\Sigma}^2
\]
into deterministic and stochastic contributions.

We use throughout the original mean-centered decomposition
\[
e_T=m_T+(e_T-m_T),
\qquad
m_T:=\mathbb E[e_T].
\]

\begin{definition}[Bias and centered variance]
\label{def:batch-bias}
Define the mean iterate error and centered fluctuation by
\[
m_t := \mathbb E[e_t],
\qquad
\delta_t := e_t-m_t.
\]
The one-pass bias and variance terms are
\[
\mathrm{Bias}_T
:=
\|m_T\|_{\Sigma}^2,
\qquad
\mathrm{Var}_{B_{1:T}}
:=
\mathbb E\|\delta_T\|_\Sigma^2.
\]
Since
\[
\mathbb E[\widehat{\Sigma}_t^{\mathrm{op}}]=\Sigma,
\qquad
\mathbb E[\widehat{\xi}_t]=0,
\]
the mean recursion satisfies
\[
m_t=(I-\gamma_t\Sigma)m_{t-1},
\qquad
m_0=-u^\ast,
\]
and hence
\[
m_T=-\prod_{t=1}^T (I-\gamma_t\Sigma)u^\ast.
\]
Therefore
\[
\boxed{
\mathrm{Bias}_T
=
\left\|
\prod_{t=1}^T (I-\gamma_t\Sigma)u^\ast
\right\|_{\Sigma}^2.
}
\]
\end{definition}

\begin{proposition}[Exact mean-centered decomposition of the excess error]
\label{prop:batch-centered-bv-decomp}
Assume Assumptions~\ref{ass:data-gaussian} and
\ref{ass:data-well-specified}. Under the one-pass batch SGD recursion
\eqref{eq:batch-one-pass-centered-recursion},
\[
\mathbb E\mathcal E_{\mathrm{ex}}
=
\mathbb E\|e_T\|_{\Sigma}^2
=
\mathrm{Bias}_T+\mathrm{Var}_{B_{1:T}}.
\]
\end{proposition}

\begin{proof}
Since \(e_T=m_T+\delta_T\), we have
\[
\|e_T\|_\Sigma^2
=
\|m_T\|_\Sigma^2
+
2\langle m_T,\delta_T\rangle_\Sigma
+
\|\delta_T\|_\Sigma^2.
\]
Taking expectation and using \(\mathbb E[\delta_T]=0\) gives
\[
\mathbb E\|e_T\|_\Sigma^2
=
\|m_T\|_\Sigma^2+\mathbb E\|\delta_T\|_\Sigma^2,
\]
which is exactly the claimed identity.
\end{proof}

\subsection{Fluctuation recursion and an exact split of the variance term}

Subtracting the mean recursion from
\eqref{eq:batch-one-pass-centered-recursion} gives the exact centered
fluctuation recursion
\begin{equation}
\delta_t
=
\bigl(I-\gamma_t\widehat\Sigma_t^{\mathrm{op}}\bigr)\delta_{t-1}
+
\gamma_t\bigl(\Sigma-\widehat\Sigma_t^{\mathrm{op}}\bigr)m_{t-1}
+
\gamma_t\widehat\xi_t.
\label{eq:batch-delta-recursion}
\end{equation}
Equivalently, in the shorthand notation of
Section~\ref{subsec:app-onepass-batch-sgd-setup},
\begin{equation}
\delta_t
=
\bigl(I-\gamma_t\bar Z_t\bigr)\delta_{t-1}
+
\gamma_t(\Sigma-\bar Z_t)m_{t-1}
+
\gamma_t\bar\xi_t.
\label{eq:batch-delta-recursion-shorthand}
\end{equation}

For the variance analysis it is convenient to split \(\delta_t\) into the
part caused by batch-covariance fluctuations and the part caused by the
additive label noise.

\begin{definition}[Centered covariance and noise components]
\label{def:batch-centered-components}
Define two auxiliary processes \((q_t)\) and \((v_t)\) by
\[
q_0=0,
\qquad
v_0=0,
\]
and, for \(t\ge 1\),
\[
q_t=(I-\gamma_t\bar Z_t)q_{t-1}+\gamma_t(\Sigma-\bar Z_t)m_{t-1},
\]
\[
v_t=(I-\gamma_t\bar Z_t)v_{t-1}+\gamma_t\bar\xi_t.
\]
Define the corresponding quadratic terms by
\[
\mathrm{Var}_{B_{1:T}}^{\mathrm{cov}}
:=
\mathbb E\|q_T\|_\Sigma^2,
\qquad
\mathrm{Var}_{B_{1:T}}^{\mathrm{noise}}
:=
\mathbb E\|v_T\|_\Sigma^2.
\]
\end{definition}

\begin{proposition}[Exact split of the centered variance]
\label{prop:batch-centered-var-split}
Assume Assumptions~\ref{ass:data-gaussian} and
\ref{ass:data-well-specified}. Then, for every \(t\),
\[
\delta_t=q_t+v_t,
\qquad
\mathbb E[q_t]=0.
\]
Moreover, if
\[
\mathcal G_t:=\sigma\bigl(z_{s,b}:1\le s\le t,\ 1\le b\le B_s\bigr),
\]
then
\[
\mathbb E[v_t\mid \mathcal G_t,w^\ast]=0,
\]
and therefore
\[
\mathrm{Var}_{B_{1:T}}
=
\mathrm{Var}_{B_{1:T}}^{\mathrm{cov}}
+\mathrm{Var}_{B_{1:T}}^{\mathrm{noise}}.
\]
\end{proposition}

\begin{proof}
The identity \(\delta_t=q_t+v_t\) follows by induction from
\eqref{eq:batch-delta-recursion-shorthand}. Indeed, it is true at time
\(t=0\), and if \(\delta_{t-1}=q_{t-1}+v_{t-1}\), then
\[
\delta_t
=
\bigl(I-\gamma_t\bar Z_t\bigr)(q_{t-1}+v_{t-1})
+
\gamma_t(\Sigma-\bar Z_t)m_{t-1}
+
\gamma_t\bar\xi_t
=q_t+v_t.
\]

Next, \(\mathbb E[q_0]=0\). If \(\mathbb E[q_{t-1}]=0\), then using
that \(q_{t-1}\) depends only on batches \(1,\dots,t-1\) while \(\bar Z_t\)
comes from the disjoint \(t\)-th batch, we may use independence of
\(q_{t-1}\) and \(\bar Z_t\), together with \(\mathbb E[\bar Z_t]=\Sigma\) and
the fact that \(m_{t-1}\) is deterministic, to obtain
\[
\mathbb E[q_t]
=
(I-\gamma_t\Sigma)\mathbb E[q_{t-1}]
+
\gamma_t\mathbb E[(\Sigma-\bar Z_t)m_{t-1}]
=0.
\]
Thus \(\mathbb E[q_t]=0\) for every \(t\).

To prove the conditional mean-zero property for \(v_t\), fix one sample in the
current batch and write \(z=Sx\). Conditioning on the sketch matrix \(S\), the
pair \((x,z)\) is jointly Gaussian with
\[
\mathbb E[x\mid z]=HS^\top(SHS^\top)^{-1}z=HS^\top\Sigma^{-1}z.
\]
Since \(u^\ast=\Sigma^{-1}SHw^\ast\), the Gaussian conditional-regression
identity yields
\[
\mathbb E[\langle x,w^\ast\rangle\mid z,w^\ast]
=
\langle \mathbb E[x\mid z],w^\ast\rangle
=
\langle HS^\top\Sigma^{-1}z,w^\ast\rangle
=
z^\top u^\ast.
\]
On the other hand, Assumption~\ref{ass:data-well-specified} implies
\[
\mathbb E[y-\langle x,w^\ast\rangle\mid x,w^\ast]=0.
\]
Therefore, by the tower property,
\[
\mathbb E[y-z^\top u^\ast\mid z,w^\ast]
=
\mathbb E[y-\langle x,w^\ast\rangle\mid z,w^\ast]
+
\mathbb E[\langle x,w^\ast\rangle-z^\top u^\ast\mid z,w^\ast]
=0.
\]
Multiplying by \(z\), which is measurable with respect to \(\sigma(z,w^\ast)\),
gives
\[
\mathbb E\bigl[z(y-z^\top u^\ast)\mid z,w^\ast\bigr]=0.
\]
Applying this to each summand in
\(\bar\xi_t=\frac1{B_t}\sum_{b=1}^{B_t}
z_{t,b}(y_{i_{t,b}}-z_{t,b}^\top u^\ast)\), and
conditioning on \(\mathcal G_t\), which fixes the current sketched covariates
\((z_{t,b})_{b=1}^{B_t}\), we obtain
\[
\mathbb E[\bar\xi_t\mid \mathcal G_t,w^\ast]=0.
\]
Because \(v_{t-1}\) depends only on the first \(t-1\) batches, it is
independent of the current block \((z_{t,b})_{b=1}^B\). Hence
\[
\mathbb E[v_t\mid \mathcal G_t,w^\ast]
=
(I-\gamma_t\bar Z_t)\mathbb E[v_{t-1}\mid \mathcal G_t,w^\ast]
+
\gamma_t\mathbb E[\bar\xi_t\mid \mathcal G_t,w^\ast].
\]
Since conditioning on the additional current block does not change
\(\mathbb E[v_{t-1}\mid \mathcal G_t,w^\ast]\), the induction hypothesis gives
\(\mathbb E[v_{t-1}\mid \mathcal G_t,w^\ast]
=
\mathbb E[v_{t-1}\mid \mathcal G_{t-1},w^\ast]=0\). Starting from \(v_0=0\), we
therefore obtain by induction that
\[
\mathbb E[v_t\mid \mathcal G_t,w^\ast]=0
\qquad\text{for all }t.
\]

Finally, \(q_T\) is \(\mathcal G_T\)-measurable, so
\[
\mathbb E\langle q_T,v_T\rangle_\Sigma
=
\mathbb E\Bigl[\mathbb E\bigl[\langle q_T,v_T\rangle_\Sigma\mid \mathcal G_T,w^\ast\bigr]\Bigr]
=0.
\]
Using \(\delta_T=q_T+v_T\), we conclude that
\[
\mathbb E\|\delta_T\|_\Sigma^2
=
\mathbb E\|q_T\|_\Sigma^2+\mathbb E\|v_T\|_\Sigma^2
=
\mathrm{Var}_{B_{1:T}}^{\mathrm{cov}}
+\mathrm{Var}_{B_{1:T}}^{\mathrm{noise}},
\]
as claimed.
\end{proof}

\section{Bias Error for One-pass Batch SGD}
\label{sec:batch-bias}

This section focuses on the one-pass batch SGD procedure
\eqref{eq:proc-onepass-batch-sgd} and uses the notation from
Section~\ref{subsec:app-onepass-batch-sgd-setup}. In particular,
\[
A_T:=\sum_{t=1}^T\gamma_t\asymp\gamma T,
\qquad
\Sigma:=SHS^\top,
\qquad
u^\ast:=\Sigma^{-1}SHw^\ast.
\]
By Definition~\ref{def:batch-bias}, this section studies the exact one-pass bias term
\[
\mathrm{Bias}_T=\|m_T\|_\Sigma^2.
\]

Define
\[
C_T:=\prod_{t=1}^T (I-\gamma_t\Sigma).
\]
Since the mean recursion satisfies \(m_t=(I-\gamma_t\Sigma)m_{t-1}\) and
\(m_0=-u^\ast\), we have
\[
m_T=-C_Tu^\ast,
\qquad
\mathrm{Bias}_T=\|C_Tu^\ast\|_\Sigma^2.
\]

\subsection{Upper and lower bounds for the bias term}
\label{subsec:batch-bias-bounds}

\begin{lemma}[Upper bound on the one-pass batch bias]
\label{lem:batch-bias-upper}
Assume Assumptions~\ref{ass:source-diag} and \ref{ass:gaussian-sketch}, and
the WSD conditions of Theorem~\ref{thm:onepass-batch-sgd-scaling}. Fix any
integer \(k\le M/3\) such that \(\operatorname{rank}(H)\ge k+M\), and define
\[
A_k:=S_{k:\infty}H_{k:\infty}S_{k:\infty}^\top.
\]
Then, with probability at least
\[
1-\exp(-\Omega(M))
\]
over the randomness of \(S\),
\[
\mathrm{Bias}_T
\lesssim
\frac{\|w^\ast_{0:k}\|_2^2}{A_T}
\left(\frac{\mu_{M/2}(A_k)}{\mu_M(A_k)}\right)^2
+
\|w^\ast_{k:\infty}\|_{H_{k:\infty}}^2.
\]
In particular,
\[
\mathrm{Bias}_T
\lesssim
\frac{\|u^\ast\|_2^2}{A_T}.
\]
\end{lemma}

\begin{proof}
By Assumption~\ref{ass:gaussian-sketch}, we may work in the diagonal coordinates of \(H\), as encoded in Assumption~\ref{ass:source-diag}. Define
\[
M_T:=C_T\Sigma C_T.
\]
By Lemma~\ref{lem:onepass-wsd-calculus}, for a constant \(c>0\),
\[
M_T
\preceq
e^{-cA_T\Sigma}
\Sigma
e^{-cA_T\Sigma}
=:M.
\]
Therefore,
\begin{align}
\mathrm{Bias}_T
&=
\langle M_T,u^\ast u^{\ast\top}\rangle
\nonumber\\
&\le
\langle M,u^\ast u^{\ast\top}\rangle
\nonumber\\
&=
 w^{\ast\top}HS^\top\Sigma^{-1}M\Sigma^{-1}SHw^\ast.
\label{eq:batch-bias-upper-start}
\end{align}

Now decompose
\[
SH=(S_{0:k}H_{0:k},\,S_{k:\infty}H_{k:\infty}).
\]
Then using the idea of spectral truncation, we have
\[
2T_1+2T_2,
\]
where
\[
T_1:=w_{0:k}^{\ast\top}H_{0:k}S_{0:k}^\top\Sigma^{-1}M\Sigma^{-1}S_{0:k}H_{0:k}w_{0:k}^\ast,
\]
and
\[
T_2:=w_{k:\infty}^{\ast\top}H_{k:\infty}S_{k:\infty}^\top\Sigma^{-1}M\Sigma^{-1}S_{k:\infty}H_{k:\infty}w_{k:\infty}^\ast.
\]

For the head term, we have
\[
T_1
\le
\|M\|_2\cdot \|\Sigma^{-1}S_{0:k}H_{0:k}\|_2^2\cdot \|w_{0:k}^\ast\|_2^2.
\]
By spectral calculus,
\[
\|M\|_2
=
\max_{x\in[0,\|\Sigma\|_2]} xe^{-2cA_Tx}
\lesssim
\frac{1}{A_T}.
\]
Moreover, Lemma~\ref{lem:aux-head-tail-resolvent} yields
\[
\|\Sigma^{-1}S_{0:k}H_{0:k}\|_2
\lesssim
\frac{\mu_{M/2}(A_k)}{\mu_M(A_k)}
\]
with probability at least \(1-\exp(-\Omega(M))\). Hence
\[
T_1
\lesssim
\frac{\|w_{0:k}^\ast\|_2^2}{A_T}
\left(\frac{\mu_{M/2}(A_k)}{\mu_M(A_k)}\right)^2.
\]

For the tail term, we also have
\[
0\preceq M=e^{-cA_T\Sigma}\Sigma
e^{-cA_T\Sigma}\preceq \Sigma.
\]
Therefore,
\begin{align*}
T_2
&\le
 w_{k:\infty}^{\ast\top}H_{k:\infty}S_{k:\infty}^\top\Sigma^{-1}\Sigma\Sigma^{-1}S_{k:\infty}H_{k:\infty}w_{k:\infty}^\ast
\\
&=
 w_{k:\infty}^{\ast\top}H_{k:\infty}^{1/2}
\Bigl(H_{k:\infty}^{1/2}S_{k:\infty}^\top\Sigma^{-1}S_{k:\infty}H_{k:\infty}^{1/2}\Bigr)
H_{k:\infty}^{1/2}w_{k:\infty}^\ast
\\
&\le
\|w_{k:\infty}^\ast\|_{H_{k:\infty}}^2,
\end{align*}
where the last step uses
\[
0\preceq H_{k:\infty}^{1/2}S_{k:\infty}^\top\Sigma^{-1}S_{k:\infty}H_{k:\infty}^{1/2}\preceq I.
\]
Combining the bounds on \(T_1\) and \(T_2\) proves the first claim.

For the simpler bound, we may ignore the head--tail split and use
\[
\mathrm{Bias}_T
\le
\|M\|_2\,\|u^\ast\|_2^2
\lesssim
\frac{\|u^\ast\|_2^2}{A_T}.
\]
\end{proof}

\begin{lemma}[Lower bound on the one-pass batch bias]
\label{lem:batch-bias-lower}
Assume the WSD conditions of
Theorem~\ref{thm:onepass-batch-sgd-scaling}. Let
\[
H_w:=\mathbb{E}_{w^\ast}[w^\ast w^{\ast\top}],
\qquad
\Sigma_w:=SHH_wHS^\top.
\]
Then, conditioned on the sketch matrix \(S\),
\[
\mathbb{E}_{w^\ast}[\mathrm{Bias}_T]
\gtrsim
\sum_{i:\,\mu_i(\Sigma)<1/A_T}
\frac{\mu_i(\Sigma_w)}{\mu_i(\Sigma)}.
\]
\end{lemma}

\begin{proof}
Let us define
\[
M_T'
:=
\Sigma e^{-CA_T\Sigma},
\]
where \(C>0\) is a sufficiently large universal constant. The lower product
comparison in Lemma~\ref{lem:onepass-wsd-calculus} gives
\[
\prod_{t=1}^T (1-\gamma_t x)^2
\ge
e^{-CA_Tx}
\qquad\text{for every } x\in[0,\|\Sigma\|_2].
\]
By functional calculus, this gives
\[
M_T:=C_T\Sigma C_T\succeq M_T'.
\]
Therefore,
\[
\mathrm{Bias}_T
=
\langle M_T,u^\ast u^{\ast\top}\rangle
\ge
\langle M_T',u^\ast u^{\ast\top}\rangle.
\]
Since
\[
\mathbb{E}_{w^\ast}[u^\ast u^{\ast\top}]
=
\Sigma^{-1}\Sigma_w\Sigma^{-1},
\]
we have
\begin{align}
\mathbb{E}_{w^\ast}[\mathrm{Bias}_T]
&\ge
\operatorname{tr}\!\bigl(M_T'\Sigma^{-1}\Sigma_w\Sigma^{-1}\bigr)
\nonumber\\
&=
\operatorname{tr}\!\bigl(\Sigma^{-1}M_T'\Sigma^{-1}\Sigma_w\bigr).
\label{eq:batch-bias-lower-start}
\end{align}
Applying von Neumann's trace inequality to
\eqref{eq:batch-bias-lower-start} gives
\[
\mathbb{E}_{w^\ast}[\mathrm{Bias}_T]
\gtrsim
\sum_{i=1}^M
\mu_{M-i+1}(\Sigma^{-1}M_T'\Sigma^{-1})\,\mu_i(\Sigma_w).
\]
Since \(M_T'\) is positive definite, the identity
\(\mu_{M-i+1}(A)=\mu_i(A^{-1})^{-1}\) yields
\[
\mu_{M-i+1}(\Sigma^{-1}M_T'\Sigma^{-1})
=
\frac{1}{\mu_i(\Sigma^2M_T'^{-1})}
=
\frac{1}{\mu_i\!\bigl(\Sigma e^{CA_T\Sigma}\bigr)}.
\]
Therefore,
\[
\mathbb{E}_{w^\ast}[\mathrm{Bias}_T]
\gtrsim
\sum_{i=1}^M
\frac{\mu_i(\Sigma_w)}{\mu_i\!\bigl(\Sigma e^{CA_T\Sigma}\bigr)}.
\]
If \(\mu_i(\Sigma)<1/A_T\), then
\[
e^{CA_T\mu_i(\Sigma)}\lesssim 1,
\]
so the corresponding denominator is comparable to \(\mu_i(\Sigma)\). Restricting the sum to this index set gives
\[
\mathbb{E}_{w^\ast}[\mathrm{Bias}_T]
\gtrsim
\sum_{i:\,\mu_i(\Sigma)<1/A_T}
\frac{\mu_i(\Sigma_w)}{\mu_i(\Sigma)},
\]
which is exactly the claimed bound.
\end{proof}

\subsection{Bounds under the source condition}
\label{subsec:batch-bias-source}

\begin{lemma}[Bounds on the one-pass batch bias under the source condition]
\label{lem:batch-bias-source}
Assume Assumptions~\ref{ass:data-power-law}, \ref{ass:source-diag}, and
\ref{ass:gaussian-sketch}, together with the WSD conditions in
Theorem~\ref{thm:onepass-batch-sgd-scaling}. Suppose moreover that
\[
a>b-1.
\]
Then there exists an \((a,b)\)-dependent constant \(c_1>0\) such that, with
probability at least
\[
1-\exp(-\Omega(M))
\]
over the randomness of \(S\),
\[
\mathbb{E}_{w^\ast}[\mathrm{Bias}_T]
\lesssim
\min\!\bigl\{M,A_T^{1/a}\bigr\}^{1-b}.
\]
Moreover, if in addition
\[
A_T^{1/a}\le \frac{M}{c_1},
\]
then, on the same event,
\[
\mathbb{E}_{w^\ast}[\mathrm{Bias}_T]
\gtrsim
A_T^{(1-b)/a}.
\]
\end{lemma}

\begin{proof}
In this proof, we verify the upper and lower bounds separately.

For the upper bound, choose
\[
k:=\min\!\bigl\{M/3,A_T^{1/a}\bigr\}.
\]
By Lemma~\ref{lem:aux-G5-E6},
\[
\frac{\mu_{M/2}(A_k)}{\mu_M(A_k)}\lesssim 1
\]
with probability at least \(1-\exp(-\Omega(M))\). Under Assumption~\ref{ass:source-diag},
\[
\mathbb{E}_{w^\ast}\|w_{0:k}^\ast\|_2^2
\asymp
\sum_{i=1}^k i^{a-b},
\qquad
\mathbb{E}_{w^\ast}\|w_{k:\infty}^\ast\|_{H_{k:\infty}}^2
\asymp
\sum_{i>k} i^{-b}.
\]
Plugging these relations into Lemma~\ref{lem:batch-bias-upper} gives
\[
\mathbb{E}_{w^\ast}[\mathrm{Bias}_T]
\lesssim
\frac{1}{A_T}\sum_{i=1}^k i^{a-b}
+
\sum_{i>k} i^{-b}
\lesssim
k^{1-b}.
\]
Since \(k\asymp \min\{M,A_T^{1/a}\}\), this yields
\[
\mathbb{E}_{w^\ast}[\mathrm{Bias}_T]
\lesssim
\min\!\bigl\{M,A_T^{1/a}\bigr\}^{1-b}.
\]

For the lower bound, set
\[
\theta:=A_T^{1/a}.
\]
Lemma~\ref{lem:aux-G4-E5} gives
\[
\mu_i(\Sigma)\asymp i^{-a}
\qquad\text{for } i\in[M]
\]
with probability at least \(1-\exp(-\Omega(M))\). Therefore, the condition
\[
\mu_i(\Sigma)<\frac{1}{A_T}
\]
forces
\[
i\gtrsim \theta.
\]
Moreover, under Assumption~\ref{ass:source-diag}, the operator \(HH_wH\) has eigenvalues of order \(i^{-a-b}\), so Lemma~\ref{lem:aux-G4-E5} applied to \(HH_wH\) yields
\[
\mu_i(\Sigma_w)\asymp i^{-a-b}.
\]
Combining these estimates with Lemma~\ref{lem:batch-bias-lower}, we obtain
\[
\mathbb{E}_{w^\ast}[\mathrm{Bias}_T]
\gtrsim
\sum_{i\gtrsim \theta} \frac{\mu_i(\Sigma_w)}{\mu_i(\Sigma)}
\asymp
\sum_{i\gtrsim \theta} i^{-b}.
\]
If \(\theta\le M/c_1\) for a sufficiently large \((a,b)\)-dependent constant
\(c_1\), then the last sum contains the range \([C\theta,2C\theta]\) for some
constant \(C>0\), and hence
\[
\mathbb{E}_{w^\ast}[\mathrm{Bias}_T]
\gtrsim
\theta^{1-b}
\gtrsim
 A_T^{(1-b)/a}.
\]
Combining the upper and lower bounds proves the claim.
\end{proof}
\section{Variance Error for One-pass Batch SGD}
\label{sec:batch-variance}

This section proves the one-pass variance bound for an arbitrary deterministic
batch-size schedule \(B_1,\dots,B_T\) satisfying
\(\sum_{t=1}^T B_t=N\). We use the notation from
Section~\ref{subsec:app-onepass-batch-sgd-setup}. In particular,
\[
\begin{aligned}
\tau_{\mathrm d}^{\mathrm{op}}&:=\frac{T-T_{\mathrm s}}{\log T},
&A_{\mathrm d}^{\mathrm{op}}&:=\gamma\tau_{\mathrm d}^{\mathrm{op}},\\
\bar Z_t&:=\frac1{B_t}\sum_{r=1}^{B_t}z_{t,r}z_{t,r}^{\top},
&\bar\xi_t&:=\frac1{B_t}\sum_{r=1}^{B_t}
z_{t,r}(y_{i_{t,r}}-z_{t,r}^{\top}u^\ast).
\end{aligned}
\]
Throughout this section, \((\gamma_t)_{t=1}^T\) is the one-pass WSD schedule
from \eqref{eq:app-onepass-wsd-schedule}.

Let \(\alpha\) be the fourth-moment constant from
Lemma~\ref{lem:aux-sketch-moment}; for Gaussian design one may take
\(\alpha=3\). Set
\[
R_\star^2:=\alpha\operatorname{tr}(\Sigma),
\qquad
\widetilde\sigma^2(w^\ast)
:=
2\bigl(\sigma^2+\alpha\|w^\ast\|_H^2\bigr),
\]
and
\[
\overline\sigma^2(w^\ast)
:=
\widetilde\sigma^2(w^\ast)+\alpha\|w^\ast\|_H^2.
\]
The constant in the one-pass stepsize assumption is chosen so that
\begin{equation}
\gamma R_\star^2\le\frac14.
\label{eq:dynamic-batch-uniform-stability}
\end{equation}
This is the same primitive stepsize condition already used in the fixed-batch
analysis. It is uniform over every schedule because \(B_t\ge1\); no additional
assumption involving \(B_{\min}\) is required.

\subsection{Time-dependent batch moments and the raw kernel}
\label{subsec:batch-variance-raw}

For any PSD matrix \(A\), independence within the \(t\)-th batch gives
\begin{align}
\mathbb E[\bar Z_tA\bar Z_t]
&=
\frac1{B_t}\mathbb E[zz^\top Azz^\top]
+\left(1-\frac1{B_t}\right)\Sigma A\Sigma,
\label{eq:dynamic-batch-fourth-moment}\\
\mathbb E[(\bar Z_t-\Sigma)A(\bar Z_t-\Sigma)]
&=
\frac1{B_t}
\left(
\mathbb E[zz^\top Azz^\top]-\Sigma A\Sigma
\right).
\label{eq:dynamic-batch-centered-fourth-moment}
\end{align}
Similarly, the conditional centering proved in
Proposition~\ref{prop:batch-centered-var-split} yields
\begin{equation}
\mathbb E[\bar\xi_t\bar\xi_t^\top]
=
\frac1{B_t}
\mathbb E\!\left[
z(y-z^\top u^\ast)^2z^\top
\right].
\label{eq:dynamic-batch-noise-moment}
\end{equation}

\begin{definition}[Dynamic raw kernel]
\label{def:raw-var-batch}
Define
\begin{align}
\mathcal K_{B_{1:T}}
&:=
\sum_{t=1}^T\frac{\gamma_t^2}{B_t}
\left\langle
\Sigma,\,
\left[\prod_{s=t+1}^T(I-\gamma_s\Sigma)^2\right]\Sigma
\right\rangle
\notag\\
&=
\sum_{t=1}^T\frac{\kappa_t^{\mathrm{op}}}{B_t}.
\label{eq:raw-var-batch}
\end{align}
By the definition of \(B_{\mathrm{eff}}\),
\begin{equation}
\mathcal K_{B_{1:T}}
=
\frac1{B_{\mathrm{eff}}}\sum_{t=1}^T\kappa_t^{\mathrm{op}}.
\label{eq:dynamic-kernel-effective-batch}
\end{equation}
\end{definition}

The main issue is to retain the factor \(1/B_t\) after the random covariance
operators are iterated. A direct replacement by \(1/B_{\min}\) would be
stable but would lose the claimed effective-batch rate. The next two
subsections use a uniform covariance bound only to control the multiplicative
perturbation; the perturbation at time \(t\) itself still carries the exact
factor \(1/B_t\).

\subsection{The additive-noise component}
\label{subsec:batch-modified-C1}

Recall
\[
v_t=(I-\gamma_t\bar Z_t)v_{t-1}+\gamma_t\bar\xi_t,
\qquad
v_0=0,
\]
and let \(C_t:=\mathbb E[v_tv_t^\top]\). For a deterministic PSD matrix \(A\),
write
\[
\mathcal A_t(A)
:=
\mathbb E[(I-\gamma_t\bar Z_t)A(I-\gamma_t\bar Z_t)]
\]
and
\[
\mathcal D_t(A)
:=
(I-\gamma_t\Sigma)A(I-\gamma_t\Sigma).
\]

\begin{lemma}[Time-dependent covariance recursion for additive noise]
\label{lem:batch-wu-C1}
Under Assumptions~\ref{ass:data-gaussian} and
\ref{ass:data-well-specified}, define
\(\Sigma_{\xi,t}:=\mathbb E[\bar\xi_t\bar\xi_t^\top]\). Then
\[
\frac{\sigma^2}{B_t}\Sigma
\preceq
\Sigma_{\xi,t}
\preceq
\frac{\widetilde\sigma^2(w^\ast)}{B_t}\Sigma
\]
and
\begin{equation}
C_t=\mathcal A_t(C_{t-1})+\gamma_t^2\Sigma_{\xi,t}.
\label{eq:batch-C-recursion}
\end{equation}
Moreover, under \eqref{eq:dynamic-batch-uniform-stability},
\begin{equation}
C_t
\preceq
\frac{\gamma\widetilde\sigma^2(w^\ast)}
{1-\gamma R_\star^2}\,I
\qquad
\text{for every }t.
\label{eq:dynamic-C-uniform}
\end{equation}
\end{lemma}

\begin{proof}
For the lower covariance bound, write
\[
\varepsilon:=y-x^\top w^\ast,
\qquad
r:=w^\ast-S^\top u^\ast.
\]
Then \(y-z^\top u^\ast=\varepsilon+x^\top r\). By the well-specified,
homoscedastic assumption,
\[
\mathbb E[\varepsilon^2zz^\top]=\sigma^2\Sigma,
\qquad
\mathbb E[\varepsilon(x^\top r)zz^\top]=0.
\]
The remaining quadratic term is PSD, so
\[
\mathbb E[z(y-z^\top u^\ast)^2z^\top]\succeq\sigma^2\Sigma.
\]
Together with \eqref{eq:dynamic-batch-noise-moment}, this gives the lower
bound on \(\Sigma_{\xi,t}\). The upper bound follows from
Lemma~\ref{lem:aux-sketch-moment}.
The mixed terms in the second moment of \(v_t\) vanish after conditioning on
the current sketched covariates and \(w^\ast\), which proves
\eqref{eq:batch-C-recursion}.

For the uniform bound, the fourth-moment inequality and
\eqref{eq:dynamic-batch-fourth-moment} imply
\[
\mathbb E[\bar Z_t^2]\preceq R_\star^2\Sigma
\qquad\text{for all }t.
\]
Set
\[
K_v:=\frac{\gamma\widetilde\sigma^2(w^\ast)}
{1-\gamma R_\star^2}.
\]
If \(C_{t-1}\preceq K_vI\), monotonicity of the fourth-moment operator,
\(B_t\ge1\), and \(\gamma_t\le\gamma\) give
\[
C_t
\preceq
K_v\bigl(I-2\gamma_t\Sigma+\gamma_t^2R_\star^2\Sigma\bigr)
+\gamma_t^2\widetilde\sigma^2(w^\ast)\Sigma
\preceq K_vI.
\]
Indeed, using
\(\gamma\widetilde\sigma^2(w^\ast)
=K_v(1-\gamma R_\star^2)\), the coefficient of \(\Sigma\) in
\(C_t-K_vI\) is at most
\[
-K_v\gamma_t\bigl(1+(\gamma-\gamma_t)R_\star^2\bigr)\le0.
\]
Since \(C_0=0\),
induction proves \eqref{eq:dynamic-C-uniform}.
\end{proof}

\begin{theorem}[Two-sided bound for the additive-noise component]
\label{thm:batch-wu-C1}
Under the assumptions of Lemma~\ref{lem:batch-wu-C1} and
\eqref{eq:dynamic-batch-uniform-stability},
\[
\sigma^2\mathcal K_{B_{1:T}}
\le
\mathrm{Var}_{B_{1:T}}^{\mathrm{noise}}
\le
\frac{\widetilde\sigma^2(w^\ast)}
{1-\gamma R_\star^2}\,
\mathcal K_{B_{1:T}}.
\]
\end{theorem}

\begin{proof}
For every PSD \(A\),
\[
\mathcal A_t(A)-\mathcal D_t(A)
=
\gamma_t^2
\mathbb E[(\bar Z_t-\Sigma)A(\bar Z_t-\Sigma)]
\succeq0.
\]
Combining this with the lower bound on \(\Sigma_{\xi,t}\), then unrolling
\eqref{eq:batch-C-recursion}, gives
\[
C_T
\succeq
\sigma^2\sum_{t=1}^T\frac{\gamma_t^2}{B_t}
\left[\prod_{s=t+1}^T(I-\gamma_s\Sigma)^2\right]\Sigma.
\]
Taking the inner product with \(\Sigma\) proves the lower bound.

For the upper bound, \eqref{eq:dynamic-batch-fourth-moment},
Lemma~\ref{lem:aux-sketch-moment}, and
\eqref{eq:dynamic-C-uniform} imply
\[
\begin{aligned}
\mathcal A_t(C_{t-1})
&\preceq
\mathcal D_t(C_{t-1})
+
\frac{\alpha\gamma_t^2}{B_t}
\operatorname{tr}(\Sigma C_{t-1})\Sigma\\
&\preceq
\mathcal D_t(C_{t-1})
+
\frac{\alpha K_v\operatorname{tr}(\Sigma)\gamma_t^2}{B_t}\Sigma.
\end{aligned}
\]
After adding the driving covariance,
\[
C_t
\preceq
\mathcal D_t(C_{t-1})
+
\frac{\gamma_t^2}{B_t}
\left(\widetilde\sigma^2(w^\ast)
+\alpha K_v\operatorname{tr}(\Sigma)\right)\Sigma.
\]
Since
\[
\widetilde\sigma^2(w^\ast)
+\alpha K_v\operatorname{tr}(\Sigma)
=
\frac{\widetilde\sigma^2(w^\ast)}
{1-\gamma R_\star^2},
\]
unrolling the deterministic maps \(\mathcal D_t\) and taking the inner
product with \(\Sigma\) proves the upper bound. Notice that the intermediate
uniform estimate does not replace \(B_t\): every driving term retains its
own factor \(1/B_t\).
\end{proof}

\subsection{The covariance-fluctuation component}
\label{subsec:batch-centered-covariance-component}

Recall
\[
q_t=(I-\gamma_t\bar Z_t)q_{t-1}
+\gamma_t(\Sigma-\bar Z_t)m_{t-1},
\qquad
q_0=0,
\]
and put
\[
Q_t:=\mathbb E[q_tq_t^\top],
\qquad
\Lambda_t
:=
\mathbb E[
(\Sigma-\bar Z_t)m_{t-1}m_{t-1}^\top(\Sigma-\bar Z_t)
].
\]

\begin{lemma}[Time-dependent covariance recursion for covariance fluctuation]
\label{lem:batch-centered-C1}
Under Assumptions~\ref{ass:data-gaussian} and
\ref{ass:data-well-specified},
\[
\Lambda_t
\preceq
\frac{\alpha\|m_{t-1}\|_\Sigma^2}{B_t}\Sigma
\preceq
\frac{\alpha\|u^\ast\|_\Sigma^2}{B_t}\Sigma,
\]
and
\begin{equation}
Q_t=\mathcal A_t(Q_{t-1})+\gamma_t^2\Lambda_t.
\label{eq:batch-Q-recursion}
\end{equation}
Under \eqref{eq:dynamic-batch-uniform-stability},
\[
Q_t
\preceq
\frac{\gamma\alpha\|u^\ast\|_\Sigma^2}
{1-\gamma R_\star^2}\,I
\qquad
\text{for every }t.
\]
\end{lemma}

\begin{proof}
Equation~\eqref{eq:dynamic-batch-centered-fourth-moment} and
Lemma~\ref{lem:aux-sketch-moment} give
\[
\Lambda_t
\preceq
\frac{\alpha}{B_t}
\operatorname{tr}(\Sigma m_{t-1}m_{t-1}^\top)\Sigma.
\]
The deterministic mean recursion is contractive in the \(\Sigma\)-norm, so
\(\|m_{t-1}\|_\Sigma\le\|u^\ast\|_\Sigma\).

The cross terms in the second moment of \(q_t\) vanish: \(q_{t-1}\) is
independent of the current batch and \(\mathbb E[q_{t-1}]=0\). This proves
\eqref{eq:batch-Q-recursion}. The uniform estimate follows from the same
induction as in Lemma~\ref{lem:batch-wu-C1}, with
\(\widetilde\sigma^2(w^\ast)\) replaced by
\(\alpha\|u^\ast\|_\Sigma^2\).
\end{proof}

\begin{theorem}[Bound on the covariance-fluctuation component]
\label{thm:batch-centered-C1}
Under the assumptions of Lemma~\ref{lem:batch-centered-C1} and
\eqref{eq:dynamic-batch-uniform-stability},
\[
0\le
\mathrm{Var}_{B_{1:T}}^{\mathrm{cov}}
\le
\frac{\alpha\|u^\ast\|_\Sigma^2}
{1-\gamma R_\star^2}\,
\mathcal K_{B_{1:T}}.
\]
\end{theorem}

\begin{proof}
Let
\[
K_q:=
\frac{\gamma\alpha\|u^\ast\|_\Sigma^2}
{1-\gamma R_\star^2}.
\]
As in the additive-noise proof,
\[
\mathcal A_t(Q_{t-1})
\preceq
\mathcal D_t(Q_{t-1})
+
\frac{\alpha K_q\operatorname{tr}(\Sigma)\gamma_t^2}{B_t}\Sigma.
\]
Combining this with the bound on \(\Lambda_t\) gives
\[
Q_t
\preceq
\mathcal D_t(Q_{t-1})
+
\frac{\gamma_t^2}{B_t}
\frac{\alpha\|u^\ast\|_\Sigma^2}
{1-\gamma R_\star^2}\Sigma.
\]
Unrolling and taking the inner product with \(\Sigma\) proves the result.
\end{proof}

\begin{proposition}[Raw dynamic-batch variance bound]
\label{prop:raw-var-controls-variance}
Under Assumptions~\ref{ass:data-gaussian} and
\ref{ass:data-well-specified}, and under
\eqref{eq:dynamic-batch-uniform-stability},
\[
\sigma^2\mathcal K_{B_{1:T}}
\le
\mathrm{Var}_{B_{1:T}}^{\mathrm{noise}}
\le
\frac{\widetilde\sigma^2(w^\ast)}
{1-\gamma R_\star^2}\mathcal K_{B_{1:T}},
\]
\[
0\le
\mathrm{Var}_{B_{1:T}}^{\mathrm{cov}}
\le
\frac{\alpha\|u^\ast\|_\Sigma^2}
{1-\gamma R_\star^2}\mathcal K_{B_{1:T}},
\]
and consequently
\[
\sigma^2\mathcal K_{B_{1:T}}
\le
\mathrm{Var}_{B_{1:T}}
\le
\frac{\overline\sigma^2(w^\ast)}
{1-\gamma R_\star^2}\mathcal K_{B_{1:T}}.
\]
\end{proposition}

\begin{proof}
Combine Proposition~\ref{prop:batch-centered-var-split},
Theorem~\ref{thm:batch-wu-C1}, and
Theorem~\ref{thm:batch-centered-C1}. To compare the two norms, set
\[
P:=H^{1/2}S^\top\Sigma^{-1}SH^{1/2}.
\]
Since \(\Sigma=SHS^\top\), one has \(P^2=P\), so \(P\) is an orthogonal
projector. Therefore
\[
\|u^\ast\|_\Sigma^2
=
w^{\ast\top}H^{1/2}PH^{1/2}w^\ast
\le
\|w^\ast\|_H^2,
\]
which yields the stated last constant.
\end{proof}

\subsection{Bounds under the source condition}
\label{subsec:batch-variance-source}

\begin{lemma}[Source-condition bounds for the exact one-pass variance terms]
\label{lem:batch-var-source}
Assume Assumptions~\ref{ass:data-gaussian},
\ref{ass:data-well-specified}, \ref{ass:data-power-law},
\ref{ass:source-diag}, and \ref{ass:gaussian-sketch}, together with the WSD
conditions in Theorem~\ref{thm:onepass-batch-sgd-scaling}. Suppose
\[
\sigma^2\asymp1,
\qquad
A_{\mathrm d}^{\mathrm{op}}=\gamma\tau_{\mathrm d}^{\mathrm{op}}\gtrsim1.
\]
Then, with probability at least \(1-\exp(-\Omega(M))\) over \(S\),
\[
\mathbb E_{w^\ast}[
\mathrm{Var}_{B_{1:T}}^{\mathrm{noise}}]
\asymp
\frac{\min\{M,(A_{\mathrm d}^{\mathrm{op}})^{1/a}\}}
{\tau_{\mathrm d}^{\mathrm{op}}B_{\mathrm{eff}}},
\]
\[
\mathbb E_{w^\ast}[
\mathrm{Var}_{B_{1:T}}^{\mathrm{cov}}]
\lesssim
\frac{\min\{M,(A_{\mathrm d}^{\mathrm{op}})^{1/a}\}}
{\tau_{\mathrm d}^{\mathrm{op}}B_{\mathrm{eff}}},
\]
and therefore
\[
\mathbb E_{w^\ast}[\mathrm{Var}_{B_{1:T}}]
\asymp
\frac{\min\{M,(A_{\mathrm d}^{\mathrm{op}})^{1/a}\}}
{\tau_{\mathrm d}^{\mathrm{op}}B_{\mathrm{eff}}}.
\]
\end{lemma}

\begin{proof}
On the power-law sketch event,
\(\operatorname{tr}(\Sigma)\lesssim\sum_{j\ge1}j^{-a}\lesssim1\). Thus the
constant in \(\gamma\le c\) can be chosen so that
\(\gamma R_\star^2\le1/4\), and Proposition~\ref{prop:raw-var-controls-variance}
applies.
Under the source condition and \(b>1\),
\[
\mathbb E_{w^\ast}\|w^\ast\|_H^2
\asymp
\sum_{j\ge1}j^{-b}
\lesssim1.
\]
Hence
\(\mathbb E_{w^\ast}\widetilde\sigma^2(w^\ast)\lesssim1\) and
\(\mathbb E_{w^\ast}\overline\sigma^2(w^\ast)\lesssim1\).
Apply Proposition~\ref{prop:raw-var-controls-variance} and
Lemma~\ref{lem:batch-var-E2} below. The additive-noise lower bound supplies
the matching lower order for the total centered variance.
\end{proof}

\subsection{The WSD influence kernel}
\label{subsec:batch-variance-spectral}

\begin{lemma}[Two-sided WSD influence-kernel bound under power-law decay]
\label{lem:batch-var-E1}
Assume the WSD phase-length and stability conditions in
Theorem~\ref{thm:onepass-batch-sgd-scaling}. On the power-law sketch event
\(\mu_j(\Sigma)\asymp j^{-a}\), suppose
\(A_{\mathrm d}^{\mathrm{op}}=\gamma\tau_{\mathrm d}^{\mathrm{op}}\gtrsim1\). Then
\begin{equation}
\sum_{t=1}^T\kappa_t^{\mathrm{op}}
\asymp
\frac1{\tau_{\mathrm d}^{\mathrm{op}}}
\sum_{j=1}^M
\min\!\left\{
1,(A_{\mathrm d}^{\mathrm{op}}\mu_j(\Sigma))^2
\right\}.
\label{eq:wsd-kernel-effective-dimension}
\end{equation}
Consequently,
\begin{equation}
\sum_{t=1}^T\kappa_t^{\mathrm{op}}
\asymp
\frac{\min\{M,(A_{\mathrm d}^{\mathrm{op}})^{1/a}\}}{\tau_{\mathrm d}^{\mathrm{op}}}.
\label{eq:wsd-kernel-power-law}
\end{equation}
\end{lemma}

\begin{proof}
Write
\[
R_t:=\sum_{s=t+1}^T\gamma_s.
\]
The stability condition and the logarithmic product bound imply
\[
\prod_{s=t+1}^T(1-\gamma_s\mu)^2
\le e^{-c\mu R_t}.
\]
For a power-law spectrum we repeatedly use
\begin{equation}
\sum_{j=1}^M\mu_j^2e^{-r\mu_j}
\lesssim
(1+r)^{-(2-1/a)}.
\label{eq:wsd-spectral-sum}
\end{equation}

First consider the decay phase. With \(t=T_{\mathrm s}+r\), its rates are
\(\gamma e^{-r/\tau_{\mathrm d}^{\mathrm{op}}}\), \(1\le r\le T_{\mathrm d}\), and
\(e^{-T_{\mathrm d}/\tau_{\mathrm d}^{\mathrm{op}}}=T^{-1}\). The geometric filter
calculation, using the logarithmic product bounds valid for
\(\gamma\mu_j\le c_0<1\), gives
\begin{equation}
\sum_{t=T_{\mathrm s}+1}^T\kappa_t^{\mathrm{op}}
\asymp
\frac1{\tau_{\mathrm d}^{\mathrm{op}}}
\sum_{j=1}^M
\min\{1,(A_{\mathrm d}^{\mathrm{op}}\mu_j)^2\}.
\label{eq:wsd-decay-kernel}
\end{equation}
This already proves the required lower bound.

It remains to show that the earlier phases do not have a larger aggregate
order. Every warmup update has at least a constant fraction of the stable
phase after it, so \(R_t\gtrsim\gamma T\). Since
\(\sum_{t\le T_{\mathrm w}}\gamma_t^2\lesssim\gamma^2T\),
\eqref{eq:wsd-spectral-sum} gives
\begin{equation}
\sum_{t=1}^{T_{\mathrm w}}\kappa_t^{\mathrm{op}}
\lesssim
\gamma^{1/a}T^{1/a-1}
\lesssim
\frac{(A_{\mathrm d}^{\mathrm{op}})^{1/a}}{\tau_{\mathrm d}^{\mathrm{op}}}.
\label{eq:wsd-warmup-kernel}
\end{equation}
If \(M<(A_{\mathrm d}^{\mathrm{op}})^{1/a}\), all retained eigenvalues lie above the decay
cutoff. In this case, writing \(y_j:=\gamma T\mu_j\) gives the direct bound
\[
\sum_{t=1}^{T_{\mathrm w}}\kappa_t^{\mathrm{op}}
\lesssim
\frac1T\sum_{j=1}^M y_j^2e^{-c y_j}
\lesssim
\frac{M}{T}
\lesssim
\frac{M}{\tau_{\mathrm d}^{\mathrm{op}}}.
\]

For a stable-phase update write \(k=T_{\mathrm s}-t\). The remaining
learning-rate mass satisfies
\(R_t\gtrsim\gamma(\tau_{\mathrm d}^{\mathrm{op}}+k)\). Summing first over \(k\) yields,
for each eigenvalue,
\[
\sum_{t=T_{\mathrm w}+1}^{T_{\mathrm s}}
\gamma^2\mu_j^2e^{-c\mu_jR_t}
\lesssim
\gamma\mu_j e^{-c'A_{\mathrm d}^{\mathrm{op}}\mu_j}.
\]
Consequently, power-law summation in the unsaturated case and the elementary
bound \(ye^{-c'y}\lesssim1\) in the saturated case give
\begin{equation}
\sum_{t=T_{\mathrm w}+1}^{T_{\mathrm s}}\kappa_t^{\mathrm{op}}
\lesssim
\frac{\min\{M,(A_{\mathrm d}^{\mathrm{op}})^{1/a}\}}{\tau_{\mathrm d}^{\mathrm{op}}}.
\label{eq:wsd-stable-kernel}
\end{equation}
Combining \eqref{eq:wsd-decay-kernel}--\eqref{eq:wsd-stable-kernel} proves
\eqref{eq:wsd-kernel-effective-dimension}. Finally, splitting the spectral
sum at \(j\asymp (A_{\mathrm d}^{\mathrm{op}})^{1/a}\) proves
\eqref{eq:wsd-kernel-power-law}.
\end{proof}

\begin{remark}[Why the power-law condition is used]
For a completely arbitrary spectrum, very small eigenvalues may accumulate a
larger stable-phase contribution than the decay-only expression in
\eqref{eq:wsd-decay-kernel}. The power-law summation in the preceding lemma
shows that this tail is lower order in the present model. Thus the theorem's
WSD kernel rate is valid without any extra batch-schedule assumption, but the
intermediate equivalence is not asserted spectrum-free.
\end{remark}

\begin{lemma}[Power-law bound for the dynamic raw kernel]
\label{lem:batch-var-E2}
Assume Assumptions~\ref{ass:data-power-law} and
\ref{ass:gaussian-sketch}, and suppose
\(A_{\mathrm d}^{\mathrm{op}}=\gamma\tau_{\mathrm d}^{\mathrm{op}}\gtrsim1\). Then, with probability at least
\(1-\exp(-\Omega(M))\) over \(S\),
\[
\sum_{t=1}^T\kappa_t^{\mathrm{op}}
\asymp
\frac{\min\{M,(A_{\mathrm d}^{\mathrm{op}})^{1/a}\}}
{\tau_{\mathrm d}^{\mathrm{op}}},
\]
and
\[
\mathcal K_{B_{1:T}}
=
\sum_{t=1}^T\frac{\kappa_t^{\mathrm{op}}}{B_t}
\asymp
\frac{\min\{M,(A_{\mathrm d}^{\mathrm{op}})^{1/a}\}}
{\tau_{\mathrm d}^{\mathrm{op}}B_{\mathrm{eff}}}.
\]
\end{lemma}

\begin{proof}
Let \(s:=A_{\mathrm d}^{\mathrm{op}}\). On the high-probability sketch event from
Lemma~\ref{lem:aux-G4-E5},
\(\mu_j(\Sigma)\asymp j^{-a}\) for \(j\le M\). Splitting the sum in
Lemma~\ref{lem:batch-var-E1} at
\(k=\min\{M,\lfloor s^{1/a}\rfloor\}\) gives
\[
\sum_{j=1}^M\min\{1,(s\mu_j)^2\}
\asymp
\min\{M,s^{1/a}\}.
\]
This proves the first display. The second follows exactly from
\eqref{eq:dynamic-kernel-effective-batch}; no comparison between the
individual \(B_t\)'s is needed.
\end{proof}

\begin{remark}[Constant-batch consistency check]
If \(B_t\equiv B\), then \(B_{\mathrm{eff}}=B\) and \(T=N/B\). Therefore
\[
\mathcal K_{B_{1:T}}
\asymp
\frac{\min\{M,(A_{\mathrm d}^{\mathrm{op}})^{1/a}\}}
{B\tau_{\mathrm d}^{\mathrm{op}}},
\]
which is the fixed-batch WSD variance law.
\end{remark}

\section{Auxiliary With-Replacement Fluctuation Analysis}
\label{sec:multipass-batch-sgd-wr-fluc}

This proof-only section develops a with-replacement comparison process for a
deterministic batch-size schedule \(B_1,\dots,B_L\), with \(1\le B_t\le N\).
Appendix~\ref{sec:multipass-batch-sgd-wor-fluc} applies the resulting bounds,
together with the finite-population correction, to the without-replacement
multi-pass procedure studied in the main text.

Write
\[
\begin{aligned}
A_L&:=\sum_{t=1}^L\gamma_t,
&A_{\mathrm d}^{\mathrm{mp}}&:=\gamma\tau_{\mathrm d}^{\mathrm{mp}},
&R_t&:=\sum_{s=t+1}^L\gamma_s,\\
\rho_t^{\mathrm{wr}}&:=\frac1{B_t},
&\rho_t^{\mathrm{wor}}&:=\frac{N-B_t}{B_t(N-1)}.
\end{aligned}
\]
For later use, abbreviate the WSD kernel scale by
\[
\mathcal K_{\mathrm d}
:=
\frac{\min\{M,(A_{\mathrm d}^{\mathrm{mp}})^{1/a}\}}
{\tau_{\mathrm d}^{\mathrm{mp}}}.
\]
Lemma~\ref{lem:multipass-wsd-calculus} gives
\(A_L\asymp\gamma L\) and
\(A_{\mathrm d}^{\mathrm{mp}}\asymp\gamma L/\log L\).
Condition on the sketched dataset sigma-field
\(\mathcal D_{\mathrm{data}}=\sigma(S,w^\ast,D)\), and use the filtration
\[
\mathcal F_t^{\mathrm{wr}}
:=
\mathcal D_{\mathrm{data}}\vee
\sigma(i_{s,r}:1\le s\le t,\ 1\le r\le B_s).
\]

Define
\[
\mathrm{Fluc}_{B_{1:L}}^{\mathrm{wr}}
:=
\mathbb E_{\mathrm{batch}}
\bigl[\|\Sigma^{1/2}(u_L^{\mathrm{wr}}-\theta_L)\|_2^2\bigr].
\]
The fluctuation process satisfies
\begin{equation}
\Delta_t
=
(I-\gamma_t\widehat\Sigma_t^{\mathrm{wr}})\Delta_{t-1}
+\gamma_t\xi_t^{\mathrm{wr}},
\qquad
\Delta_0=0,
\label{eq:batch-fluc-recursion}
\end{equation}
where
\begin{equation}
\xi_t^{\mathrm{wr}}
:=
-(\widehat\Sigma_t^{\mathrm{wr}}-\widehat\Sigma)
(\theta_{t-1}-u^\ast)
+(\widehat c_t^{\mathrm{wr}}-\widehat c).
\label{eq:batch-xi-def}
\end{equation}
Since \(\theta_{t-1}\) is determined by \(\mathcal D_{\mathrm{data}}\), the current batch is
independent of \(\mathcal F_{t-1}^{\mathrm{wr}}\) and
\begin{equation}
\mathbb E[
\xi_t^{\mathrm{wr}}
\mid\mathcal F_{t-1}^{\mathrm{wr}}
]
=0.
\label{eq:wr-dynamic-mds}
\end{equation}

\subsection{Conditional covariance identities}
\label{subsec:batch-upper-bound}

For \(i\in[N]\), define the centered single-sample perturbation
\[
\zeta_t(i)
:=
-(Sx_ix_i^\top S^\top-\widehat\Sigma)(\theta_{t-1}-u^\ast)
+(Sx_i\widetilde\varepsilon_i-\widehat c).
\]
Then
\[
\frac1N\sum_{i=1}^N\zeta_t(i)=0,
\qquad
\xi_t^{\mathrm{wr}}
=
\frac1{B_t}\sum_{r=1}^{B_t}\zeta_t(i_{t,r}).
\]
Consequently, if
\[
G_t:=\frac1N\sum_{i=1}^N\zeta_t(i)\zeta_t(i)^\top,
\]
then
\begin{equation}
\mathbb E[
\xi_t^{\mathrm{wr}}(\xi_t^{\mathrm{wr}})^\top
\mid\mathcal F_{t-1}^{\mathrm{wr}}
]
=
\rho_t^{\mathrm{wr}}G_t.
\label{eq:wr-dynamic-noise-cov}
\end{equation}

Let \(z_i:=Sx_i\), \(Z_i:=z_iz_i^\top\), and
\[
C_A:=\max_{i\in[N]}\|z_i\|_2^2.
\]
For any deterministic PSD matrix \(C\), the transition covariance obeys
\begin{equation}
\mathbb E[
(\widehat\Sigma_t^{\mathrm{wr}}-\widehat\Sigma)
C
(\widehat\Sigma_t^{\mathrm{wr}}-\widehat\Sigma)
\mid\mathcal F_{t-1}^{\mathrm{wr}}
]
=
\rho_t^{\mathrm{wr}}
\frac1N\sum_{i=1}^N
(Z_i-\widehat\Sigma)C(Z_i-\widehat\Sigma).
\label{eq:wr-dynamic-transition-cov}
\end{equation}
In particular,
\[
\mathbb E[
(\widehat\Sigma_t^{\mathrm{wr}}-\widehat\Sigma)^2
\mid\mathcal F_{t-1}^{\mathrm{wr}}
]
\preceq
\rho_t^{\mathrm{wr}}C_A\widehat\Sigma.
\]
These identities show that both the additive stochastic source and the
random-transition correction created at iteration \(t\) carry the same
time-dependent factor \(\rho_t^{\mathrm{wr}}\).

\begin{lemma}[Dynamic-batch conditional covariance bound]
\label{lem:wr-dynamic-covariance}
Define
\[
\sigma_\xi^2
:=
2\max_{\substack{i\in[N]\\t\in[L]}}
\left[
\bigl(x_i^\top S^\top(\theta_{t-1}-u^\ast)\bigr)^2
+\widetilde\varepsilon_i^2
\right].
\]
Then, almost surely with respect to \(\mathcal D_{\mathrm{data}}\),
\[
G_t\preceq\sigma_\xi^2\widehat\Sigma
\qquad
\text{for every }t,
\]
and hence
\begin{equation}
\mathbb E[
\xi_t^{\mathrm{wr}}(\xi_t^{\mathrm{wr}})^\top
\mid\mathcal F_{t-1}^{\mathrm{wr}}
]
\preceq
\rho_t^{\mathrm{wr}}\sigma_\xi^2\widehat\Sigma.
\label{eq:batch-claim19}
\end{equation}
\end{lemma}

\begin{proof}
Put \(d_t:=\theta_{t-1}-u^\ast\). The covariance of
\(-(Z_i-\widehat\Sigma)d_t\) is bounded by its second moment, which is at
most
\[
\max_j(z_j^\top d_t)^2\,\widehat\Sigma.
\]
Likewise, the covariance of
\(z_i\widetilde\varepsilon_i-\widehat c\) is at most
\[
\max_j\widetilde\varepsilon_j^2\,\widehat\Sigma.
\]
Using
\((a+b)(a+b)^\top\preceq2aa^\top+2bb^\top\) proves the result.
\end{proof}

For the lower bound, the response noise itself supplies a non-degenerate
stochastic source. Write
\(\varepsilon_i:=y_i-\langle x_i,w^\ast\rangle\), and let
\(\mathbb E_\varepsilon\) denote expectation over
\((\varepsilon_i)_{i=1}^N\) conditional on \(S,w^\ast\), and the covariates.

\begin{lemma}[Lower covariance of the batch perturbation]
\label{lem:wr-dynamic-covariance-lower}
Let \(B_\nu:=\max_{i\in[N]}\|z_i\|_2\), so \(C_A=B_\nu^2\), and define
\[
\beta_{\xi}^{\mathrm{WSD}}
:=
\left[
\frac{N-1}{N}
-
\frac{C A_LB_\nu^2}{N}
\right]_+,
\]
where \(C>0\) is a universal constant.
Under Assumption~\ref{ass:data-well-specified}, for every \(t\in[L]\),
\begin{equation}
\mathbb E_\varepsilon[G_t\mid S,w^\ast,X]
\succeq
\sigma^2\beta_{\xi}^{\mathrm{WSD}}\widehat\Sigma.
\label{eq:dynamic-source-covariance-lower}
\end{equation}
Consequently,
\[
\mathbb E_\varepsilon\!\left[
\mathbb E_{\mathrm{batch}}\!\left[
\xi_t^{\mathrm{wr}}(\xi_t^{\mathrm{wr}})^\top
\mid\mathcal F_{t-1}^{\mathrm{wr}}
\right]
\middle|S,w^\ast,X
\right]
\succeq
\rho_t^{\mathrm{wr}}\sigma^2\beta_{\xi}^{\mathrm{WSD}}\widehat\Sigma.
\]
\end{lemma}

\begin{proof}
Conditional on \(S,w^\ast,X\), the normal-GD iterate is affine in the
response-noise vector. Split the single-sample perturbation as
\(\zeta_t(i)=\zeta_t^{(0)}(i)+\zeta_t^{(\varepsilon)}(i)\), where the first
term is independent of the response noises and the second is linear in them.
The conditional mean-zero and homoscedastic assumptions imply that the cross
terms vanish and hence
\[
\mathbb E_\varepsilon[G_t\mid S,w^\ast,X]
\succeq
\frac1N\sum_{i=1}^N
\mathbb E_\varepsilon
\left[
\zeta_t^{(\varepsilon)}(i)
\zeta_t^{(\varepsilon)}(i)^\top
\mid S,w^\ast,X
\right].
\]

Let
\[
C_{t-1}:=\prod_{s=1}^{t-1}(I-\gamma_s\widehat\Sigma).
\]
Put
\[
V_{t-1}:=(I-C_{t-1})\widehat\Sigma^\dagger,
\qquad
q:=\frac1N\sum_{k=1}^Nz_k\varepsilon_k,
\qquad
A_i:=Z_i-\widehat\Sigma.
\]
The response-noise part of the GD error is \(V_{t-1}q\), and therefore
\[
\zeta_t^{(\varepsilon)}(i)
=
\underbrace{z_i\varepsilon_i-q}_{a_i}
-
\underbrace{A_iV_{t-1}q}_{b_i}.
\]
Expanding \((a_i-b_i)(a_i-b_i)^\top\) exactly, we use
\[
\frac1N\sum_{i=1}^N
\mathbb E_\varepsilon[a_ia_i^\top\mid S,w^\ast,X]
=
\sigma^2\frac{N-1}{N}\widehat\Sigma,
\qquad
\mathbb E_\varepsilon[qq^\top\mid S,w^\ast,X]
=
\frac{\sigma^2}{N}\widehat\Sigma.
\]
Moreover,
\[
\mathbb E_\varepsilon[a_iq^\top\mid S,w^\ast,X]
=\frac{\sigma^2}{N}A_i,
\]
so the two averaged cross terms have operator magnitude at most
\[
\frac{2\sigma^2}{N^2}\sum_{i=1}^NA_iV_{t-1}A_i
\preceq
\frac{2\sigma^2A_LC_A}{N}\widehat\Sigma.
\]
The schedule-dependent filter satisfies
\[
0\preceq
V_{t-1}\widehat\Sigma V_{t-1}
=
(I-C_{t-1})^2\widehat\Sigma^\dagger
\preceq
\bigl(I-C_{t-1}^2\bigr)\widehat\Sigma^\dagger
\preceq
2\sum_{s=1}^{t-1}\gamma_s I
\preceq
2A_LI,
\]
which follows from the scalar inequality
\(1-\prod_s(1-u_s)^2\le2\sum_su_s\). Finally,
\[
\frac1N\sum_{i=1}^N(Z_i-\widehat\Sigma)^2
\preceq
\frac1N\sum_{i=1}^NZ_i^2
\preceq
C_A\widehat\Sigma.
\]
The averaged \(b_ib_i^\top\) term is PSD and may be dropped in this lower
bound. Combining these displays, and enlarging the universal constant \(C\), proves
\eqref{eq:dynamic-source-covariance-lower}.
The batch statement follows by first applying
\eqref{eq:wr-dynamic-noise-cov} conditional on the full dataset and then
averaging over the response noises.
\end{proof}

\subsection{A time-inhomogeneous stochastic approximation lemma}
\label{subsec:batch-D2-analogue}

\begin{lemma}[Stochastic approximation with dynamic covariance factors]
\label{lem:batch-D2-analogue}
Let
\[
\mu_t=(I-\gamma_tA_t)\mu_{t-1}+\gamma_t\xi_t,
\qquad
\mu_0=0,
\]
and let
\(\mathcal F_t=\sigma((A_s,\xi_s):1\le s\le t)\). Suppose the deterministic
numbers \(\rho_t\) lie in \([0,1]\), the current pair
\((A_t,\xi_t)\) is independent of \(\mathcal F_{t-1}\), and, uniformly in
\(t\),
\[
\mathbb E[A_t]=\Sigma_\nu,
\qquad
\mathbb E[\xi_t\mid\mathcal F_{t-1}]=0,
\qquad
\mathbb E[\xi_t\xi_t^\top\mid\mathcal F_{t-1}]
\preceq\rho_t\sigma_\xi^2\Sigma_\nu,
\]
\[
\mathbb E[(A_t-\Sigma_\nu)^2]
\preceq
\rho_t C_A\Sigma_\nu,
\qquad
\mathbb E[A_t^2]\preceq C_A\Sigma_\nu,
\qquad
\|\Sigma_\nu\|_2\le C_A.
\]
Assume also \(0<\gamma_t\le\gamma\) and
\(\gamma C_A\le1/8\).
Let
\[
P_{t+1:L}:=\prod_{s=t+1}^L(I-\gamma_s\Sigma_\nu).
\]
Then
\begin{equation}
\mathbb E[\mu_L\mu_L^\top]
\preceq
\frac98\sigma_\xi^2
\sum_{t=1}^L
\rho_t\gamma_t^2
P_{t+1:L}\Sigma_\nu P_{t+1:L}.
\label{eq:dynamic-random-psd-filter}
\end{equation}
Consequently, for \(u\in[0,1]\),
\[
\mathbb E\|\Sigma_\nu^{u/2}\mu_L\|_2^2
\lesssim
\sigma_\xi^2
\sum_{t=1}^L\rho_t\gamma_t^2
\operatorname{tr}\!\left(
\Sigma_\nu^{u+1}P_{t+1:L}^2
\right).
\]
\end{lemma}

\begin{proof}
The independence of the current pair from the past and the martingale
difference property imply \(\mathbb E[\mu_t]=0\). They also make the mixed
terms in the unconditional second-moment recursion vanish, even though
\(A_t\) and \(\xi_t\) may be dependent within the current pair. Thus, with
\(C_t:=\mathbb E[\mu_t\mu_t^\top]\),
\[
C_t
=
\mathbb E[(I-\gamma_tA_t)C_{t-1}(I-\gamma_tA_t)]
+\gamma_t^2\mathbb E[\xi_t\xi_t^\top].
\]
First, induction gives the schedule-independent crude estimate
\[
C_t\preceq\sigma_\xi^2\gamma I.
\]
Indeed, use \(\rho_t\le1\),
\(\mathbb E[A_t^2]\preceq C_A\Sigma_\nu\), and
\(\gamma C_A\le1/8\).

Next, decompose the transition map around its mean:
\[
\begin{aligned}
\mathbb E[(I-\gamma_tA_t)C_{t-1}(I-\gamma_tA_t)]
={}&
(I-\gamma_t\Sigma_\nu)C_{t-1}(I-\gamma_t\Sigma_\nu)
\\
&+
\gamma_t^2
\mathbb E[(A_t-\Sigma_\nu)C_{t-1}(A_t-\Sigma_\nu)].
\end{aligned}
\]
The crude estimate and the centered transition bound imply that the last
term is at most
\[
\rho_t\sigma_\xi^2\gamma C_A\gamma_t^2\Sigma_\nu.
\]
Adding the driving covariance and using \(\gamma C_A\le1/8\) gives
\[
C_t
\preceq
(I-\gamma_t\Sigma_\nu)C_{t-1}(I-\gamma_t\Sigma_\nu)
+
\frac98\rho_t\sigma_\xi^2\gamma_t^2\Sigma_\nu.
\]
Unrolling proves \eqref{eq:dynamic-random-psd-filter}; taking the appropriate
trace proves the second claim.
\end{proof}

\begin{lemma}[Verification for dynamic with-replacement batches]
\label{lem:wr-batch-verify-D2}
Conditioned on \(\mathcal D_{\mathrm{data}}\), set
\[
A_t:=\widehat\Sigma_t^{\mathrm{wr}},
\qquad
\Sigma_\nu:=\widehat\Sigma,
\qquad
\rho_t:=\rho_t^{\mathrm{wr}}.
\]
Then all transition-matrix assumptions of
Lemma~\ref{lem:batch-D2-analogue} hold. Together with
Lemma~\ref{lem:wr-dynamic-covariance}, its noise assumptions also hold with
the \(\sigma_\xi^2\) defined there.
\end{lemma}

\begin{proof}
The conditional mean is \(\widehat\Sigma\), and
\eqref{eq:wr-dynamic-transition-cov} gives the centered transition bound.
Moreover, pathwise
\[
0\preceq A_t\preceq C_AI,
\qquad
A_t^2\preceq C_AA_t.
\]
Taking expectations gives
\(\mathbb E[A_t^2]\preceq C_A\widehat\Sigma\), while
\(\widehat\Sigma\preceq C_AI\). Finally, on the maximum-norm stability event
controlled by Lemma~\ref{lem:app-derived-multipass-regularity},
\(\gamma C_A\le1/8\).
\end{proof}

\subsection{Influence kernels and covariance replacement}
\label{subsec:multipass-dynamic-kernel}

For \(u\in\{0,1\}\), define the empirical kernels
\[
\widehat\kappa_t^{(u)}
:=
\gamma_t^2
\operatorname{tr}\!\left(
\widehat\Sigma^{u+1}
\prod_{s=t+1}^L(I-\gamma_s\widehat\Sigma)^2
\right).
\]
Define also the empirical cross-kernel
\begin{equation}
\widehat\kappa_t^{(\times)}
:=
\gamma_t^2
\operatorname{tr}\!\left(
\Sigma\widehat\Sigma
\prod_{s=t+1}^L(I-\gamma_s\widehat\Sigma)^2
\right).
\label{eq:dynamic-cross-kernel}
\end{equation}
The population final-risk kernel is
\begin{equation}
\kappa_t^{\mathrm{mp}}
:=
\gamma_t^2
\operatorname{tr}\!\left(
\Sigma^2
\prod_{s=t+1}^L(I-\gamma_s\Sigma)^2
\right)
=
\gamma_t^2
\sum_{j=1}^M\mu_j^2
\prod_{s=t+1}^L(1-\gamma_s\mu_j)^2.
\label{eq:multipass-dynamic-kappa}
\end{equation}
Define
\begin{equation}
\rho_{\mathrm{eff},\mathrm{WSD}}^{\mathrm{wr}}
:=
\frac{\sum_{t=1}^L\rho_t^{\mathrm{wr}}\kappa_t^{\mathrm{mp}}}
{\sum_{t=1}^L\kappa_t^{\mathrm{mp}}},
\qquad
\rho_{\mathrm{eff},\mathrm{WSD}}^{\mathrm{wor}}
:=
\frac{\sum_{t=1}^L\rho_t^{\mathrm{wor}}\kappa_t^{\mathrm{mp}}}
{\sum_{t=1}^L\kappa_t^{\mathrm{mp}}}.
\label{eq:multipass-dynamic-rho-eff}
\end{equation}

\begin{lemma}[Uniform empirical-to-population kernel comparison]
\label{lem:dynamic-kernel-comparison}
Assume the power-law and sketch conditions, and let
\[
\lambda_L:=A_L^{-1}.
\]
Under the horizon condition
\[
A_L\lesssim N^{(1-\varepsilon)a},
\qquad
A_L^{1/a}\lesssim M,
\]
the same high-probability empirical-covariance event used in the fixed-batch
proof gives, uniformly in \(t\),
\[
\widehat\kappa_t^{(1)}
+\lambda_L\widehat\kappa_t^{(0)}
\lesssim
\kappa_t^{\mathrm{mp}}.
\]
Consequently, for every deterministic sequence
\(\rho_t\in[0,1]\),
\[
\sum_{t=1}^L\rho_t
\left(
\widehat\kappa_t^{(1)}
+\lambda_L\widehat\kappa_t^{(0)}
\right)
\lesssim
\sum_{t=1}^L\rho_t\kappa_t^{\mathrm{mp}}.
\]
\end{lemma}

\begin{proof}
For either \(\Sigma\) or \(\widehat\Sigma\), the scalar logarithmic bounds
for \(0\le\gamma_s\nu\le c_0<1\) reduce the two kernels to spectral
sums of the form
\[
\gamma_t^2\sum_j\nu_j^q e^{-cR_t\nu_j},
\qquad
q\in\{1,2\}.
\]
For a power-law spectrum,
\[
\sum_j\nu_j^q e^{-cR\nu_j}
\lesssim
(1+R)^{-(q-1/a)}.
\]
The matching population lower estimate follows by summing over the indices
\(j\asymp(1+R)^{1/a}\). The condition \(A_L^{1/a}\lesssim M\) keeps these
indices below the sketch cutoff uniformly in \(t\). The empirical lower
spectral range from
Lemma~\ref{lem:aux-2025-E8} is sufficient because
\(R_t\le A_L\lesssim N^{(1-\varepsilon)a}\); the remaining empirical tail is
undamped and is controlled by its upper power-law bound.

Finally,
\[
\lambda_L
\frac{(1+R_t)^{-(1-1/a)}}
{(1+R_t)^{-(2-1/a)}}
\lesssim
\frac{1+R_t}{A_L}
\lesssim1.
\]
This proves the pointwise comparison. Multiplication by arbitrary
nonnegative \(\rho_t\)'s and summation preserve it.
\end{proof}

\begin{lemma}[Uniform lower comparison for the cross-kernel]
\label{lem:dynamic-kernel-lower-comparison}
Assume the power-law and Gaussian-design conditions. Suppose
\[
A_L\gtrsim1,
\qquad
A_L^{1/a}\le c_0M,
\qquad
A_L\log N\le c_0N
\]
for a sufficiently small \(a\)-dependent constant \(c_0>0\). Then,
conditioned on the usual high-probability sketch event, with probability at
least \(1-\exp(-\Omega(N))\) over the sample,
\begin{equation}
\widehat\kappa_t^{(\times)}
\gtrsim
\kappa_t^{\mathrm{mp}}
\qquad\text{uniformly for }t\in[L].
\label{eq:dynamic-cross-kernel-lower}
\end{equation}
\end{lemma}

\begin{proof}
Set \(\lambda_0:=A_L^{-1}\). Under the power law,
\[
d_{\lambda_0}
=
\operatorname{tr}\!\left(
\Sigma(\Sigma+\lambda_0I)^{-1}
\right)
\lesssim
A_L^{1/a}.
\]
The assumptions therefore imply the two-sided comparison in
Lemma~\ref{lem:aux-two-sided-covariance} at \(\lambda_0\), and hence at every
\(\lambda\ge\lambda_0\). Intersect this event with the empirical power-law
event in Lemma~\ref{lem:aux-2025-E8}. The condition
\(A_L\log N\le c_0N\) implies
\(A_L^{1/a}\le cN\), so all spectral indices used below lie in the range of
that lemma.

Fix \(t\), write \(R:=R_t\), and first suppose \(R\ge R_0\), where \(R_0\)
is a sufficiently large absolute constant. Put \(\lambda:=R^{-1}\).
Choose fixed constants \(K_2>K_1>1\), sufficiently large relative to the
constants in \eqref{eq:two-sided-regularized-covariance}, and let \(Q_t\) be
the spectral projector of \(\widehat\Sigma\) corresponding to eigenvalues in
\([K_1\lambda,K_2\lambda]\). Lemma~\ref{lem:aux-2025-E8} gives
\[
\operatorname{rank}(Q_t)\asymp R^{1/a}.
\]
Moreover, the upper half of the regularized comparison gives
\[
\widehat\Sigma+\lambda I
\preceq
C(\Sigma+\lambda I).
\]
Consequently, by choosing \(K_1>C-1\),
\[
Q_t\Sigma Q_t\succeq c\lambda Q_t,
\qquad
\operatorname{tr}(\Sigma Q_t)\gtrsim
\lambda R^{1/a}.
\]

Let
\[
h_t(x):=x\prod_{s=t+1}^L(1-\gamma_sx)^2.
\]
The stepsize stability condition and the logarithmic product bound imply
\[
h_t(x)\gtrsim \lambda
\qquad
\text{for }x\in[K_1\lambda,K_2\lambda],
\]
because \(xR\) is bounded above and below by fixed constants on that
interval. Hence
\[
\widehat\kappa_t^{(\times)}
=
\gamma_t^2\operatorname{tr}\bigl(\Sigma h_t(\widehat\Sigma)\bigr)
\gtrsim
\gamma_t^2\lambda\operatorname{tr}(\Sigma Q_t)
\gtrsim
\gamma_t^2R^{1/a-2}.
\]
The population power law and the same product comparison give
\(\kappa_t^{\mathrm{mp}}\lesssim\gamma_t^2R^{1/a-2}\).

It remains to consider \(R<R_0\). On the empirical power-law event,
\(\|\widehat\Sigma\|_2\lesssim1\), so all future products are bounded below
by a fixed positive constant. Therefore
\[
\widehat\kappa_t^{(\times)}
\gtrsim
\gamma_t^2\operatorname{tr}(\Sigma\widehat\Sigma).
\]
The scalar Gaussian concentration bound applied to
\(N^{-1}\sum_i z_i^\top\Sigma z_i\), together with
\(\operatorname{tr}(\Sigma^2)\asymp1\), yields
\(\operatorname{tr}(\Sigma\widehat\Sigma)\gtrsim1\). On the other hand,
\(\kappa_t^{\mathrm{mp}}\lesssim\gamma_t^2\operatorname{tr}(\Sigma^2)
\lesssim\gamma_t^2\). This proves the uniform comparison.
\end{proof}

\begin{lemma}[Shape and total mass of the multi-pass WSD influence kernel]
\label{lem:multipass-kernel-shape}
Under the power-law spectrum,
\[
\kappa_t^{\mathrm{mp}}
\lesssim
\gamma_t^2(1+R_t)^{-(2-1/a)}.
\]
Moreover, if \(A_{\mathrm d}^{\mathrm{mp}}\gtrsim1\), then
\begin{equation}
\sum_{t=1}^L\kappa_t^{\mathrm{mp}}
\asymp
\frac1{\tau_{\mathrm d}^{\mathrm{mp}}}
\sum_{j=1}^M
\min\!\left\{1,(A_{\mathrm d}^{\mathrm{mp}}\mu_j)^2\right\}
\asymp
\frac{\min\{M,(A_{\mathrm d}^{\mathrm{mp}})^{1/a}\}}
{\tau_{\mathrm d}^{\mathrm{mp}}}.
\label{eq:multipass-kernel-total}
\end{equation}
The constants may depend on the fixed WSD phase fractions and on the
stepsize-stability constant.
\end{lemma}

\begin{proof}
The first claim follows from the logarithmic product comparison and the
power-law spectral-sum estimate with exponent two. We split the time sum into
warmup, stable, and decay phases.

For \(t\le L_{\mathrm w}\), the remaining stable phase has mass
\(R_t\gtrsim\gamma L\). Since
\(\sum_{t\le L_{\mathrm w}}\gamma_t^2\lesssim\gamma^2L\),
\[
\sum_{t=1}^{L_{\mathrm w}}\kappa_t^{\mathrm{mp}}
\lesssim
\gamma^{1/a}L^{1/a-1}.
\]
For the stable phase put \(k=L_{\mathrm s}-t\). Direct summation of the
decay schedule gives
\[
R_t\asymp\gamma(k+\tau_{\mathrm d}^{\mathrm{mp}}).
\]
Since \(A_{\mathrm d}^{\mathrm{mp}}\gtrsim1\), summing the pointwise kernel
bound over \(k\ge0\) yields
\[
\sum_{t=L_{\mathrm w}+1}^{L_{\mathrm s}}
\kappa_t^{\mathrm{mp}}
\lesssim
\gamma^{1/a}(\tau_{\mathrm d}^{\mathrm{mp}})^{1/a-1}
=
\frac{(A_{\mathrm d}^{\mathrm{mp}})^{1/a}}
{\tau_{\mathrm d}^{\mathrm{mp}}}.
\]

For the decay phase put \(q=e^{-1/\tau_{\mathrm d}^{\mathrm{mp}}}\) and
\(r=t-L_{\mathrm s}\). Then \(\gamma_t=\gamma q^r\),
\(q^{L_{\mathrm d}}=L^{-1}\), and
\(1-q\asymp1/\tau_{\mathrm d}^{\mathrm{mp}}\). For each eigenvalue \(\mu\),
a geometric-series/integral comparison gives
\begin{equation}
\sum_{t=L_{\mathrm s}+1}^L
\gamma_t^2\mu^2
\prod_{s=t+1}^L(1-\gamma_s\mu)^2
\asymp
\frac1{\tau_{\mathrm d}^{\mathrm{mp}}}
\min\!\left\{1,(A_{\mathrm d}^{\mathrm{mp}}\mu)^2\right\}.
\label{eq:multipass-wsd-decay-filter}
\end{equation}
Indeed, after setting
\(z_r=A_{\mathrm d}^{\mathrm{mp}}\mu q^r\), the sum is comparable to
\((\tau_{\mathrm d}^{\mathrm{mp}})^{-2}
\sum_r z_r^2e^{-c z_r}\), which is of order
\((\tau_{\mathrm d}^{\mathrm{mp}})^{-1}\min\{1,z_1^2\}\).
Thus the decay phase alone supplies the lower bound.

For the upper bound, the warmup estimate is smaller than the decay order by a
fixed power of \(\log L\). The stable estimate has the same order when the
cutoff lies below \(M\); when \(M<(A_{\mathrm d}^{\mathrm{mp}})^{1/a}\),
the per-eigenvalue bound
\(z e^{-cz}\lesssim1\), with
\(z=A_{\mathrm d}^{\mathrm{mp}}\mu_j\), improves it to
\(O(M/\tau_{\mathrm d}^{\mathrm{mp}})\). Hence all three phases are bounded
by the decay expression. Finally,
\[
\sum_{j=1}^M\min\{1,(A_{\mathrm d}^{\mathrm{mp}}\mu_j)^2\}
\asymp
\min\{M,(A_{\mathrm d}^{\mathrm{mp}})^{1/a}\}
\]
under \(\mu_j\asymp j^{-a}\), completing the proof.
\end{proof}

\begin{remark}[Exact weights are essential]
The phase estimates determine the total mass, but do not give a uniform
phasewise proxy for arbitrary dynamic batch schedules. We therefore retain
the exact \(\kappa_t^{\mathrm{mp}}\) in
\(\rho_{\mathrm{eff},\mathrm{WSD}}\), especially near phase boundaries and
the final endpoint.
\end{remark}

\subsection{Dynamic fluctuation bounds}
\label{subsec:wr-batch-fluc-source}

For \(i\in[N]\), let \(\theta_t^{(-i)}\) be the leave-one-out GD iterate
\[
\theta_t^{(-i)}
=
(I-\gamma_t\widehat\Sigma^{(-i)})\theta_{t-1}^{(-i)}
+\gamma_t(SX^\top y)^{(-i)},
\qquad
\theta_0^{(-i)}=0,
\]
where
\[
\widehat\Sigma^{(-i)}
:=
\frac1N\sum_{j\ne i}Sx_jx_j^\top S^\top,
\qquad
(SX^\top y)^{(-i)}
:=
\frac1N\sum_{j\ne i}Sx_jy_j.
\]
Define
\[
a_{\max}
:=
\max_{\substack{i\in[N]\\t\in[L]}}
|y_i-x_i^\top S^\top\theta_t^{(-i)}|,
\qquad
\lambda_L:=A_L^{-1},
\]
\[
R_p
:=
\|(\Sigma+\lambda_LI)^{1/2}
(\widehat\Sigma+\lambda_LI)^{-1/2}\|_2^2,
\]
\[
B_\Delta
:=
a_{\max}^2
\max_i\|Sx_i\|_2^2
R_p\frac{A_L^{2-s}}{N^2},
\]
and
\[
F_{\mathrm{mp}}
:=
R_p\left(
\max_i(x_i^\top S^\top u^\ast)^2
+\max_i\widetilde\varepsilon_i^2
+\max_{i,t}(x_i^\top S^\top\theta_t^{(-i)})^2
+\max_i\|x_i^\top S^\top\|_{\Sigma^{-s}}^2B_\Delta
\right).
\]
The leave-one-out argument used in the fixed-batch proof gives
\[
\sigma_\xi^2\lesssim F_{\mathrm{mp}}/R_p.
\]

\begin{lemma}[Upper bound for dynamic with-replacement batches]
\label{lem:wr-batch-fluc-upper}
Assume Assumptions~\ref{ass:data-gaussian},
\ref{ass:data-well-specified}, \ref{ass:data-power-law},
\ref{ass:gaussian-sketch}, and \ref{ass:stepsize}, and work on the
maximum-norm stability event controlled by
Lemma~\ref{lem:app-derived-multipass-regularity}. For a fixed
\(\varepsilon\in(0,1)\), suppose
\[
A_L\lesssim N^{(1-\varepsilon)a},
\qquad
A_L^{1/a}\lesssim M,
\].
Then
\[
\mathbb E[
\mathrm{Fluc}_{B_{1:L}}^{\mathrm{wr}}
\mid\mathcal D_{\mathrm{data}}
]
\lesssim
F_{\mathrm{mp}}
\sum_{t=1}^L
\rho_t^{\mathrm{wr}}\kappa_t^{\mathrm{mp}}.
\]
\end{lemma}

\begin{proof}
Apply Lemma~\ref{lem:batch-D2-analogue} with
\(A_t=\widehat\Sigma_t^{\mathrm{wr}}\) and
\(\Sigma_\nu=\widehat\Sigma\). The covariance replacement
\[
\|\Sigma^{1/2}\Delta_L\|_2^2
\le
R_p\left(
\|\widehat\Sigma^{1/2}\Delta_L\|_2^2
+\lambda_L\|\Delta_L\|_2^2
\right)
\]
and \(\sigma_\xi^2\lesssim F_{\mathrm{mp}}/R_p\) yield
\[
\mathbb E[
\mathrm{Fluc}_{B_{1:L}}^{\mathrm{wr}}
\mid\mathcal D_{\mathrm{data}}
]
\lesssim
F_{\mathrm{mp}}
\sum_{t=1}^L\rho_t^{\mathrm{wr}}
\left(
\widehat\kappa_t^{(1)}
+\lambda_L\widehat\kappa_t^{(0)}
\right).
\]
Lemma~\ref{lem:dynamic-kernel-comparison} proves the claim.
\end{proof}

\begin{lemma}[Empirical lower bound for dynamic with-replacement batches]
\label{lem:wr-batch-fluc-lower-empirical}
Under the notation above,
\begin{equation}
\mathbb E_{\varepsilon}
\left[
\mathrm{Fluc}_{B_{1:L}}^{\mathrm{wr}}
\mid S,w^\ast,X
\right]
\ge
\sigma^2\beta_{\xi}^{\mathrm{WSD}}
\sum_{t=1}^L
\rho_t^{\mathrm{wr}}\widehat\kappa_t^{(\times)}.
\label{eq:dynamic-fluc-empirical-lower}
\end{equation}
\end{lemma}

\begin{proof}
First condition on the entire sketched dataset. Since
\(\mathbb E_{\mathrm{batch}}[\Delta_t]=0\), the mixed terms vanish in the
conditional second-moment recursion. For every deterministic PSD matrix
\(C\),
\[
\begin{aligned}
&\mathbb E_{\mathrm{batch}}\!\left[
(I-\gamma_t\widehat\Sigma_t^{\mathrm{wr}})
C
(I-\gamma_t\widehat\Sigma_t^{\mathrm{wr}})
\mid\mathcal F_{t-1}^{\mathrm{wr}}
\right]
\\
&\qquad=
(I-\gamma_t\widehat\Sigma)C
(I-\gamma_t\widehat\Sigma)
\\
&\qquad\quad+
\gamma_t^2
\mathbb E_{\mathrm{batch}}\!\left[
(\widehat\Sigma_t^{\mathrm{wr}}-\widehat\Sigma)
C
(\widehat\Sigma_t^{\mathrm{wr}}-\widehat\Sigma)
\mid\mathcal F_{t-1}^{\mathrm{wr}}
\right]
\\
&\qquad\succeq
(I-\gamma_t\widehat\Sigma)C
(I-\gamma_t\widehat\Sigma).
\end{aligned}
\]
Unrolling this lower recursion and then taking expectation over the response
noises gives, by Lemma~\ref{lem:wr-dynamic-covariance-lower},
\[
\mathbb E_{\varepsilon,\mathrm{batch}}[\Delta_L\Delta_L^\top
\mid S,w^\ast,X]
\succeq
\sigma^2\beta_{\xi}^{\mathrm{WSD}}
\sum_{t=1}^L
\rho_t^{\mathrm{wr}}\gamma_t^2
\prod_{s=t+1}^L(I-\gamma_s\widehat\Sigma)
\widehat\Sigma
\prod_{s=t+1}^L(I-\gamma_s\widehat\Sigma).
\]
Multiplying by \(\Sigma\), taking the trace, and using
\eqref{eq:dynamic-cross-kernel} proves the claim.
\end{proof}

\begin{lemma}[Source-condition bound for dynamic with-replacement batches]
\label{lem:wr-batch-fluc-source}
Assume the conditions of Theorem~\ref{thm:multipass-batch-sgd-scaling}.
Then, with probability at least \(1-\exp(-\Omega(M))\) over \(S\),
\begin{align}
\sum_{t=1}^L\rho_t^{\mathrm{wr}}\kappa_t^{\mathrm{mp}}
&\lesssim
\mathbb E[\mathrm{Fluc}_{B_{1:L}}^{\mathrm{wr}}]
\lesssim
\log N
\left[
1+\frac{\log^2N\,A_L^{2-s}}{N^2}
\right]
\sum_{t=1}^L\rho_t^{\mathrm{wr}}\kappa_t^{\mathrm{mp}},
\label{eq:dynamic-fluc-kernel-sharp}\\
\rho_{\mathrm{eff},\mathrm{WSD}}^{\mathrm{wr}}
\mathcal K_{\mathrm d}
&\lesssim
\mathbb E[\mathrm{Fluc}_{B_{1:L}}^{\mathrm{wr}}]
\lesssim
\rho_{\mathrm{eff},\mathrm{WSD}}^{\mathrm{wr}}
\log N
\left[
1+\frac{\log^2N\,A_L^{2-s}}{N^2}
\right]
\mathcal K_{\mathrm d}.
\label{eq:dynamic-fluc-source-sharp}
\end{align}
In particular, the fluctuation rate is sharp up to the logarithmic factor.
\end{lemma}

\begin{proof}
For the upper bound, choose a fixed
\(\varepsilon\in(0,1-1/a)\). The horizon condition in the theorem implies
\(A_L\log N\le c_0N\), and hence
\(A_L\lesssim N^{(1-\varepsilon)a}\), and the leave-one-out argument gives
\[
\mathbb E[F_{\mathrm{mp}}\mid S,w^\ast]
\lesssim
(\sigma^2+\|w^\ast\|_H^2)\log N
\left[
1+\frac{\log^2N\,A_L^{2-s}}{N^2}
\right]
\]
for
\[
s=1-\frac1{a(1-\varepsilon/2)}.
\]
As in the fixed-batch proof, the theorem's moderate-reuse condition also gives
\[
\frac{\log^2N\,A_L^{2-s}}{N^2}
\lesssim
\frac{A_L}{N}
\lesssim1.
\]
Taking expectations in Lemma~\ref{lem:wr-batch-fluc-upper} proves the upper
half of \eqref{eq:dynamic-fluc-kernel-sharp}.

For the lower bound, condition on a sketch in the power-law event. Gaussian
concentration and a union bound give
\[
C_A=\max_{i\in[N]}\|Sx_i\|_2^2\lesssim\log N
\]
with probability at least \(1-\exp(-\Omega(N))\) over the covariates. If the
constant in \(A_L\log N\le c_0N\) is sufficiently small, this event implies
\(\beta_{\xi}^{\mathrm{WSD}}\ge c_\xi\) for an absolute
\(c_\xi>0\). Intersect it with the event in
Lemma~\ref{lem:dynamic-kernel-lower-comparison}. On this intersection,
Lemma~\ref{lem:wr-batch-fluc-lower-empirical} gives
\[
\mathbb E_{\varepsilon}
\left[
\mathrm{Fluc}_{B_{1:L}}^{\mathrm{wr}}
\mid S,w^\ast,X
\right]
\gtrsim
\sum_{t=1}^L
\rho_t^{\mathrm{wr}}\kappa_t^{\mathrm{mp}}.
\]
The intersection has probability bounded below by an absolute constant, so
taking expectation over the covariates and \(w^\ast\) proves the lower half
of \eqref{eq:dynamic-fluc-kernel-sharp}.

Finally,
\[
\sum_t\rho_t^{\mathrm{wr}}\kappa_t^{\mathrm{mp}}
=
\rho_{\mathrm{eff},\mathrm{WSD}}^{\mathrm{wr}}
\sum_t\kappa_t^{\mathrm{mp}},
\]
and the sharp total-mass estimate
\eqref{eq:multipass-kernel-total} proves
\eqref{eq:dynamic-fluc-source-sharp}.
\end{proof}

\begin{remark}[Constant-batch sharp rate]
If \(B_t\equiv B\), then
\(\rho_t^{\mathrm{wr}}\equiv1/B\) and
\(\rho_{\mathrm{eff},\mathrm{WSD}}^{\mathrm{wr}}=1/B\) algebraically.
Consequently,
\[
\frac{\mathcal K_{\mathrm d}}{B}
\lesssim
\mathbb E[\mathrm{Fluc}_{B}^{\mathrm{wr}}]
\lesssim
\frac{\log N}{B}
\left[
1+\frac{\log^2N\,A_L^{2-s}}{N^2}
\right]
\mathcal K_{\mathrm d}.
\]
\end{remark}

\section{Fluctuation Error under Multi-pass Batch SGD without Replacement}
\label{sec:multipass-batch-sgd-wor-fluc}

This section treats the deterministic schedule \(B_1,\dots,B_L\) when the
batch at each iteration is sampled uniformly without replacement. Define
\[
\rho_t^{\mathrm{wor}}
:=
\frac{N-B_t}{B_t(N-1)},
\qquad
\mathcal F_t^{\mathrm{wor}}
:=
\mathcal D_{\mathrm{data}}\vee\sigma(I_1,\dots,I_t).
\]
The fluctuation process is
\[
\Delta_t^{\mathrm{wor}}
=
(I-\gamma_t\widehat\Sigma_t^{\mathrm{wor}})
\Delta_{t-1}^{\mathrm{wor}}
+\gamma_t\xi_t^{\mathrm{wor}},
\qquad
\Delta_0^{\mathrm{wor}}=0,
\]
and
\[
\mathbb E[
\xi_t^{\mathrm{wor}}
\mid\mathcal F_{t-1}^{\mathrm{wor}}
]
=0.
\]
Set
\[
\mathrm{Fluc}_{B_{1:L}}^{\mathrm{wor}}
:=
\mathbb E_{\mathrm{batch}}
\|\Sigma^{1/2}\Delta_L^{\mathrm{wor}}\|_2^2.
\]

\begin{lemma}[Finite-population covariance identity]
\label{lem:wor-finite-pop-cov}
Let \(v_1,\dots,v_N\) be vectors in a finite-dimensional inner-product
space and suppose \(N^{-1}\sum_i v_i=0\). If \(I\subset[N]\) is uniform
among subsets of size \(B\), then
\[
\mathbb E_I\left[
\left(\frac1B\sum_{i\in I}v_i\right)
\left(\frac1B\sum_{i\in I}v_i\right)^\top
\right]
=
\frac{N-B}{B(N-1)}
\frac1N\sum_{i=1}^Nv_iv_i^\top.
\]
The identity also holds for centered matrices \(V_i\) in the form
\[
\mathbb E_I\left[
\left(\frac1B\sum_{i\in I}V_i\right)
C
\left(\frac1B\sum_{i\in I}V_i\right)
\right]
=
\frac{N-B}{B(N-1)}
\frac1N\sum_{i=1}^NV_iCV_i
\]
for every deterministic PSD matrix \(C\).
\end{lemma}

\begin{proof}
Expand the two sums and use
\[
\mathbb P(i\in I)=\frac BN,
\qquad
\mathbb P(i,j\in I)=\frac{B(B-1)}{N(N-1)}
\quad(i\ne j),
\]
together with
\(\sum_{i\ne j}v_iv_j^\top=-\sum_iv_iv_i^\top\). The matrix statement is
the same calculation applied to \(V_iC^{1/2}\).
\end{proof}

For the single-sample perturbations \(\zeta_t(i)\) from
Appendix~\ref{sec:multipass-batch-sgd-wr-fluc},
\[
\xi_t^{\mathrm{wor}}
=
\frac1{B_t}\sum_{i\in I_t}\zeta_t(i).
\]
Conditionally on \(\mathcal F_{t-1}^{\mathrm{wor}}\), the family
\((\zeta_t(i))_{i=1}^N\) is deterministic and centered. Therefore
\begin{equation}
\mathbb E[
\xi_t^{\mathrm{wor}}(\xi_t^{\mathrm{wor}})^\top
\mid\mathcal F_{t-1}^{\mathrm{wor}}
]
=
\rho_t^{\mathrm{wor}}G_t.
\label{eq:wor-dynamic-noise-cov}
\end{equation}
Likewise, with \(Z_i=Sx_ix_i^\top S^\top\),
\begin{equation}
\begin{aligned}
&\mathbb E[
(\widehat\Sigma_t^{\mathrm{wor}}-\widehat\Sigma)
C
(\widehat\Sigma_t^{\mathrm{wor}}-\widehat\Sigma)
\mid\mathcal F_{t-1}^{\mathrm{wor}}
]
\\
&\qquad=
\rho_t^{\mathrm{wor}}
\frac1N\sum_{i=1}^N
(Z_i-\widehat\Sigma)C(Z_i-\widehat\Sigma).
\end{aligned}
\label{eq:wor-dynamic-transition-cov}
\end{equation}
In particular,
\[
\mathbb E[
(\widehat\Sigma_t^{\mathrm{wor}}-\widehat\Sigma)^2
\mid\mathcal F_{t-1}^{\mathrm{wor}}
]
\preceq
\rho_t^{\mathrm{wor}}C_A\widehat\Sigma.
\]

\begin{corollary}[Dynamic fluctuation bounds without replacement]
\label{cor:wor-batch-fluc}
Under the assumptions of Lemma~\ref{lem:wr-batch-fluc-upper},
\[
\mathbb E[
\mathrm{Fluc}_{B_{1:L}}^{\mathrm{wor}}
\mid\mathcal D_{\mathrm{data}}
]
\lesssim
F_{\mathrm{mp}}
\sum_{t=1}^L
\rho_t^{\mathrm{wor}}\kappa_t^{\mathrm{mp}}.
\]
Under the source-condition assumptions of
Lemma~\ref{lem:wr-batch-fluc-source},
\begin{align}
\sum_{t=1}^L\rho_t^{\mathrm{wor}}\kappa_t^{\mathrm{mp}}
&\lesssim
\mathbb E[\mathrm{Fluc}_{B_{1:L}}^{\mathrm{wor}}]
\lesssim
\log N
\left[
1+\frac{\log^2N\,A_L^{2-s}}{N^2}
\right]
\sum_{t=1}^L\rho_t^{\mathrm{wor}}\kappa_t^{\mathrm{mp}},
\label{eq:wor-fluc-kernel-sharp}\\
\rho_{\mathrm{eff},\mathrm{WSD}}^{\mathrm{wor}}
\mathcal K_{\mathrm d}
&\lesssim
\mathbb E[\mathrm{Fluc}_{B_{1:L}}^{\mathrm{wor}}]
\lesssim
\rho_{\mathrm{eff},\mathrm{WSD}}^{\mathrm{wor}}
\log N
\left[
1+\frac{\log^2N\,A_L^{2-s}}{N^2}
\right]
\mathcal K_{\mathrm d}.
\label{eq:wor-fluc-source-sharp}
\end{align}
\end{corollary}

\begin{proof}
Equations~\eqref{eq:wor-dynamic-noise-cov} and
\eqref{eq:wor-dynamic-transition-cov} verify the two
\(\rho_t\)-weighted covariance assumptions in
Lemma~\ref{lem:batch-D2-analogue}. The remaining conditions follow
pathwise from
\[
0\preceq\widehat\Sigma_t^{\mathrm{wor}}\preceq C_AI,
\qquad
(\widehat\Sigma_t^{\mathrm{wor}})^2
\preceq C_A\widehat\Sigma_t^{\mathrm{wor}},
\]
and from independence of the batches across iterations. The proof of
Lemmas~\ref{lem:wr-batch-fluc-upper} and
\ref{lem:wr-batch-fluc-source} then applies with
\(\rho_t^{\mathrm{wr}}\) replaced by \(\rho_t^{\mathrm{wor}}\).
For the lower bound, the finite-population identity gives
\[
\mathbb E_\varepsilon\!\left[
\mathbb E_{\mathrm{batch}}\!\left[
\xi_t^{\mathrm{wor}}(\xi_t^{\mathrm{wor}})^\top
\mid\mathcal F_{t-1}^{\mathrm{wor}}
\right]
\middle|S,w^\ast,X
\right]
\succeq
\rho_t^{\mathrm{wor}}\sigma^2\beta_{\xi}^{\mathrm{WSD}}\widehat\Sigma
\]
by Lemma~\ref{lem:wr-dynamic-covariance-lower}. The transition comparison in
the proof of Lemma~\ref{lem:wr-batch-fluc-lower-empirical} is unchanged.
Thus
\[
\mathbb E_{\varepsilon}
\left[
\mathrm{Fluc}_{B_{1:L}}^{\mathrm{wor}}
\mid S,w^\ast,X
\right]
\ge
\sigma^2\beta_{\xi}^{\mathrm{WSD}}
\sum_{t=1}^L
\rho_t^{\mathrm{wor}}\widehat\kappa_t^{(\times)}.
\]
Applying Lemma~\ref{lem:dynamic-kernel-lower-comparison}, the
\(C_A\lesssim\log N\) event, and
\eqref{eq:multipass-kernel-total} proves both lower bounds.
\end{proof}

\begin{remark}[Full-batch consistency]
If \(B_t=N\) for every \(t\), then
\(\rho_t^{\mathrm{wor}}=0\) and
\(\widehat\Sigma_t^{\mathrm{wor}}=\widehat\Sigma\),
\(\widehat b_t^{\mathrm{wor}}=\widehat b\). Thus
\[
u_t^{\mathrm{wor}}=\theta_t
\qquad\text{for all }t,
\]
and the fluctuation error is exactly zero.
\end{remark}

\begin{remark}[Constant-batch sharp rate]
If \(B_t\equiv B\), then
\[
\rho_t^{\mathrm{wor}}
\equiv
\frac{N-B}{B(N-1)}
=
\rho_{\mathrm{eff},\mathrm{WSD}}^{\mathrm{wor}}.
\]
Consequently, in the sharp regime of the theorem,
\[
\frac{N-B}{B(N-1)}
\mathcal K_{\mathrm d}
\lesssim
\mathbb E[\mathrm{Fluc}_{B}^{\mathrm{wor}}]
\lesssim
\frac{N-B}{B(N-1)}\log N
\left[
1+\frac{\log^2N\,A_L^{2-s}}{N^2}
\right]
\mathcal K_{\mathrm d}.
\]
\end{remark}
\section{Collected Auxiliary Lemmas}
\label{sec:auxiliary-lemmas}

This section collects the spectral, moment, and concentration ingredients used
repeatedly across the appendix proofs for
\eqref{eq:proc-onepass-batch-sgd} and for the auxiliary multi-pass analyses in
Appendices~\ref{sec:multipass-batch-sgd-wr-fluc} and
\ref{sec:multipass-batch-sgd-wor-fluc}.

We use the common notation from Section~\ref{sec:app-prelim}, especially Section~\ref{subsec:app-block-notation}. Let \((\lambda_i)_{i\ge 1}\) denote the eigenvalues of \(H\) in non-increasing order. For integers \(0\le k_\ast\le k\le \infty\), define
\[
\Sigma_{k_\ast:k} := S_{k_\ast:k}H_{k_\ast:k}S_{k_\ast:k}^\top,
\qquad
\Sigma_{k:\infty} := S_{k:\infty}H_{k:\infty}S_{k:\infty}^\top.
\]
For any symmetric PSD matrix \(A\), we write \(\mu_1(A)\ge \mu_2(A)\ge \cdots\) for its eigenvalues.

\subsection{General concentration lemmas}

\begin{lemma}[Covariance replacement]
\label{lem:aux-2025-E1}
Let \(\lambda>0\). If
\[
\sum_{i=1}^M \frac{\mu_i(\Sigma)}{\mu_i(\Sigma)+\lambda}\le \frac{N}{4},
\]
then with probability at least \(1-\exp(-\Omega(N))\) over the sample,
\[
\bigl\|(\widehat{\Sigma}+\lambda I)^{-1/2}(\Sigma+\lambda I)^{1/2}\bigr\|_2\le 3.
\]
Moreover,
\[
\mathbb{E}_X\bigl\|(\widehat{\Sigma}+\lambda I)^{-1/2}(\Sigma+\lambda I)^{1/2}\bigr\|_2^4
\le
100+\exp(-cN)\frac{\|\Sigma\|_2^2}{\lambda^2}
\]
for some absolute constant \(c>0\).
\end{lemma}

\begin{lemma}[Two-sided regularized covariance comparison]
\label{lem:aux-two-sided-covariance}
Condition on \(S\), and suppose \(z=Sx\sim\mathcal N(0,\Sigma)\). For
\(\lambda>0\), let
\[
d_{\lambda}
:=
\operatorname{tr}\!\left(\Sigma(\Sigma+\lambda I)^{-1}\right).
\]
There exist absolute constants \(c_0,c_1,c_2,c_3>0\) such that, whenever
\(d_{\lambda}\le c_0N\), with probability at least
\(1-2\exp(-c_1N)\) over the sample,
\begin{equation}
c_2(\Sigma+\lambda I)
\preceq
\widehat\Sigma+\lambda I
\preceq
c_3(\Sigma+\lambda I).
\label{eq:two-sided-regularized-covariance}
\end{equation}
Moreover, if \eqref{eq:two-sided-regularized-covariance} holds at
\(\lambda_0\), then, after changing only the absolute constants, it holds
simultaneously for every \(\lambda\ge\lambda_0\).
\end{lemma}

\begin{proof}
Put
\[
g_i:=(\Sigma+\lambda I)^{-1/2}z_i,
\qquad
K:=\mathbb E[g_ig_i^\top]
=
(\Sigma+\lambda I)^{-1/2}\Sigma(\Sigma+\lambda I)^{-1/2}.
\]
Then \(0\preceq K\preceq I\) and
\(\operatorname{tr}(K)=d_\lambda\). The standard Gaussian covariance
concentration bound, applied to \(N^{-1}\sum_i g_ig_i^\top\), gives
\[
\left\|
\frac1N\sum_{i=1}^Ng_ig_i^\top-K
\right\|_2
\le \frac12
\]
with probability at least \(1-2\exp(-c_1N)\), provided the absolute constant
in \(d_\lambda\le c_0N\) is sufficiently small. Conjugating the inequalities
\(-I/2\preceq N^{-1}\sum_i g_ig_i^\top-K\preceq I/2\) by
\((\Sigma+\lambda I)^{1/2}\), and then adding \(\lambda I\), proves
\eqref{eq:two-sided-regularized-covariance}.

For \(\lambda\ge\lambda_0\), add
\((\lambda-\lambda_0)I\) to all three matrices in the comparison at
\(\lambda_0\).  Since the comparison constants can be enlarged toward one,
the same two-sided estimate follows uniformly for every larger regularization
level.
\end{proof}

\begin{lemma}[Sketched fourth-moment and residual covariance bounds]
\label{lem:aux-sketch-moment}
Assume Assumptions~\ref{ass:data-gaussian} and
\ref{ass:data-well-specified}. Condition on the sketch matrix \(S\) and the
target \(w^\ast\), and let
\[
z:=Sx,
\qquad
\Sigma:=SHS^\top,
\qquad
u^\ast:=\Sigma^{-1}SHw^\ast.
\]
Then there exists an absolute constant \(\alpha>0\) such that the following
hold.
\begin{enumerate}[label=\textnormal{(\roman*)}]
    \item For every PSD matrix \(A\in\mathbb R^{M\times M}\),
    \[
    \mathbb E[zz^\top A zz^\top]
    \preceq
    \alpha\,\operatorname{tr}(\Sigma A)\Sigma.
    \]

    \item Writing
    \[
    \widetilde{\sigma}^2(w^\ast):=2\bigl(\sigma^2+\alpha\|w^\ast\|_H^2\bigr),
    \]
    one has
    \[
    \mathbb E\!\left[(y-z^\top u^\ast)^2 zz^\top\right]
    \preceq
    \widetilde{\sigma}^2(w^\ast)\,\Sigma.
    \]
\end{enumerate}
Under Gaussian design, one may take \(\alpha=3\).
\end{lemma}

\begin{proof}
Since \(z\sim\mathcal N(0,\Sigma)\), Isserlis' formula gives
\[
\mathbb E[zz^\top Azz^\top]
=
\operatorname{tr}(\Sigma A)\Sigma+2\Sigma A\Sigma.
\]
For \(A\succeq0\),
\(\Sigma A\Sigma\preceq\operatorname{tr}(\Sigma A)\Sigma\), so part
\textnormal{(i)} follows with \(\alpha=3\).

For part \textnormal{(ii)}, write
\[
\varepsilon:=y-x^\top w^\ast,
\qquad
q:=w^\ast-S^\top u^\ast.
\]
Then \(y-z^\top u^\ast=\varepsilon+x^\top q\). By
Assumption~\ref{ass:data-well-specified},
\[
\mathbb E[\varepsilon^2zz^\top]
=
\mathbb E\!\left[
\mathbb E[\varepsilon^2\mid x,w^\ast]zz^\top
\right]
=
\sigma^2\Sigma.
\]
Moreover, another application of the Gaussian fourth-moment identity yields
\[
\mathbb E[(x^\top q)^2zz^\top]
\preceq
\alpha\|q\|_H^2\Sigma.
\]
The definition \(u^\ast=\Sigma^{-1}SHw^\ast\) makes \(S^\top u^\ast\)
the \(H\)-orthogonal projection of \(w^\ast\) onto
\(\operatorname{range}(S^\top)\), and hence
\(\|q\|_H\le\|w^\ast\|_H\). Therefore, using
\((r+s)^2\le2r^2+2s^2\),
\[
\mathbb E[(y-z^\top u^\ast)^2zz^\top]
\preceq
2\bigl(\sigma^2+\alpha\|w^\ast\|_H^2\bigr)\Sigma,
\]
as claimed.
\end{proof}

\begin{lemma}[Head--tail eigenvalue comparison]
\label{lem:aux-G1-E2}
There exists an absolute constant \(c>1\) such that for every \(0\le k\le M\), with probability at least
\[
1-\exp(-\Omega(M)) - \exp(-\Omega(k)),
\]
we have for every \(j\in[M]\),
\[
\left|\mu_j(\Sigma)-\lambda_j-\frac{1}{M}\sum_{i>k}\lambda_i\right|
\le
c\left(
\frac{k}{M}\lambda_j
+\lambda_{k+1}
+\sqrt{\frac{\sum_{i>k}\lambda_i^2}{M}}
\right).
\]
In particular, if \(k\le M/c^2\), then the term \(\tfrac{k}{M}\lambda_j\) can be absorbed into the left-hand side.
\end{lemma}

\begin{lemma}[Tail concentration]
\label{lem:aux-G2-E3}
For any \(k\ge 0\), with probability at least \(1-\delta\),
\[
\left\|\Sigma_{k:\infty}-\frac{1}{M}\sum_{i>k}\lambda_i\,I_M\right\|_2
\lesssim
\lambda_{k+1}\left(1+\frac{\log(1/\delta)}{M}\right)
+
\sqrt{\frac{\sum_{i>k}\lambda_i^2}{M}\left(1+\frac{\log(1/\delta)}{M}\right)}.
\]
In particular, with probability at least \(1-\exp(-\Omega(M))\),
\[
\left\|\Sigma_{k:\infty}-\frac{1}{M}\sum_{i>k}\lambda_i\,I_M\right\|_2
\lesssim
\lambda_{k+1}
+
\sqrt{\frac{\sum_{i>k}\lambda_i^2}{M}}.
\]
\end{lemma}

\begin{lemma}[Head concentration]
\label{lem:aux-G3-E4}
For any \(k\ge 1\), with probability at least \(1-\delta\),
\[
\bigl|\mu_j(\Sigma_{0:k})-\lambda_j\bigr|
\lesssim
\frac{k+\log(1/\delta)}{M}\,\lambda_j,
\qquad j\le k.
\]
In particular, with probability at least \(1-\exp(-\Omega(k))\),
\[
\bigl|\mu_j(\Sigma_{0:k})-\lambda_j\bigr|
\lesssim
\frac{k}{M}\lambda_j,
\qquad j\le k.
\]
\end{lemma}

\begin{lemma}[Head--tail resolvent estimate]
\label{lem:aux-head-tail-resolvent}
Fix an integer \(k\le M/3\) such that \(\operatorname{rank}(H)\ge k+M\), and define
\[
A_k:=S_{k:\infty}H_{k:\infty}S_{k:\infty}^\top,
\qquad
\Sigma=S_{0:k}H_{0:k}S_{0:k}^\top+A_k.
\]
Then, with probability at least
\[
1-\exp(-\Omega(M)),
\]
one has
\[
\|\Sigma^{-1}S_{0:k}H_{0:k}\|_2
\lesssim
\frac{\mu_{M/2}(A_k)}{\mu_M(A_k)}.
\]
\end{lemma}

\subsection{Power-law auxiliary lemmas}

\begin{lemma}[Power-law spectrum of the sketched covariance]
\label{lem:aux-G4-E5}
Suppose the population spectrum obeys \(\lambda_j\asymp j^{-a}\) for some \(a>1\). Then, with probability at least \(1-\exp(-\Omega(M))\),
\[
\mu_j(\Sigma)\asymp j^{-a},
\qquad j\in[M].
\]
\end{lemma}

\begin{lemma}[Tail spectral ratio under power law]
\label{lem:aux-G5-E6}
Suppose \(\lambda_j\asymp j^{-a}\) with \(a>1\). Then there exists an \(a\)-dependent constant \(c>0\) such that for any \(k\ge 0\),
\[
\frac{\mu_{M/2}(\Sigma_{k:\infty})}{\mu_M(\Sigma_{k:\infty})}\le c
\]
with probability at least \(1-\exp(-\Omega(M))\).
\end{lemma}

\begin{lemma}[Source-condition approximation bound]
\label{lem:aux-C5-E7}
Suppose the source condition with exponent \(b>1\) holds. Then, with probability at least \(1-\exp(-\Omega(M))\) over the sketch matrix,
\[
M^{1-b}
\lesssim
\mathbb{E}_{w^\ast}[\mathrm{Approx}]
\lesssim
M^{1-b}.
\]
The hidden constants depend only on the source exponents.
\end{lemma}

\begin{lemma}[Empirical spectrum under power law]
\label{lem:aux-2025-E8}
Assume \(\mu_j(\Sigma)\asymp j^{-a}\) for \(j\in[M]\). Then there exist \(a\)-dependent constants \(c,c_1,c_2>0\) such that, conditioned on \(S\), with probability at least \(1-\exp(-\Omega(N))\) over the sample,
\[
c_1 j^{-a}
\le
\mu_j(\widehat{\Sigma})
\le
c_2 j^{-a},
\qquad j\le \min\{M,N/c\},
\]
and
\[
\mu_j(\widehat{\Sigma})\lesssim j^{-a},
\qquad j\le \min\{M,N\}.
\]
\end{lemma}
\section{Experimental setup}
\label{sec:exp-setup}

The main-text experiments are generated by
\path{experiments/experiments_dynamic.py}, with optimization primitives in
\path{experiments/helpers.py}. From the project root, the reported results can
be regenerated with
\begin{lstlisting}[basicstyle=\ttfamily\footnotesize,breaklines=true]
python3 experiments/experiments_dynamic.py \
  --output-dir experiments \
  --d 10000 --m 64 --n 512 --t 64 --l 512 \
  --reps 300 --seed 123 \
  --one-pass-n-values 96,128,160,192,224,256,288,320,352,384,416,448,480,512 \
  --multipass-grad-budget-values 512,640,768,1024,1280,1536,2048,2560,3072,3584,4096
\end{lstlisting}
The omitted defaults are \(a=2\), \(b=1.5\), \(\gamma=0.05\), and
\(\sigma=1\). Since \texttt{--grad-budget} is omitted, the code sets
\(D_{\mathrm{grad}}=L(N/T)=4096\), which is used in Experiment 4.

\subsection{Synthetic sketched-regression model}

The ambient covariance is diagonal with \(\lambda_i=i^{-a}\). Independently
for each coordinate, the code draws
\[
w_i^\ast\sim
\mathcal N\!\left(0,\frac{i^{-b}}{\lambda_i}\right),
\]
and draws the Gaussian sketch entries from \(\mathcal N(0,1/M)\). It then
forms the exact sketched covariance \(\Sigma=SHS^\top\) and the joint Gaussian
law of the sketched feature \(z=Sx\) and clean signal
\(\langle x,w^\ast\rangle\). Each dataset is sampled directly from this joint
law with
\[
y=\langle x,w^\ast\rangle+\varepsilon,
\qquad
\varepsilon\sim\mathcal N(0,\sigma^2).
\]
The sketch and target are fixed using seed \(123\). The dataset and
optimization randomness are refreshed over \(R=300\) repetitions, and all
schedules at a common budget use the same dataset within each repetition.

\subsection{WSD influence and batch schedules}

For a horizon \(J\), the implementation sets the warmup endpoint to
\(\lfloor0.1J\rfloor\), the stable endpoint to \(\lfloor0.5J\rfloor\), and
uses the WSD schedule in Section~\ref{sec:prelim}. From the eigenvalues of the
simulated \(\Sigma\), it computes
\[
\kappa_t^{(J)}
=
\gamma_t^2\sum_{j=1}^M\mu_j^2
\prod_{s=t+1}^J(1-\gamma_s\mu_j)^2.
\]
This is done once for \(T=64\) and once for \(L=512\).

At each total budget, four integer schedules are constructed from the weights
\[
\begin{aligned}
\sqrt{\kappa_t^{(J)}}&\quad\text{(oracle)},
&1&\quad\text{(constant)},\\
J-t+1&\quad\text{(early-large/late-small)},
&t&\quad\text{(early-small/late-large)}.
\end{aligned}
\]
Each update receives at least one sample. The remaining budget is allocated
proportionally to the corresponding weights and rounded by a largest-remainder
rule that preserves the exact total; multi-pass batches are additionally
capped at \(N\). A negligible positive offset is added inside the oracle
square root only to avoid zero numerical weights.

\subsection{Experiments 1 and 2: dynamic-schedule scaling}

For one pass, we fix \(T=64\) and vary
\[
N\in\{96,128,160,192,224,256,288,320,352,384,416,448,480,512\}.
\]
Every schedule satisfies \(\sum_{t=1}^T B_t=N\), and every sample is used
once. If \(u_{T,r}^{\mathrm{op}}\) is the output on repetition \(r\), the
reported centered variance is
\[
\widehat{\mathrm{Var}}
=
\frac{1}{R-1}\sum_{r=1}^R
\left\|\Sigma^{1/2}
\left(u_{T,r}^{\mathrm{op}}-\frac1R\sum_{q=1}^R
u_{T,q}^{\mathrm{op}}\right)\right\|_2^2.
\]
Figure~\ref{fig:experiments}(a) indexes this quantity by
\[
\frac1{B_{\mathrm{eff}}}
=
\frac{\sum_t\kappa_t^{(T)}/B_t}{\sum_t\kappa_t^{(T)}}.
\]
The dotted reference is a single fitted multiple of
\(B_{\mathrm{eff}}^{-1}\) across all schedules and budgets.

For multi-pass SGD, we fix \(N=L=512\) and use
\[
D_{\mathrm{grad}}
\in\{512,640,768,1024,1280,1536,2048,2560,3072,3584,4096\}.
\]
Each mini-batch is sampled uniformly without replacement from the fixed
dataset, independently across updates. On the same dataset and WSD trajectory,
we compute the full-batch reference \(\theta_L\) and report
\[
\widehat{\mathrm{Fluc}}
=
\frac1R\sum_{r=1}^R
\left\|\Sigma^{1/2}
\left(u_{L,r}^{\mathrm{mp}}-\theta_{L,r}\right)\right\|_2^2.
\]
Figure~\ref{fig:experiments}(b) indexes this quantity by
\[
\rho_{\mathrm{eff}}
=
\frac{\sum_t\rho_t\kappa_t^{(L)}}{\sum_t\kappa_t^{(L)}},
\qquad
\rho_t=\frac{N-B_t}{B_t(N-1)}.
\]
The dotted reference is one fitted multiple of \(\rho_{\mathrm{eff}}\).

\subsection{Experiment 3: relative allocation comparison}

At \(N=512\) and \(D_{\mathrm{grad}}=4096\), the code normalizes both the
measured stochastic errors and their influence-weighted predictions by the
constant schedule. The resulting ratios are stored in
\path{experiments/dynamic_allocation_comparison.csv} and displayed in
Figure~\ref{fig:experiments}(c). This comparison directly includes the
rounded oracle schedule used in Corollary~\ref{cor:influence-matched-allocation}.

\subsection{Experiment 4: WSD influence profile}

For \(L=512\) and \(D_{\mathrm{grad}}=4096\), the warmup, stable, and decay
phases are iterations \(1{:}51\), \(52{:}256\), and \(257{:}512\). The code
normalizes \(\gamma_t\), \(\kappa_t\), and \(\sqrt{\kappa_t}\) by their
respective maxima and plots them with the oracle integer schedule and the
constant batch size \(D_{\mathrm{grad}}/L=8\). The oracle schedule sums to
\(4096\), ranges from \(1\) to \(23\), and reaches its maximum at iteration
\(256\), where \(\kappa_t\) also peaks. The full per-iteration data are saved
in \path{experiments/dynamic_oracle_influence_profile.csv}.

\subsection{Reported files}

The files \path{experiments/dynamic_one_pass_summary.csv} and
\path{experiments/dynamic_multipass_wor_summary.csv} contain the sweep
statistics. The corresponding \texttt{*\_schedules.csv} files record every
\(B_t\), \(\kappa_t\), and, for multi-pass, \(\rho_t\). The allocation table
and WSD profile use \path{experiments/dynamic_allocation_comparison.csv} and
\path{experiments/dynamic_oracle_influence_profile.csv}, respectively.
\section{Limitations and Broader Effects}
\label{sec:limitations-broader-effects}

\paragraph{Limitations.}
Our analysis is intentionally stylized. The theory is proved for sketched linear regression under a Gaussian design, a homoscedastic teacher--student label model, power-law covariance decay, and a source condition on the target parameter. These assumptions let us separate approximation, bias, variance, and fluctuation cleanly, but they do not cover misspecification, heavy-tailed or dependent data, non-Gaussian sketching, or feature learning in nonlinear models. The experiments are likewise synthetic and are designed to test the predicted scaling behavior rather than empirical competitiveness on real tasks. Accordingly, the resulting batch-size scaling laws should be interpreted as precise results for a controlled regime, not as universal prescriptions for all SGD training problems.

\paragraph{Broader effects.}
One positive effect of this work is sharper guidance for how batch size changes optimization noise in large-scale training. By isolating when batching primarily alters stochastic terms and when without-replacement sampling can further reduce fluctuation, the results can inform more compute-efficient training strategies and improve theoretical intuition for algorithm design. The main risk is overgeneralization: if stylized scaling laws derived under restrictive assumptions are transferred directly to complex real-world systems, practitioners may choose training rules that are poorly calibrated for misspecified models, distribution shift, or fairness and safety constraints that are absent from our setup. We therefore view these results as a theoretical foundation that should be paired with application-specific validation before informing deployment decisions.


\end{document}